\documentclass{article}
\usepackage[T1]{fontenc}
\usepackage[utf8]{inputenc}
\usepackage[a4paper, margin=1in]{geometry}

\usepackage{microtype}
\usepackage{graphicx}
\usepackage{subcaption}
\usepackage{booktabs}

\usepackage[colorlinks,urlcolor=blue,citecolor=blue,linkcolor=blue]{hyperref}

\usepackage{amsmath}
\usepackage{amssymb}
\usepackage{mathtools}
\usepackage{amsthm}

\usepackage[capitalize,noabbrev]{cleveref}

\theoremstyle{plain}
\newtheorem{theorem}{Theorem}[section]
\newtheorem{proposition}[theorem]{Proposition}
\newtheorem{lemma}[theorem]{Lemma}

\theoremstyle{definition}
\newtheorem{definition}[theorem]{Definition}
\newtheorem{assumption}[theorem]{Assumption}
\theoremstyle{remark}
\newtheorem{remark}[theorem]{Remark}

\usepackage{tikz}
\usetikzlibrary{
  arrows.meta,
  positioning,
  decorations.pathreplacing,
  calc
}

\usepackage{bbm}
\usepackage{overpic}
\usepackage{enumitem}
\usepackage{nicematrix} 
\usepackage{stmaryrd} 
\usepackage[numbers,sort&compress]{natbib} 
\usepackage{tabularx}
\usepackage{pifont}
\usepackage[dvipsnames]{xcolor}

\newcommand{\cmark}{\textcolor{green}{\ding{51}}}
\newcommand{\xmark}{\textcolor{red}{\ding{55}}}
\newcommand{\noComparison}{} 
\newcommand{\noExplicitMatch}{{$\emptyset$}} 
\newcommand{\notDiscussed}{} 
\newcommand{\unlikely}{\textcolor{gray}{\ding{55}}}
\newcommand{\Cone}{\textsf{C1}}
\newcommand{\Ctwo}{\textsf{C2}}

\DeclareMathOperator*{\argmax}{argmax}
\DeclareMathOperator*{\softmax}{softmax}
\DeclareMathOperator*{\sign}{sign}
\DeclareMathOperator*{\Var}{Var}
\DeclareMathOperator*{\diag}{diag}
\DeclareMathOperator*{\dist}{dist}
\DeclareMathOperator*{\linspan}{span}

\newcommand{\inner}[2]{\langle #1, #2 \rangle}

\def\din{{d_{\mathrm{in}}}}
\def\dout{{d_{\mathrm{out}}}}

\def\model{\mathrm{model}}
\def\drift{\mathcal{D}}
\def\Cov{\mathrm{Cov}}

\def\MSE{\mathrm{MSE}}
\def\CE{\mathrm{CE}}
\def\margin{\mathrm{margin}}
\def\argmin{\mathrm{argmin}}
\def\shift{\mathbf{S}}
\def\E{\mathbb{E}}
\def\R{\mathbb{R}}

\def\eps{\varepsilon}
\def\stdnoise{\sigma_{\eps}}
\def\I{\mathbf I}
\def\L{\mathcal{L}}

\def\energyinc{\mathcal{E}}
\def\drift{\mathcal{D}}
\def\DISV{\drift_{\mathrm{SV}}}
\def\restISV{\DISV^{\mathrm{rest}}}
\def\restoreISV{\DISV^{\mathrm{restore}}}
\def\rest{\mathrm{rest}}

\def\DnonISV{\drift_{\mathrm{non-SV}}}

\newcommand{\noiseDistrib}{\mathcal{D}_{\xi}} 
\newcommand{\WV}{\mathbf W_{\mathrm{V}}}
\newcommand{\WQ}{\mathbf W_{\mathrm{Q}}}
\newcommand{\WK}{\mathbf W_{\mathrm{K}}}
\newcommand{\Wkq}{{\mathbf W_{\mathrm{KQ}}}} 
\newcommand{\maxmargin}{{\hat{\mathbf W}_{\mathrm{KQ}}}} 
\newcommand{\Ectx}{\mathbf E} 
\newcommand{\query}{\mathbf e} 
\newcommand{\token}{\mathbf e} 
\newcommand{\target}{\mathbf v} 

\newcommand{\attn}{\mathbf a} 
\newcommand{\Cqe}{\mathbf C} 
\def\stdtoken{\sigma_{\mathrm{token}}}
\def\stdnoise{\sigma_{\mathrm{noise}}}
\def\vartoken{\sigma^2_{\mathrm{token}}}
\def\varnoise{\sigma^2_{\mathrm{noise}}}
\newcommand{\indicator}{\delta}

\newcommand{\ISV}{\mathcal{I}_{\mathrm{SV}}}
\newcommand{\nonISV}{\mathcal{I}_{\mathrm{nonSV}}}
\newcommand{\flatten}{\operatorname{flatten}}
\newcommand{\DataGeom}[1]{\textbf{#1 (\textcolor{ForestGreen}{data geometry})}}
\newcommand{\Converg}[1]{\textbf{#1 (\textcolor{orange}{stable trajectory})}}
\definecolor{DustyRose}{RGB}{181, 101, 118}
\definecolor{Mauve}{RGB}{176, 141, 156}
\definecolor{WarmLavender}{RGB}{190, 160, 180}
\newcommand{\Success}[1]{\textbf{#1 (\textcolor{DustyRose}{successful run})}}

\newcommand\textmagenta[1]{{\color{magenta}#1}}
\newcommand\mathmagenta[1]{{\color{magenta}{#1}}}
\newcommand\magenta[1]{\relax\ifmmode\mathmagenta{#1}\else\textmagenta{#1}\fi}

\title{
Gaussian Match-and-Copy\\
\large A Minimalist Benchmark for Studying Transformer Induction
}

\author{
Antoine Gonon
\and
Alexandre Cordonnier
\and
Nicolas Boumal
\and
\small Institute of Mathematics, EPFL, Switzerland
}

\date{} 

\begin{document}
\maketitle

\begin{abstract}
Match-and-copy is 
a core retrieval primitive 
used at inference time by large language models
to retrieve a matching token from 
the context then copy its successor. 
Yet, understanding how this behavior emerges on natural data 
is challenging because 
retrieval and memorization are entangled. 
To disentangle the two, we introduce 
\emph{Gaussian Match-and-Copy} (GMC), 
a minimalist benchmark that isolates long-range 
retrieval through pure second-order correlation signals. 
Numerical investigations show that this task retains 
key qualitative aspects of how Transformers develop 
match-and-copy circuits in practice, 
and separates architectures by their retrieval capabilities.
We also analyze the optimization dynamics 
in a simplified attention setting. 
Although many solutions are \emph{a priori} 
possible under a regression objective, 
including ones that do not implement retrieval, 
we identify an implicit-bias regime in which gradient 
descent drives the parameters to diverge while 
their direction aligns with the max-margin separator, 
yielding hard match selection. 
We prove this max-margin alignment for GD 
trajectories that reach vanishing empirical loss 
under explicit technical conditions. 
\end{abstract}

\section{Introduction}
\label{sec:intro}

How do Transformers learn to learn in-context? 
A long-studied hypothesis in the literature is that \emph{match-and-copy} 
is an important primitive underlying in-context learning (ICL) in large language models 
\citep{Elhage21TransformerCircuits,Olsson22AnthropicInContextLearning,
      Bietti23TransformerMemoryViewpoint,CallumMcDougall23InductionHeadsIllustrated,
      Chan22DataDistributionDrivesInContextLearning,Singh23TransienceICL,
      Reddy24GaussianDataforMechanisticStudyClassificationICL,
      Singh24InductionHeadFormation,Crosbie24AblatingIHDeterioratesICL,
      Musat25EmergenceInductionHeadsICL}. 
For instance, in
\[
\text{``Granny Susie comes tomorrow. I love Granny }\rule{0.5cm}{0.4pt}\text{''}
\]
a natural continuation is \text{``Susie''}. 
This prediction can arise either from memorization 
(a stored association \text{``Granny''}$\mapsto$\text{``Susie''}) 
or from an inference-time retrieval procedure that searches the 
current context for a match, then copies its successor. 
Because the space of possible associations is vast, 
models that can perform match-and-copy are better equipped to adapt to the context. 

In practice, 
Transformers of various scales have been shown to develop a two-head attention circuit 
that implements a match-and-copy 
mechanism in an abstract way (the same heads work independently of the token distribution)
\citep{Elhage21TransformerCircuits,Olsson22AnthropicInContextLearning}. 
Moreover, the formation of such heads coincides with sharp decreases in training loss, 
suggesting that acquiring match-and-copy is important for overall model performance. 
This two-head circuit is composed of a \emph{Previous-Token Head} (PTH) 
that marks the predecessor of a token, 
and an \emph{Induction Head} (IH) 
in a subsequent layer 
that searches for that mark to copy the continuation 
\citep{Elhage21TransformerCircuits,Olsson22AnthropicInContextLearning,
CallumMcDougall23InductionHeadsIllustrated}. 

Despite these insights, it remains difficult to pinpoint 
\emph{why} and \emph{when} this circuit emerges so reliably. 
Analyses based on large-scale language data 
are expensive to run and entangle many factors, 
offering limited experimental control 
\citep{%
Elhage21TransformerCircuits,
Olsson22AnthropicInContextLearning,
Min22ICLwithRandomLabels,
Crosbie24AblatingIHDeterioratesICL,
Lee24ICLwithOtherArchitectures}. 
This has motivated the development of synthetic benchmarks 
that are cheaper to run and 
amenable to analysis. 
There is a long history of such synthetic tasks in the literature 
(see our review in \Cref{sec:related-work}), 
each with its own set of design choices. 
Different constructions are meant to emphasize 
different aspects of match-and-copy (or more generally, in-context learning) mechanisms. 
In this work, we focus on isolating \emph{long-range correlation-based} 
match-and-copy retrieval, 
in an inexpensive and robust way (i.e., we want the PTH$\to$IH circuit to emerge systematically, 
without shortcut solutions). 
This focus reflects two properties of attention-based models: 
their ability to model long-range dependencies--often cited as a key strength over other sequence models 
\citep{Arora24ZoologyRecall,Jelassi24TransformersVsSSMAtCopying,Patro24Mamba360,Fu23H3}--and the fact that 
correlation provides a particularly natural signal for attention to exploit. 

To this end, we introduce \textbf{Gaussian Match-and-Copy (GMC)}. 
GMC is a minimalist prediction task where the only signal is a single, 
hidden correlation between the query (the last token in the context) 
and an earlier 
context token 
positioned at random in an otherwise i.i.d.\ Gaussian context. 
To solve it, a model cannot memorize associations; 
it must learn an active search-and-retrieval algorithm.

\begin{figure}[t]
\centering
\resizebox{0.9\columnwidth}{!}{%
\begin{tikzpicture}[
  font=\small,
  >=Latex,
  token/.style={draw, rounded corners=2pt, minimum width=10mm, minimum height=7mm, inner sep=1.5pt, align=center},
  tokenRed/.style={token, draw=red!75!black, line width=0.8pt},
  tokenMatch/.style={token, draw=teal!70!black, fill=teal!12, line width=0.9pt},
  tokenSucc/.style={token, draw=purple!70!black, fill=purple!10, line width=0.9pt},
  tokenQuery/.style={token, draw=green!55!black, fill=green!10, line width=0.9pt},
  tokenTarget/.style={token, draw=orange!75!black, fill=orange!12, line width=0.9pt},
  lab/.style={inner sep=1pt}
]

\def\gap{2mm} 
\def\qtgap{32mm} 

\node[tokenRed] (x1) {$\token_{1}$};
\node[tokenRed, right=\gap of x1] (x2) {$\token_{2}$};
\node[lab, right=\gap of x2] (xdotsL) {$\cdots$};

\node[tokenMatch, right=\gap of xdotsL] (xt0) {$\token_{t_0}$};
\node[tokenSucc,  right=\gap of xt0]    (xt0p1) {$\token_{t_0+1}$};

\node[lab, right=\gap of xt0p1] (xdotsR) {$\cdots$};
\node[tokenRed, right=\gap of xdotsR] (xT) {$\token_{T}$};

\node[tokenQuery, right=\gap of xT] (q) {$\query$};
\node[tokenTarget, right=\qtgap of q] (y) {$\target$}; 

\draw[black, line width=0.7pt, -{Latex[length=2.2mm]}]
  ($(q.east)+(0,0)$) -- ($(y.west)+(0,0)$);

\node[lab, above=4.5mm of q] (qeq) {$\query=\Cqe\,\token_{t_0}+\xi$};
\node[lab, above=4.5mm of y] (yeq) {$\target=\WV\,\token_{t_0+1}+\varepsilon$};

\draw[green!55!black, line width=0.5pt, -{Latex[length=1.8mm]}] (qeq.south) -- (q.north);
\draw[orange!75!black, line width=0.5pt, -{Latex[length=1.8mm]}] (yeq.south) -- (y.north);

\draw[teal!70!black, line width=0.9pt, -{Latex[length=2.2mm]}]
  (q.north west) .. controls +(-10mm,12mm) and +(0,12mm) .. (xt0.north)
  node[pos=0.55, above=2mm, lab] {(1) Match (correlation-based)};

\draw[purple!70!black, line width=0.9pt, -{Latex[length=2.2mm]}]
  (xt0p1.south) .. controls +(0,-14mm) and +(-12mm,-14mm) .. (y.south west)
  node[pos=0.55, below=2mm, lab] {(2) Copy};

\def\braceY{-12mm}   
\def\braceLabY{3mm} 

\draw[decorate, decoration={brace, mirror, amplitude=5pt}, yshift=\braceY]
  (x1.south west) -- (xT.south east)
  node[midway, below=\braceLabY] {context};

\draw[decorate, decoration={brace, mirror, amplitude=5pt}, yshift=\braceY]
  (q.south west) -- (q.south east)
  node[midway, below=\braceLabY] {query};

\draw[decorate, decoration={brace, mirror, amplitude=5pt}, yshift=\braceY]
  (y.south west) -- (y.south east)
  node[midway, below=\braceLabY] {target};

\end{tikzpicture}
}
\caption{\textbf{Gaussian Match-and-Copy} task (\Cref{def:mc}). 
The query matches a hidden context token via correlation, 
and the target copies the successor of that token.}
\label{fig:gaussian-mc}
\end{figure}
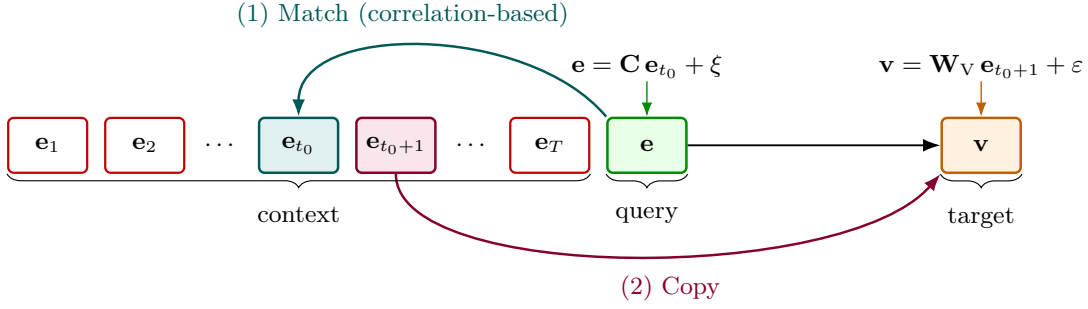

\begin{definition}[Gaussian Match-and-Copy]\label{def:mc}
Fix dimensions $\din,\dout$, noise level $\varnoise$, and token variance $\vartoken$. 
Consider $\WV \in \mathbb{R}^{\dout \times \din}$, 
$\Cqe \in \mathbb{R}^{\din \times \din}$ 
called value and correlation matrix, 
and $\noiseDistrib$ a distribution over $\din$-dimensional tokens. 
A sample $(\token_1, \dots, \token_T, \query, \target)$ is generated as follows:
\begin{enumerate}
    \item \textbf{Context:} 
    Draw $T$ independent token embeddings 
    $\token_1, \dots, \token_T \sim \mathcal{N}(0, \frac{\vartoken}{\din} \I_{\din})$.
    \item \textbf{Hidden Match:} 
    Select a hidden index $t_0$ uniformly at random from $\{1, \dots, T-1\}$.
    \item \textbf{Query Injection:} 
    The query $\query$ is drawn such that it is correlated 
    \emph{only} with $\token_{t_0}$:
    \[
    \Cov(\query, \token_t) = 
    \begin{cases} 
        \frac{\vartoken}{\din}\Cqe & \text{if } t=t_0,\\ 
        0 & \text{if } t \neq t_0.
    \end{cases}
    \]
    Concretely, this is achieved by setting
    $\query=\Cqe\,\token_{t_0}+\xi$, 
    where $\xi\sim \noiseDistrib$ is independent noise. 
    \item \textbf{Target:} The target is the \emph{next} token, projected and noisy: 
    $\target = \WV \token_{t_0+1} + \eps$, where $\eps \sim \mathcal{N}(0, \stdnoise^2 \I_{\dout})$.
\end{enumerate}
\end{definition}

The most natural strategy is clear: 
compare $\query$ with all $\token_t$ to maximize the overlap 
$\token_t^\top \Cqe \query$, identify $t_0$, 
and output $\WV \token_{t_0+1}$. 
This requires a mechanism that can (1) Match and (2) Copy, 
see \Cref{fig:gaussian-mc}.

The key design choices are as follows.
First, there is no fixed association rule to memorize: 
the matched token $\token_{t_0}$ and its successor $\token_{t_0+1}$ 
are independent random variables, so successful prediction cannot rely on stored input--output pairs.
Second, all context tokens share identical marginal distributions.
In particular, the matched token $\token_{t_0}$ cannot be identified by any local statistic, 
such as its moments or a special marker.
The only signal distinguishing it is its second-order correlation with the query.
Third, the match index $t_0$ can occur arbitrarily far 
in the context, requiring genuine long-range retrieval.
Finally, the Gaussian construction admits 
simple geometric and concentration properties, 
which we later exploit to analyze the optimization dynamics theoretically.

\paragraph{Contributions.} 
From the proposed GMC task, 
we extract several conclusions.

(i) \textbf{Reliable PTH$\to$IH emergence in Transformers.}
In our experiments, standard Transformers trained on GMC develop a PTH$\to$IH circuit, 
and the appearance of these heads aligns with a sharp drop 
in the validation loss. 
This mirrors prior observations made on LLMs, 
e.g., in \citep{Elhage21TransformerCircuits,Olsson22AnthropicInContextLearning} 
for models with sizes ranging between 13M and 13B parameters 
trained to predict the next token on generic web data. 
We find this behavior to be robust across a range of architectural and optimization choices.

(ii) \textbf{Transfer beyond the Gaussian distribution.}
Freezing the attention blocks after training on GMC, 
then 
fine-tuning only input/output 
embeddings enables rapid adaptation 
to new data distributions. 
This provides evidence that the 
learned heads implement an abstract 
match-and-copy strategy rather than 
Gaussian-specific heuristics.

(iii) \textbf{Architectural separation.}
Under similar training and inference budgets, 
Transformers solve GMC reliably,
whereas structured state-space models 
and other non-attention sequence models perform substantially worse.
This reveals a clear architectural gap in long-range, 
correlation-based retrieval.

(iv) \textbf{Induction via an implicit max-margin regime in simplified attention models.} 
Motivated by the architectural gap observed above, we study a minimal attention-only model that can implement
a PTH$\to$IH circuit.
In this setting, many optimization outcomes are \emph{a priori} plausible under square loss (MSE, a natural loss for GMC regression),  
including finite-norm
interpolants or divergence along different separating directions. 
Empirically, we observe regimes in which gradient descent attains very small training loss while the weights grow in norm
and their direction aligns with the max-margin separator, driving attention toward increasingly hard match selection.
On the theory side, we prove a conditional max-margin result:
for GMC data conditioned on a high-probability geometric event, any gradient descent trajectory that reaches vanishing loss
and satisfies explicit MSE-specific regularity conditions converges in direction to the max-margin solution and diverges at a logarithmic rate. 
A central ingredient is the geometric event, which 
rules out finite-norm MSE minimizers and forces successful runs into a diverging, increasing-margin regime. 
The remaining assumptions isolate MSE-specific interaction terms needed to recover the max-margin limit.

Together, these results position GMC as 
an inexpensive and simple diagnostic task for studying induction mechanisms in sequence models: 
it isolates a core inference primitive, exposes a capability 
gap between architectures, 
and opens a promising route toward theoretical understanding.

\section{Related Work}
\label{sec:related-work}
Mechanistic studies has established 
\emph{previous-token heads} (PTH) 
and \emph{induction heads} (IH) as canonical circuits enabling 
match-and-copy behaviors in language models 
\citep{Elhage21TransformerCircuits,Olsson22AnthropicInContextLearning}, 
with subsequent work refining probes, metrics, 
and interpretations of induction-style behavior 
\citep{%
CallumMcDougall23InductionHeadsIllustrated,
Chan22DataDistributionDrivesInContextLearning,
Singh23TransienceICL,
Musat25EmergenceInductionHeadsICL,
Bietti23TransformerMemoryViewpoint,
Nichani24OneGramTaskLearnsIH,
Chen24nGramTaskLearnsIH,
Rajaraman24ApproximationNGramsByTransformers,
Edelman24OnetoTwoGramMarkovLearning,
Akyurek24ICLLanguageGenbyAutomata,
Varre25LearningInContextNGrams,
Graves14NeuralTuringMachines,
Wang25LazyRichLearningIH,
Elhage21TransformerCircuits,
Olsson22AnthropicInContextLearning,
Crosbie24AblatingIHDeterioratesICL,
Singh24InductionHeadFormation,
Reddy24GaussianDataforMechanisticStudyClassificationICL,
Bhattamishra23ICLBooleanFunctions}. 
A growing set of studies now explicitly tracks or ablates induction-style 
heads in diverse in-context learning (ICL) settings, 
providing evidence that such circuits can matter beyond 
next-token prediction on natural text 
\citep{Singh24InductionHeadFormation,
Reddy24GaussianDataforMechanisticStudyClassificationICL,
Crosbie24AblatingIHDeterioratesICL,
Nichani24OneGramTaskLearnsIH,
Chen24nGramTaskLearnsIH,
Wang25LazyRichLearningIH,
Bhattamishra23ICLBooleanFunctions,
Musat25EmergenceInductionHeadsICL}. 

At the same time, many influential synthetic ICL 
studies focus on high-level task behavior or 
learning dynamics without explicitly probing for PTH/IH mechanisms 
\citep{Garg22InContextCaseStudy,
Akyurek23InContextLearningLinearModelsviaGD,
Oswald23InContextGD,
Lee24ICLwithOtherArchitectures,
Marion22AttentionLayers,
Xie22InContextLearningImplicitBayesianInference,
Makkuva24AttentionMarkov, 
Ba16FastWeightsAttendRecentPast, 
Fu23H3,
Poli23HyenaHierarchy, 
Arora24ZoologyRecall,
Min22ICLwithRandomLabels,
Chan22DataDistributionDrivesInContextLearning,
Singh23TransienceICL,
Dai23GPTMetaOptimization}. 
This distinction motivates controlled benchmarks 
where the emergence of induction circuits can 
be studied directly and reproducibly, 
while still enabling comparisons across architectures.

\paragraph{Match-and-copy benchmarks, and related ICL tasks.}
To study these mechanisms, a natural data structure involves sequences of the form:
\begin{equation}
    \label{eq:template-mc}
    \underbrace{\token_1, \dots, \token_{t_0}, \token_{t_0+1}, \dots, \token_{T}}_{\text{context}}, 
    \underbrace{\query}_{\text{query}} \to \underbrace{\token_{t_0+1}}_{\text{target}}
\end{equation}
Here, the model receives a query $\query$ 
(the final prompt token) that ``matches'' 
a context token embedding 
$\token_{t_0}$ (placed at an unknown position $t_0$) 
and must copy the subsequent token embedding $\token_{t_0+1}$ 
or a function thereof. 
This template has a long history in associative recall and retrieval-style sequence modeling, 
with variants spanning discrete and continuous tokens and 
multiple architectural choices 
\citep{Graves14NeuralTuringMachines,
Ba16FastWeightsAttendRecentPast,
Fu23H3,
Poli23HyenaHierarchy,
Arora24ZoologyRecall}. 
Match-and-copy is also related to 
broader ICL tasks where the model must learn a 
function $f$ (new at each prompt) from $n$ 
demonstrations $(x_i,f(x_i))_{i\le n}$ to predict $f(x)$ for a new input $x$, 
i.e., tasks of the form:
\begin{equation}
\label{eq:template-icl-rw}
(x_1,f(x_1),\dots,x_n,f(x_n),x)\ \mapsto\ f(x),
\end{equation}
where $f$ may be linear regression 
\citep{Garg22InContextCaseStudy,
Akyurek23InContextLearningLinearModelsviaGD,
Oswald23InContextGD,
Lee24ICLwithOtherArchitectures,
Marion22AttentionLayers}, 
a classifier 
\citep{Chan22DataDistributionDrivesInContextLearning,
Singh23TransienceICL,
Dai23GPTMetaOptimization,
Lee24ICLwithOtherArchitectures,
Singh24InductionHeadFormation,
Reddy24GaussianDataforMechanisticStudyClassificationICL,
Crosbie24AblatingIHDeterioratesICL}, 
a Markov/$n$-gram rule 
\citep{Xie22InContextLearningImplicitBayesianInference,
Bietti23TransformerMemoryViewpoint,
Makkuva24AttentionMarkov,
Nichani24OneGramTaskLearnsIH,
Chen24nGramTaskLearnsIH,
Rajaraman24ApproximationNGramsByTransformers,
Edelman24OnetoTwoGramMarkovLearning,
Akyurek24ICLLanguageGenbyAutomata,
Varre25LearningInContextNGrams}, 
or a Boolean function \citep{Bhattamishra23ICLBooleanFunctions}. 
Note that here, 
depending on how the $x_i$'s and $x$ are generated, 
there might not be an explicit ``match'' signal 
injected at data generation time 
between $x$ and one of the $x_i$'s: 
the model may need to interpolate/extrapolate 
based on similarity. 
In some regimes, however, 
ICL tasks can in principle be solved 
by a match-and-copy circuit 
(by retrieving the $x_i$ closest to $x$ and copying $f(x_i)$ as a best guess for $f(x)$).

\paragraph{Design choices and shortcut solutions.}
Synthetic benchmarks for match-and-copy and in-context learning (ICL) necessarily involve
many design choices: discrete versus continuous tokens; exact matches (equality) versus similarity-based matches;
explicit injection of match signals or not; 
the nature of the match signal (e.g., first- versus second-order statistics);
copying a successor token versus copying a function/label; 
the presence of distractors; and the choice of objective,
architecture, and evaluation protocol.
These choices strongly affect
whether the generic PTH$\to$IH solution is required to solve the task. 
Instead, heuristics 
based on memorization or simpler retrieval mechanisms can sometimes be used. 
E.g., if the set of possible match token values is finite and fixed globally across samples, 
it can be memorized, and used to trigger, only conditionally to these values, 
a PTH$\to$IH-like behavior, ending up with conditional PTH$\to$IH circuits 
that are not truly general-purpose match-and-copy mechanisms
\citep{Bietti23TransformerMemoryViewpoint,Fu23H3}. 
Simpler retrieval mechanisms can also be used when the 
token are distinguishable (not i.i.d.)---for instance, relying on
first-order statistics when 
the match token has a 
distinct mean compared to other context tokens
\citep{Marion22AttentionLayers}. 
When such shortcuts are available, 
the emergence of PTH$\to$IH circuits is not guaranteed 
and can be quite sensitive to other design choices: 
whether induction-style solutions arise robustly 
or only in narrow regimes can be 
delicate to characterize---e.g., the emergence of PTH$\to$IH circuits can 
depend on structural properties of the 
data \citep{Chan22DataDistributionDrivesInContextLearning,Bhattamishra23ICLBooleanFunctions},
can co-exist with other mechanisms, 
or can be transient over training time \citep{Singh23TransienceICL}.
Overall, these factors complicate the use of existing benchmarks to reliably probe specific capabilities
when designing new architectures, and make theoretical analysis more challenging.

\paragraph{Most closely related benchmark constructions.}
Our task instantiates \eqref{eq:template-mc} with i.i.d.\ 
Gaussian context tokens and defines the match through a
\emph{second-order} correlation between the query and a single hidden context token, 
which can occur arbitrarily far in the context. 
The first most related task  
is the original associative recall (AR) 
from \citet{Graves14NeuralTuringMachines}, 
which is essentially discrete match-and-copy 
with exact match and exact copy. 
Indeed, it instantiates \eqref{eq:template-mc} with 
i.i.d.\ discrete tokens sampled uniformly from a finite vocabulary, with 
the query equal to one of the context tokens, 
and the target equal to the successor token. 
This early work does not study PTH$\to$IH circuits (pre-Transformer era), 
and the GMC task can be seen as a continuous-valued analogue 
that isolates more generally correlation-based retrieval. 
We also note that most explicit constructions of PTH$\to$IH circuits 
use orthogonal embeddings of discrete tokens, 
and our choice of Gaussian tokens can be seen as a continuous relaxation of this idea
in high dimension, so this AR task from 
\citep{Graves14NeuralTuringMachines}
is definitely the most closely related to GMC.
Since then, AR benchmarks have been revisited and now span a variety of 
design 
\citep{Ba16FastWeightsAttendRecentPast,Fu23H3,Bietti23TransformerMemoryViewpoint,Arora24ZoologyRecall}. 
These more recent AR variants are less directly related to GMC: 
instead they instantiate discrete ICL tasks of the form \eqref{eq:template-icl-rw}. 
In particular, 
the match$\mapsto$target mapping $x\mapsto f(x)$ 
changes at each prompt,
so success cannot rely on memorizing a global match$\mapsto$target dictionary 
stored in the weights; 
these tasks all require in-context retrieval. 
GMC offers a complementary design point 
to these AR benchmarks with two key features:
(i) an architectural gap (Transformers succeed, other models like Hyena/state-space models (SSM)/RNN fail), and
(ii) its amenability to theoretical analysis of PTH$\to$IH emergence. 
In contrast, 
(i) AR benchmarks admit non-attention architectures that 
can solve them reliably. Indeed, 
Hyena/SSM have been explicitly designed to 
close the gap with attention models \citep{Fu23H3,Poli23HyenaHierarchy}
on these AR benchmarks. 
Moreover, (ii) 
there are no existing 
theoretical analyses 
of the optimization dynamics when training 
on AR benchmarks 
\citep{Bietti23TransformerMemoryViewpoint,Arora24ZoologyRecall}, 
which we believe is mostly due to the fact that 
in these AR settings, 
the
PTH$\to$IH circuits do not always emerge reliably across optimization and architectural choices. 
This makes it difficult to isolate a meaningful regime for 
theoretical guarantees. 
The second most related existing task appears in \citet{Marion22AttentionLayers}: 
single-location Gaussian regression. 
It can be interpreted as an instance of \eqref{eq:template-mc} with i.i.d.\ Gaussian tokens, 
but with a fixed deterministic query across samples, 
injected as a mean shift to the matched token. 
Because the matched token is not identically 
distributed as the other context tokens, 
this does not naturally 
promote a PTH$\to$IH circuit,
whose key operation is correlation-based retrieval,
as the model can simply learn to identify 
the mean-shifted token via first-order statistics.

\paragraph{Positioning and novelty.}
Rather than proposing a ``new match-and-copy task'' per se 
or claiming superiority over existing setups, 
our goal is to identify a point in the design space where 
several desiderata are simultaneously satisfied, 
with the next question in mind:
\emph{What is a simple, inexpensive and theory-friendly benchmark 
for testing architectures on long-range correlation-based 
retrieval via PTH$\to$IH-like circuits?}
We come up with GMC, 
that we see as complementary to prior benchmarks. 
\textbf{GMC is a synthethic match-and-copy benchmark where 
(i) retrieval can only exploit long-range second-order correlations, 
(ii) the PTH$\to$IH circuit emerges reliably across optimization 
and architectural choices, 
with qualitative similarities 
to that observed in large language models, 
(iii) attention and non-attention models can 
be compared on their ability to perform 
long-range correlation-based retrieval 
without shortcut solutions, 
and 
(iv) the setup is amenable to theoretical analysis.} 
\Cref{tab:benchmark-comparison} in 
\Cref{app:related-works-table}
summarizes how 
the most closely related 
prior benchmarks
instantiate key design choices and 
highlights the specific combination 
realized by GMC. 

\section{Empirical Properties of Gaussian Match-and-Copy}
\label{sec:expes}

We now study how sequence models trained on GMC behave in practice. 
All experiments minimize the mean-squared error
\begin{equation}
\label{eq:mse-loss-expes}
  \frac1{n\dout}\sum_{i=1}^{n}\lVert
    \target^{(i)} - 
    \model(x^{(i)};W)
  \rVert_2^{2}
\end{equation}
over a training set of $n$ samples, 
where 
$x^{(i)}=(\token^{(i)}_1,\dots,\token^{(i)}_T,\query^{(i)})$ 
and $\target^{(i)}$ form the $i$-th GMC sample and 
$\model(\cdot;W)$ is a model with trainable weights $W$. 
For experiments, we normalize by $n\dout$ to make the loss 
comparable across different settings. 

Our experiments address three questions:
(i) Do Transformers trained on GMC 
reproduce the qualitative behavior seen in large scale 
LLMs, i.e., do previous-token and induction-head circuits emerge, and does this coincide with a sharp change in training dynamics?
(ii) Do the learned mechanisms transfer beyond the Gaussian training distribution?
(iii) Do competing sequence models perform as well as Transformers on GMC when given equal resources?

\subsection{Transformers solve GMC via PTH\,$\to$\,IH circuits}
\label{subsec:MCviaPTH+IH}

In the settings we tested, Transformers reliably learn the GMC task, and when trained long enough they approach the
irreducible noise floor induced by the observation noise.
Crucially, we observe the characteristic loss drop 
documented in LLMs trained on large-scale web data
\citep{Olsson22AnthropicInContextLearning}, 
together with the spontaneous emergence of PTH and IH circuits:

\begin{enumerate}
    \item \textbf{Loss Drop:} The loss follows a plateau-drop-plateau pattern 
    (\Cref{fig:phase-change-llama3}).
    \item \textbf{Coincidental PTH\,$\to$\,IH Formation:} 
    The \textbf{PTH Score} and \textbf{IH Scores} 
    (defined below) follow the same plateau-drop-plateau pattern. 
    During the initial plateau, scores remain close to random baselines. 
    Coinciding exactly with the loss drop, 
    they rise sharply to values close to $1$ (perfect PTH and IH behavior) 
    and remain stable thereafter.
\end{enumerate}

Layer-wise analysis confirms the expected circuit topology: 
PTHs always appear in an earlier layer than IHs 
(e.g., in our 2-layer default: Layer 0 specializes as PTH and Layer 1 as IH), 
adhering to the composition mechanism 
originally described in \citet{Elhage21TransformerCircuits,Olsson22AnthropicInContextLearning}. 
In \Cref{app:mechanistic-validation}, 
we provide a more detailed per-head analysis 
to confirm that the identified heads 
indeed look similar in function and structure to the
ones documented in LLMs. 
Overall, GMC-trained Transformers 
thus reproduce key qualitative aspects 
of induction head formation seen in 
LLMs trained on natural language data,
despite the drastic difference 
in scale and data distribution.

\begin{figure}[ht]
    \centering
    \includegraphics[width=0.5\linewidth]{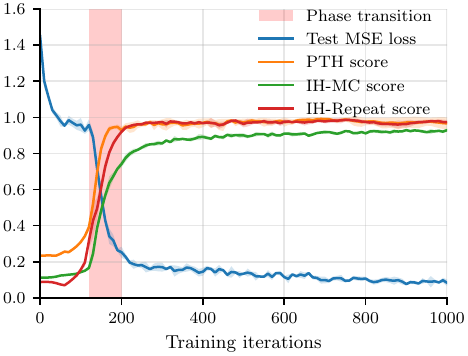}
    \caption{\textbf{Co-occurrence of Loss Drop and Circuit Emergence.} 
    The sudden drop in test loss aligns perfectly with the saturation of PTH and IH attention scores. 
    Here:
    2-layer Llama 3 trained on GMC with $T=8$, 
$\din=16$, $\Cqe=\frac{1}{(1.2)^2}\I$, 
and hyperparameters from \Cref{tab:mc-experiment-specs} in the appendix. 
See \Cref{app:robustness-emergence} for other settings 
(varying depth, $T$, $\din$, and $\Cqe$), 
all showing the same behavior. 
}
    \label{fig:phase-change-llama3}
\end{figure}

\paragraph{Cost and robustness of PTH\,$\to$\,IH emergence.} 
A central motivation for GMC is to provide 
inexpensive setups where induction-style 
circuits can be studied
in a controlled and reproducible way.
To keep the benchmark accessible,
we therefore focus on small-dimensional 
regimes that already exhibit a clear 
loss-drop together with PTH and IH scores. 
We also ran additional sweeps 
over Llama 3 and GPT-2 architectures 
across various depths 
($n_\mathrm{layer}\in\{2,4,8,16,32\}$), 
context length (up to $T=1024$), 
data dimension (up to $\din=1024$), 
correlation strength and
optimization 
hyperparameters (reported in \Cref{app:robustness-emergence}).
While suboptimal choices can affect the \emph{timing} 
and \emph{sharpness} of the transition (or lead to non-convergence),
when a model solves the task in our experiments, 
it does so with high PTH/IH scores 
that rise concurrently with the loss drop.

\paragraph{Head scores definition.}
Following the seminal work by 
\citet{Olsson22AnthropicInContextLearning}, 
we probe attention heads using scores computed 
from attention coefficients $a_{t,s}$ 
(from query position $t$ to key position $s$). 
A \emph{previous-token head} concentrates its attention to the 
immediately preceding position: 
the \emph{PTH score} of a head is $a_{t,t-1}$ 
averaged over query positions $t$. 
We report the maximum PTH score over all heads. 
An \emph{induction head} links a query token to the 
successor of a matched position. 
We use the \emph{IH-MC score}, tailored to GMC, 
which measures attention from the query to the 
successor of 
the hidden match ($a_{\text{last}, \text{match}+1}$), 
and the \emph{IH-Repeat} score 
\citep{Olsson22AnthropicInContextLearning}, 
which measures attention on duplicated sequences.
All scores are in $[0, 1]$; 
precise definitions and random baselines 
are provided in \Cref{app:pth-ih}.

\subsection{Transfer beyond Gaussian data}
Does GMC teach a general match-and-copy mechanism or just Gaussian statistics? 
We take a Llama 3 model pretrained \emph{only} on GMC, 
freeze its internals (hence the induction circuit), 
replace 
the linear input/output embeddings 
by new ones with compatible dimensions 
with the Omniglot dataset \citep{OmniglotDataset},
which is a standard ICL benchmark 
\citep{Chan22DataDistributionDrivesInContextLearning,
Singh23TransienceICL,Singh24InductionHeadFormation,Lee24ICLwithOtherArchitectures},  
and fine-tune only these embeddings 
on this new task. 
Details of the setup are in \Cref{app:downstream-omniglot}. 
\textbf{Result in \Cref{fig:downstream-omniglot}:} 
The GMC-pretrained model rapidly learns to solve Omniglot 
just by adapting the input/output embeddings, 
suggesting the match-and-copy mechanism 
learned on GMC 
is abstract and transferable 
to other data distributions beyond Gaussians.
Compared to a baseline where we train the full model 
from scratch on Omniglot, 
the GMC-pretrained model reaches $0.9$ 
accuracy with at least $3\times$ fewer FLOPs, 
demonstrating that the learned induction 
circuit also helps efficiency of downstream learning. 

\begin{figure}[ht]
  \centering
  \includegraphics[width=\linewidth]{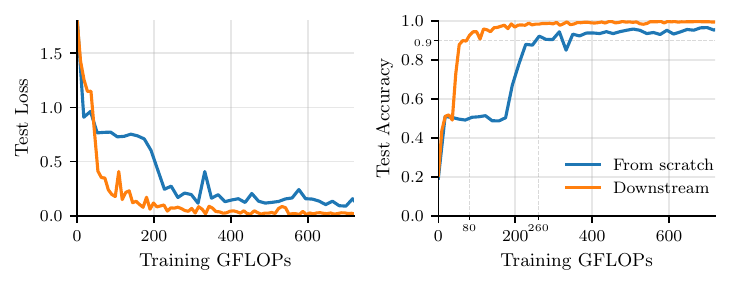}
  \caption{\textbf{Transfer Learning.} A model pretrained on GMC 
  can reuse its attention layers to solve Omniglot few-shot classification via 
  lightweight embedding adaptation.}
  \label{fig:downstream-omniglot}
\end{figure}

\subsection{GMC separates attention from non-attention sequence models}
\label{subsec:gmc-comparison-archis}
Finally, we compare Transformers against Gated RNNs (GRU) and State-Space Models (S4, H3, Hyena), 
at matched parameter counts ($\approx 2.5-3$M), inference GFLOPS 
($\approx 24$), and training FLOPs 
(see \Cref{app:comparison-other-archs} for details).
\textbf{Result in \Cref{fig:comparison-other-archs}:} 
While Transformers hit the noise floor, 
all non-attention baselines---including RNNs and SSM-based or convolutional sequence models---plateau 
at significantly higher loss, 
even when granted $30\times$ more training steps (\Cref{app:comparison-other-archs}), 
or when exploring different compute-equivalent hyperparameter trade-offs 
(e.g., depth versus width).

Note that H3 and Hyena have 
both been designed to 
perform well on associative recall tasks 
\citep{Fu23H3,Poli23HyenaHierarchy}, 
yet \Cref{fig:comparison-other-archs} shows 
they still struggle on GMC. 
This suggests that small differences 
in the task structure (as highlighted in \Cref{sec:related-work}) 
can
matter, 
and that long-range correlation-based retrieval is 
still a  
regime where attention-based architectures have a clear advantage. 
This could serve as a useful diagnostic for future 
architecture design.

\begin{figure}[ht]
    \centering
    \includegraphics[width=0.5\linewidth]{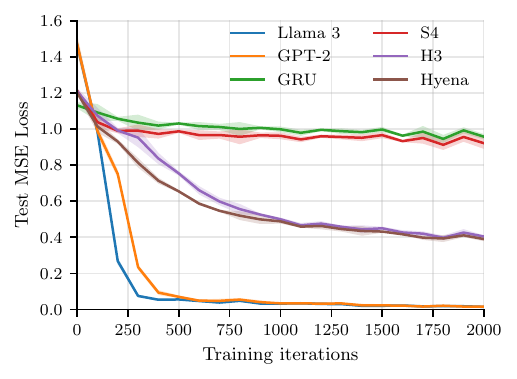}
    \caption{\textbf{Architecture Gap.} Transformers solve GMC; alternatives struggle.} 
    \label{fig:comparison-other-archs}
\end{figure}

\section{Implicit Bias in a Minimal PTH$\to$IH Attention Model}
\label{sec:theory-implicit-bias}

\subsection{Overview and main result}
\label{subsec:theory-overview}

We zoom in on a simplified setting where a minimal attention model can realize a
PTH$\to$IH circuit on GMC.
In this regime, the task 
trained with MSE \eqref{eq:mse-loss-expes} 
can \emph{a priori} be solved in many different ways:
there may exist finite-norm global minimizers, and even when there are none,
there is typically a cone of separating directions along which the parameters
can diverge while the loss vanishes. 
Empirically, in a range of regimes we observe that 
gradient descent exhibits a beneficial implicit bias:
the parameters diverge while aligning with a 
\emph{max-margin} separator, 
which drives attention toward hard match selection
and yields margins that keep increasing over training.
We confirm this behavior theoretically in a specific setting under explicit assumptions. 
The proof is deferred to \Cref{app:mse-bias}. 

\begin{theorem}[Max-margin implicit bias for MSE in the minimal model]
\label{thm:mse-bias-main}
Under the technical assumptions collected in \Cref{subsec:mse-technical-assumptions},
for GMC data conditioned on the event of \Cref{subsec:gmc-geometric-event} (which holds with
high probability), the following holds.
Consider gradient descent on the empirical MSE loss for the minimal model of
\Cref{subsec:minimal-setting-empirical}.
If along the iterates $(\Wkq(\tau))_{\tau\ge 0}$ we have
\begin{equation}
\label{eq:success-conditioning}
\MSE(\Wkq(\tau)) \to 0
\text{ and }
\sum_{\tau=0}^{\infty}\|\nabla \MSE(\Wkq(\tau))\|_F^2 < \infty,
\end{equation}
then:
\begin{equation}
\label{eq:main-bias-direction}
\frac{\Wkq(\tau)}{\|\Wkq(\tau)\|_F}
\ \longrightarrow\
\frac{\maxmargin}{\|\maxmargin\|_F},
\end{equation}
where $\maxmargin$ is the max-margin solution defined in
\Cref{subsec:minimal-setting-empirical}.
Moreover, $\|\Wkq(\tau)\|_F$ diverges at a rate 
$\frac{\|\maxmargin\|_F}{2} \log(\tau)$, 
and the induced 
margins increase accordingly (hence attention becomes asymptotically hard on the
true match).
\end{theorem}
The conclusion is conditional: 
it applies to trajectories that both 
succeed in driving the MSE to zero and satisfy the
regularity conditions in \Cref{subsec:mse-technical-assumptions}. 
In \Cref{subsec:assumptions-hold} we provide an empirical
sanity check for these conditions on a representative successful run.

The remainder of this section is devoted to (i) specifying the minimal setting
and the empirical phenomena we aim to explain, 
(ii) stating the event and justifying it holds whp for GMC data, and
(iii) isolating the remaining MSE-specific technical assumptions and explaining
how they differ from the standard max-margin theory for monotone margin losses 
\cite{Soudry18ImplicitBiasCE,JiTelgarsky19Nonseparable,
TarzanaghLZO23MaxMarginTokenSelection,Tarzanagh23TransformersAsSVM,
Vasudeva25ImplicitBiasRatesSelfAttention}.  
This organization is meant to clarify which 
parts of the conclusion are driven by typical GMC geometry (point (ii)),
and which parts rely on additional 
MSE-specific asymptotic regularity conditions once optimization succeeds (point (iii)).

\subsection{Simplified setting and empirical observations}
\label{subsec:minimal-setting-empirical}

\paragraph{Simplified task: exact match, exact copy.}
To isolate the core retrieval mechanism, we strip away the stochastic noise in GMC
and work with the noiseless match-and-copy template:
\begin{equation}
\label{eq:exact-mc}
\query^{(i)}=\token^{(i)}_{t_0^{(i)}}
\qquad\text{and}\qquad
\target^{(i)}=\token^{(i)}_{t_0^{(i)}+1},
\end{equation}
for $n$ training sequences $x^{(i)}=(\token^{(i)}_1,\dots,\token^{(i)}_T,\query^{(i)})$. 

\paragraph{Minimal model: two-layer attention-only, with a frozen PTH.}
The smallest architecture that can implement a PTH$\to$IH circuit is a 2-layer
attention-only model with one head per layer.
We freeze the first layer to implement an exact previous-token head (PTH). This is
equivalent to shifting the keys of the second layer by one position; hence, we
work directly with the standard \emph{shifted-key} parametrization \citep{Singh24InductionHeadFormation,Wang25LazyRichLearningIH}. 
We fix $\WV$ to the ground-truth value map 
used by GMC (and in the present simplified exact-copy setup, $\WV=\I$). 
This isolates the non-convex match-selection mechanism in $\Wkq$. 

Fixing $\WV$ is justified because the problem is \emph{strictly convex} with respect to $\WV$, 
and it does not interact with the match selection mechanism. 
In principle, $\WV$ could be recovered optimally in a first optimization stage. 
Prior works typically do optimization in stages \cite{Nichani24OneGramTaskLearnsIH,Chen24nGramTaskLearnsIH,Wang25LazyRichLearningIH}. 
Here we directly focus on the second stage, 
optimizing over the merged key-query matrix
\begin{equation}
\label{eq:merged-kq}
\Wkq := \WK^\top \WQ \in \R^{\din\times\din}.
\end{equation}
This isolates the \emph{non-convex} part of the model, 
responsible for selecting the match
position. 

Concretely, for context $\Ectx=(\token_1,\dots,\token_T)$ and query $\query$, the
model output is
\begin{equation}
\label{eq:minimal-model-output}
\begin{multlined}
\model(\Wkq;\Ectx,\query)
=\\
\sum_{t=2}^{T}
\underbrace{\softmax\!\Big(\big(\token_s^\top \Wkq \query\big)_{s=1}^{T-1}\Big)_{t-1}}_{=:a_t(\Wkq;\Ectx,\query)}
\ \WV\,\token_{t}.
\end{multlined}
\end{equation}
This is exactly the computation performed by an IH that reads a PTH mark: it
scores predecessors via $\token_s^\top \Wkq \query$ and copies the successor
token through the one-step shift 
\citep{Singh24InductionHeadFormation,Wang25LazyRichLearningIH}.

\paragraph{Objective: empirical MSE.}
We train with the empirical MSE loss, which is the objective used throughout the empirical study in \Cref{sec:expes}:
\begin{equation}
\label{eq:mse-minimal}
\MSE(\Wkq)
:=
\sum_{i=1}^n
\Big\|
\target^{(i)}-\model(\Wkq;\Ectx^{(i)},\query^{(i)})
\Big\|_2^2.
\end{equation}
We focus on MSE (rather than a cross-entropy over the match position $\CE(\Wkq)=-\sum_i \log a_{t_0^{(i)}+1}(\Wkq;\Ectx^{(i)},\query^{(i)})$
) because our
benchmark is a regression task 
so it is the natural objective. And more importantly, because there is no 
architecture-agnostic analogue of such a cross-entropy for non-attention models (no attention weights to interpret as match probabilities),
so MSE is the meaningful objective for the architecture comparisons in
\Cref{subsec:gmc-comparison-archis}.

\paragraph{Margins and the associated max-margin problem.}
For each sample $i$ and distractor position $t\neq t_0^{(i)}$, define the match
margin
\begin{equation}
\label{eq:margin}
\margin_{i,t}(\Wkq)
:=
\big(\token^{(i)}_{t_0^{(i)}}-\token^{(i)}_{t}\big)^\top \Wkq\,\query^{(i)}.
\end{equation}
Large positive margins imply $a_{t_0^{(i)}+1}(\Wkq;\Ectx^{(i)},\query^{(i)})\approx 1$
and thus $\model(\Wkq;\Ectx^{(i)},\query^{(i)})\approx \WV\token^{(i)}_{t_0^{(i)}+1}=\target^{(i)}$.

We define the max-margin separator by the 
SVM problem:
\begin{equation}
\label{eq:svm-minimal}
\begin{multlined}
\maxmargin
=
\argmin_{\Wkq}\ \frac12\|\Wkq\|_F^2
\\
\text{s.t. }
\margin_{i,t}(\Wkq)\ge 1
\ \ \forall i\in[n],\ \forall t\neq t_0^{(i)}.
\end{multlined}
\end{equation}

\paragraph{Empirical observations.} 
When we train \eqref{eq:mse-minimal} by gradient descent in this minimal setting,
we observe the following behavior in regimes where training attains very small loss and alignment is visible within our iteration budget:
\begin{enumerate}
\item $\MSE(\Wkq(\tau))$ decreases to a small value and the cumulative squared gradient norm remains bounded,
matching the success condition \eqref{eq:success-conditioning};
\item $\|\Wkq(\tau)\|_F$ grows over training;
\item the normalized direction $\Wkq(\tau)/\|\Wkq(\tau)\|_F$ becomes increasingly aligned with
$\maxmargin/\|\maxmargin\|_F$;
\item the margins $\min_{i,t\neq t_0^{(i)}}\margin_{i,t}(\Wkq(\tau))$ increase, and
attention becomes hard on the correct match.
\end{enumerate} 
An example is given in \Cref{subsec:assumptions-hold}. 
In other configurations, the cosine similarity between
$\Wkq(\tau)/\|\Wkq(\tau)\|_F$ and $\maxmargin/\|\maxmargin\|_F$
does not reach $1$ within the iteration budget.
While it reaches the range $[0.8,0.95]$ in general, 
it is unclear if it reflects slow convergence 
or if it converged toward a different asymptotic direction, 
hence settled that way. 

Recent work on implicit bias in attention models with monotone margin losses
shows that some regimes can converge to \emph{locally optimal} max-margin directions
rather than the globally optimal one, depending on data geometry and parameterization \cite{Tarzanagh23TransformersAsSVM}.
Whether analogous phenomena occur under MSE remains open.
In this work, we  focus on identifying and explaining regimes
where clear global max-margin alignment is observed within a finite horizon,
and on characterizing the properties of trajectories that lead to this behavior.

\subsection{A geometric event on GMC data}
\label{subsec:gmc-geometric-event}

A key difference between MSE and the usual monotone margin losses for max-margin implicit bias results \cite{Soudry18ImplicitBiasCE,JiTelgarsky19Nonseparable,TarzanaghLZO23MaxMarginTokenSelection,Tarzanagh23TransformersAsSVM,Vasudeva25ImplicitBiasRatesSelfAttention} (e.g., logistic
or exponential losses in binary classification, 
discussed in \Cref{subsec:mse-bias-ce}) 
is that MSE can admit finite-norm minimizers. 
Even when it does not, the margins need not keep increasing \emph{monotonically} along
the trajectory, as MSE gradients contain interaction terms absent in monotone
margin losses. 
Here we identify a \emph{geometric event} on the training set under which:
(i) there is \emph{no} finite-norm minimizer for \eqref{eq:mse-minimal}, and
(ii) any successful optimization trajectory must enter a regime where 
margins keep increasing \emph{monotonically}. 

\paragraph{The event.}
We define an explicit event $\mathcal{E}_{\mathrm{GMC}}$ on the $n$ sampled GMC
sequences. 
It is purely geometric, as it can be expressed as 
a set of inequalities involving the token Gram matrices and the 
difference vectors
$\{\token^{(i)}_{t_0^{(i)}}-\token^{(i)}_{t}\}_{i,t\neq t_0^{(i)}}$. 
We defer its full statement to 
\Cref{app:mse-bias} 
and only summarize its role here:
it rules out the existence of a finite-norm zero-loss interpolant
(\Cref{subsec:identifiability-margin-divergence}), 
and it enforces a favorable geometry under 
which the leading terms in the dynamics
promote increasing margins \emph{monotonically}  
in the successful regime (see the discussion above \Cref{lem:mse-asymptotics-dvg-margins}).

\paragraph{High-probability guarantee on GMC.}
We prove that $\mathcal{E}_{\mathrm{GMC}}$ holds with high probability under the
GMC sampling model (stated for the exact-match/exact-copy specialization considered here).
A precise statement and proof are given in \Cref{subsec:gaussian-data}.
The probability improves as the geometry becomes more overparameterized
in the sense relevant to match selection 
(informally: larger $\din$ relative to
the effective number of distractors). 
It justifies why GMC
instances \emph{typically} land in a regime 
where divergence and margin growth
are unavoidable, even though 
alternative behaviors are conceivable in other data
geometries. 

\subsection{MSE-specific technical assumptions}
\label{subsec:mse-technical-assumptions}

With $\mathcal{E}_{\rm geom}$ in place, the 
remaining obstacles are not about GMC geometry
but about the optimization dynamics induced by \eqref{eq:mse-minimal}. 
For margin-based classification losses (like logistic/exponential) \cite{Soudry18ImplicitBiasCE,JiTelgarsky19Nonseparable,TarzanaghLZO23MaxMarginTokenSelection,Tarzanagh23TransformersAsSVM,Vasudeva25ImplicitBiasRatesSelfAttention}, 
the gradient is a positive combination of separating directions, with \emph{fixed} coefficients 
multiplied by a monotone function of the margins. 
For MSE, the gradient is also a positive combination of separating directions (thanks to $\mathcal{E}_{\rm geom}$), keeping margins growing, 
but the coefficients depend on the $\ell^2$ residuals, 
and includes cross-terms induced by the square (this comparison is discussed in details in \Cref{rem:ce-comparison}). 
To recover a max-margin limit, this motivates the next assumption. We defer the precise statement to \Cref{app:mse-bias}. 

\begin{assumption}[Asymptotic stabilization of MSE/softmax interactions]
\label{ass:mse-stabilization}
Asymptotically along the trajectory, 
(i) the gradient coefficients 
that are not monotone functions of the margins 
stabilize asymptotically so that the dynamics become
effectively margin-driven, and
(ii) the cross-terms become negligible relative to the main margin terms so that the gradient becomes dominated by a stable set of support constraints. 
\end{assumption}

\Cref{ass:mse-stabilization} is the main MSE-specific hypothesis.
It replaces automatic simplifications enjoyed by
monotone margin losses: 
it rules out interference from MSE-specific cross-terms
and enforces a support-dominated asymptotic regime in which max-margin arguments can be adapted. 
In \Cref{subsec:assumptions-hold} we provide 
an empirical sanity check of these conditions on a representative run
where strong alignment is observed within the training budget.

\section{Conclusion}
\label{sec:conclusion}

GMC provides a simple and theory-friendly benchmark in which retrieval circuits emerge reliably,
separate architectures by their retrieval capabilities, and admit analytical study. 
In a minimal attention setting, we show that despite many \emph{a priori} possible regression
solutions, there exists a regime in which successful training is accompanied by a max-margin implicit bias that drives attention toward
hard match selection. 

Open questions include how this bias changes when optimizing the unmerged key-query
parameterization (e.g., toward a nuclear-norm max-margin solution \citep{Tarzanagh23TransformersAsSVM}), 
whether the analysis extends beyond the minimal setting, and what additional
ingredients are required to explain the sharp loss-drop dynamics observed in deeper Transformers.

\section*{Acknowledgements}

We thank Pierre Marion for insightful discussions.  
This work was supported in part by the Swiss 
State Secretariat for Education, Research and Innovation (SERI)
under contract number MB22.00027.

\bibliographystyle{plainnat}
\bibliography{references}

@article{Phuong22FormalTransformers,
  author       = {Mary Phuong and
                  Marcus Hutter},
  title        = {Formal Algorithms for Transformers},
  journal      = {CoRR},
  volume       = {abs/2207.09238},
  year         = {2022},
  url          = {https://doi.org/10.48550/arXiv.2207.09238},
  doi          = {10.48550/ARXIV.2207.09238},
  eprinttype    = {arXiv},
  eprint       = {2207.09238},
  bibsource    = {dblp computer science bibliography, https://dblp.org}
}

@inproceedings{Cho14GRU,
  author       = {Kyunghyun Cho and
                  Bart van Merrienboer and
                  {\c{C}}aglar G{\"{u}}l{\c{c}}ehre and
                  Dzmitry Bahdanau and
                  Fethi Bougares and
                  Holger Schwenk and
                  Yoshua Bengio},
  editor       = {Alessandro Moschitti and
                  Bo Pang and
                  Walter Daelemans},
  title        = {Learning Phrase Representations using {RNN} Encoder-Decoder for Statistical
                  Machine Translation},
  booktitle    = {Proceedings of the 2014 Conference on Empirical Methods in Natural
                  Language Processing, {EMNLP} 2014, October 25-29, 2014, Doha, Qatar,
                  {A} meeting of SIGDAT, a Special Interest Group of the {ACL}},
  pages        = {1724--1734},
  publisher    = {{ACL}},
  year         = {2014},
  url          = {https://doi.org/10.3115/v1/d14-1179},
  doi          = {10.3115/V1/D14-1179},
  bibsource    = {dblp computer science bibliography, https://dblp.org}
}

@inproceedings{Gu22S4StructuredStateSpaceSequenceModel,
  author       = {Albert Gu and
                  Karan Goel and
                  Christopher R{\'{e}}},
  title        = {Efficiently Modeling Long Sequences with Structured State Spaces},
  booktitle    = {The Tenth International Conference on Learning Representations, {ICLR}
                  2022, Virtual Event, April 25-29, 2022},
  publisher    = {OpenReview.net},
  year         = {2022},
  url          = {https://openreview.net/forum?id=uYLFoz1vlAC},
  bibsource    = {dblp computer science bibliography, https://dblp.org}
}

@inproceedings{Fu23H3,
  author       = {Daniel Y. Fu and
                  Tri Dao and
                  Khaled Kamal Saab and
                  Armin W. Thomas and
                  Atri Rudra and
                  Christopher R{\'{e}}},
  title        = {Hungry Hungry Hippos: Towards Language Modeling with State Space Models},
  booktitle    = {The Eleventh International Conference on Learning Representations,
                  {ICLR} 2023, Kigali, Rwanda, May 1-5, 2023},
  publisher    = {OpenReview.net},
  year         = {2023},
  url          = {https://openreview.net/forum?id=COZDy0WYGg},
  bibsource    = {dblp computer science bibliography, https://dblp.org}
}

@inproceedings{Marion22AttentionLayers,
  author       = {Pierre Marion and
                  Rapha{\"{e}}l Berthier and
                  G{\'{e}}rard Biau and
                  Claire Boyer},
  title        = {Attention layers provably solve single-location regression},
  booktitle    = {The Thirteenth International Conference on Learning Representations,
                  {ICLR} 2025, Singapore, April 24-28, 2025},
  publisher    = {OpenReview.net},
  year         = {2025},
  url          = {https://openreview.net/forum?id=DVlPp7Jd7P},
  bibsource    = {dblp computer science bibliography, https://dblp.org}
}

@article{Olsson22AnthropicInContextLearning,
  author       = {Catherine Olsson and
                  Nelson Elhage and
                  Neel Nanda and
                  Nicholas Joseph and
                  Nova DasSarma and
                  Tom Henighan and
                  Ben Mann and
                  Amanda Askell and
                  Yuntao Bai and
                  Anna Chen and
                  Tom Conerly and
                  Dawn Drain and
                  Deep Ganguli and
                  Zac Hatfield{-}Dodds and
                  Danny Hernandez and
                  Scott Johnston and
                  Andy Jones and
                  Jackson Kernion and
                  Liane Lovitt and
                  Kamal Ndousse and
                  Dario Amodei and
                  Tom Brown and
                  Jack Clark and
                  Jared Kaplan and
                  Sam McCandlish and
                  Chris Olah},
  title        = {In-context Learning and Induction Heads},
  journal      = {CoRR},
  volume       = {abs/2209.11895},
  year         = {2022},
  url          = {https://doi.org/10.48550/arXiv.2209.11895},
  doi          = {10.48550/ARXIV.2209.11895},
  eprinttype    = {arXiv},
  eprint       = {2209.11895},
  bibsource    = {dblp computer science bibliography, https://dblp.org}
}

@misc{CallumMcDougall23InductionHeadsIllustrated,
  author = {CallumMcDougall},
  title = {Induction Heads Illustrated},
  year = {2023},
  url = {https://www.lesswrong.com/posts/TvrfY4c9eaGLeyDkE/induction-heads-illustrated},
  note = {Accessed: 2026-01-15}
}

@inproceedings{Singh23TransienceICL,
  author       = {Aaditya K. Singh and
                  Stephanie C. Y. Chan and
                  Ted Moskovitz and
                  Erin Grant and
                  Andrew M. Saxe and
                  Felix Hill},
  editor       = {Alice Oh and
                  Tristan Naumann and
                  Amir Globerson and
                  Kate Saenko and
                  Moritz Hardt and
                  Sergey Levine},
  title        = {The Transient Nature of Emergent In-Context Learning in Transformers},
  booktitle    = {Advances in Neural Information Processing Systems 36: Annual Conference
                  on Neural Information Processing Systems 2023, NeurIPS 2023, New Orleans,
                  LA, USA, December 10 - 16, 2023},
  year         = {2023},
  bibsource    = {dblp computer science bibliography, https://dblp.org}
}

@inproceedings{Chan22DataDistributionDrivesInContextLearning,
  author       = {Stephanie C. Y. Chan and
                  Adam Santoro and
                  Andrew K. Lampinen and
                  Jane X. Wang and
                  Aaditya K. Singh and
                  Pierre H. Richemond and
                  James L. McClelland and
                  Felix Hill},
  editor       = {Sanmi Koyejo and
                  S. Mohamed and
                  A. Agarwal and
                  Danielle Belgrave and
                  K. Cho and
                  A. Oh},
  title        = {Data Distributional Properties Drive Emergent In-Context Learning
                  in Transformers},
  booktitle    = {Advances in Neural Information Processing Systems 35: Annual Conference
                  on Neural Information Processing Systems 2022, NeurIPS 2022, New Orleans,
                  LA, USA, November 28 - December 9, 2022},
  year         = {2022},
  bibsource    = {dblp computer science bibliography, https://dblp.org}
}

@article{Llama3,
  author       = {Llama Team},
  title        = {The Llama 3 Herd of Models},
  journal      = {CoRR},
  volume       = {abs/2407.21783},
  year         = {2024},
  url          = {https://doi.org/10.48550/arXiv.2407.21783},
  doi          = {10.48550/ARXIV.2407.21783},
  eprinttype    = {arXiv},
  eprint       = {2407.21783},
  bibsource    = {dblp computer science bibliography, https://dblp.org}
}

@article{Crosbie24AblatingIHDeterioratesICL,
  author       = {Joy Crosbie and
                  Ekaterina Shutova},
  title        = {Induction Heads as an Essential Mechanism for Pattern Matching in
                  In-context Learning},
  journal      = {CoRR},
  volume       = {abs/2407.07011},
  year         = {2024},
  url          = {https://doi.org/10.48550/arXiv.2407.07011},
  doi          = {10.48550/ARXIV.2407.07011},
  eprinttype    = {arXiv},
  eprint       = {2407.07011},
  bibsource    = {dblp computer science bibliography, https://dblp.org}
}

@inproceedings{Min22ICLwithRandomLabels,
  author       = {Sewon Min and
                  Xinxi Lyu and
                  Ari Holtzman and
                  Mikel Artetxe and
                  Mike Lewis and
                  Hannaneh Hajishirzi and
                  Luke Zettlemoyer},
  editor       = {Yoav Goldberg and
                  Zornitsa Kozareva and
                  Yue Zhang},
  title        = {Rethinking the Role of Demonstrations: What Makes In-Context Learning
                  Work?},
  booktitle    = {Proceedings of the 2022 Conference on Empirical Methods in Natural
                  Language Processing, {EMNLP} 2022, Abu Dhabi, United Arab Emirates,
                  December 7-11, 2022},
  pages        = {11048--11064},
  publisher    = {Association for Computational Linguistics},
  year         = {2022},
  url          = {https://doi.org/10.18653/v1/2022.emnlp-main.759},
  doi          = {10.18653/V1/2022.EMNLP-MAIN.759},
  bibsource    = {dblp computer science bibliography, https://dblp.org}
}

@article{Wang25LazyRichLearningIH,
  title={How Transformers Get Rich: Approximation and Dynamics Analysis},
  author={Wang, Mingze and Yu, Ruoxi and E, Weinan and Wu, Lei},
  journal={arXiv preprint arXiv:2410.11474},
  year={2025}
}

@inproceedings{Akyurek24ICLLanguageGenbyAutomata,
  author       = {Ekin Aky{\"{u}}rek and
                  Bailin Wang and
                  Yoon Kim and
                  Jacob Andreas},
  title        = {In-Context Language Learning: Architectures and Algorithms},
  booktitle    = {Forty-first International Conference on Machine Learning, {ICML} 2024,
                  Vienna, Austria, July 21-27, 2024},
  publisher    = {OpenReview.net},
  year         = {2024},
  url          = {https://openreview.net/forum?id=3Z9CRr5srL},
  bibsource    = {dblp computer science bibliography, https://dblp.org}
}

@inproceedings{Oswald23InContextGD,
  author       = {Johannes von Oswald and
                  Eyvind Niklasson and
                  Ettore Randazzo and
                  Jo{\~{a}}o Sacramento and
                  Alexander Mordvintsev and
                  Andrey Zhmoginov and
                  Max Vladymyrov},
  editor       = {Andreas Krause and
                  Emma Brunskill and
                  Kyunghyun Cho and
                  Barbara Engelhardt and
                  Sivan Sabato and
                  Jonathan Scarlett},
  title        = {Transformers Learn In-Context by Gradient Descent},
  booktitle    = {International Conference on Machine Learning, {ICML} 2023, 23-29 July
                  2023, Honolulu, Hawaii, {USA}},
  series       = {Proceedings of Machine Learning Research},
  volume       = {202},
  pages        = {35151--35174},
  publisher    = {{PMLR}},
  year         = {2023},
  url          = {https://proceedings.mlr.press/v202/von-oswald23a.html},
  bibsource    = {dblp computer science bibliography, https://dblp.org}
}

@inproceedings{Akyurek23InContextLearningLinearModelsviaGD,
  author       = {Ekin Aky{\"{u}}rek and
                  Dale Schuurmans and
                  Jacob Andreas and
                  Tengyu Ma and
                  Denny Zhou},
  title        = {What learning algorithm is in-context learning? Investigations with
                  linear models},
  booktitle    = {The Eleventh International Conference on Learning Representations,
                  {ICLR} 2023, Kigali, Rwanda, May 1-5, 2023},
  publisher    = {OpenReview.net},
  year         = {2023},
  url          = {https://openreview.net/forum?id=0g0X4H8yN4I},
  bibsource    = {dblp computer science bibliography, https://dblp.org}
}

@inproceedings{Garg22InContextCaseStudy,
  author       = {Shivam Garg and
                  Dimitris Tsipras and
                  Percy Liang and
                  Gregory Valiant},
  editor       = {Sanmi Koyejo and
                  S. Mohamed and
                  A. Agarwal and
                  Danielle Belgrave and
                  K. Cho and
                  A. Oh},
  title        = {What Can Transformers Learn In-Context? {A} Case Study of Simple Function
                  Classes},
  booktitle    = {Advances in Neural Information Processing Systems 35: Annual Conference
                  on Neural Information Processing Systems 2022, NeurIPS 2022, New Orleans,
                  LA, USA, November 28 - December 9, 2022},
  year         = {2022},
  bibsource    = {dblp computer science bibliography, https://dblp.org}
}

@inproceedings{Bhattamishra23ICLBooleanFunctions,
  author       = {Satwik Bhattamishra and
                  Arkil Patel and
                  Phil Blunsom and
                  Varun Kanade},
  title        = {Understanding In-Context Learning in Transformers and LLMs by Learning
                  to Learn Discrete Functions},
  booktitle    = {The Twelfth International Conference on Learning Representations,
                  {ICLR} 2024, Vienna, Austria, May 7-11, 2024},
  publisher    = {OpenReview.net},
  year         = {2024},
  url          = {https://openreview.net/forum?id=ekeyCgeRfC},
  bibsource    = {dblp computer science bibliography, https://dblp.org}
}

@inproceedings{Lee24ICLwithOtherArchitectures,
  author       = {Ivan Lee and
                  Nan Jiang and
                  Taylor Berg{-}Kirkpatrick},
  title        = {Is attention required for ICL? Exploring the Relationship Between
                  Model Architecture and In-Context Learning Ability},
  booktitle    = {The Twelfth International Conference on Learning Representations,
                  {ICLR} 2024, Vienna, Austria, May 7-11, 2024},
  publisher    = {OpenReview.net},
  year         = {2024},
  url          = {https://openreview.net/forum?id=Qwq4cpLtoX},
  bibsource    = {dblp computer science bibliography, https://dblp.org}
}

@inproceedings{Dai23GPTMetaOptimization,
  author       = {Damai Dai and
                  Yutao Sun and
                  Li Dong and
                  Yaru Hao and
                  Shuming Ma and
                  Zhifang Sui and
                  Furu Wei},
  editor       = {Anna Rogers and
                  Jordan L. Boyd{-}Graber and
                  Naoaki Okazaki},
  title        = {Why Can {GPT} Learn In-Context? Language Models Secretly Perform Gradient
                  Descent as Meta-Optimizers},
  booktitle    = {Findings of the Association for Computational Linguistics: {ACL} 2023,
                  Toronto, Canada, July 9-14, 2023},
  pages        = {4005--4019},
  publisher    = {Association for Computational Linguistics},
  year         = {2023},
  url          = {https://doi.org/10.18653/v1/2023.findings-acl.247},
  doi          = {10.18653/V1/2023.FINDINGS-ACL.247},
  bibsource    = {dblp computer science bibliography, https://dblp.org}
}

@inproceedings{Reddy24GaussianDataforMechanisticStudyClassificationICL,
  author       = {Gautam Reddy},
  title        = {The mechanistic basis of data dependence and abrupt learning in an
                  in-context classification task},
  booktitle    = {The Twelfth International Conference on Learning Representations,
                  {ICLR} 2024, Vienna, Austria, May 7-11, 2024},
  publisher    = {OpenReview.net},
  year         = {2024},
  url          = {https://openreview.net/forum?id=aN4Jf6Cx69},
  bibsource    = {dblp computer science bibliography, https://dblp.org}
}

@inproceedings{Singh24InductionHeadFormation,
  author       = {Aaditya K. Singh and
                  Ted Moskovitz and
                  Felix Hill and
                  Stephanie C. Y. Chan and
                  Andrew M. Saxe},
  title        = {What needs to go right for an induction head? {A} mechanistic study
                  of in-context learning circuits and their formation},
  booktitle    = {Forty-first International Conference on Machine Learning, {ICML} 2024,
                  Vienna, Austria, July 21-27, 2024},
  publisher    = {OpenReview.net},
  year         = {2024},
  url          = {https://openreview.net/forum?id=O8rrXl71D5},
  bibsource    = {dblp computer science bibliography, https://dblp.org}
}

@inproceedings{JiTelgarsky19Nonseparable,
  author       = {Ziwei Ji and
                  Matus Telgarsky},
  editor       = {Alina Beygelzimer and
                  Daniel Hsu},
  title        = {The implicit bias of gradient descent on nonseparable data},
  booktitle    = {Conference on Learning Theory, {COLT} 2019, 25-28 June 2019, Phoenix,
                  AZ, {USA}},
  series       = {Proceedings of Machine Learning Research},
  volume       = {99},
  pages        = {1772--1798},
  publisher    = {{PMLR}},
  year         = {2019},
  url          = {http://proceedings.mlr.press/v99/ji19a.html},
  bibsource    = {dblp computer science bibliography, https://dblp.org}
}

@article{Sanford24_1LayerTransformer_arent_Induction_Head,
  author       = {Clayton Sanford and
                  Daniel Hsu and
                  Matus Telgarsky},
  title        = {One-layer transformers fail to solve the induction heads task},
  journal      = {CoRR},
  volume       = {abs/2408.14332},
  year         = {2024},
  url          = {https://doi.org/10.48550/arXiv.2408.14332},
  doi          = {10.48550/ARXIV.2408.14332},
  eprinttype    = {arXiv},
  eprint       = {2408.14332},
  bibsource    = {dblp computer science bibliography, https://dblp.org}
}

@inproceedings{Bietti23TransformerMemoryViewpoint,
  author       = {Alberto Bietti and
                  Vivien Cabannes and
                  Diane Bouchacourt and
                  Herv{\'{e}} J{\'{e}}gou and
                  L{\'{e}}on Bottou},
  editor       = {Alice Oh and
                  Tristan Naumann and
                  Amir Globerson and
                  Kate Saenko and
                  Moritz Hardt and
                  Sergey Levine},
  title        = {Birth of a Transformer: {A} Memory Viewpoint},
  booktitle    = {Advances in Neural Information Processing Systems 36: Annual Conference
                  on Neural Information Processing Systems 2023, NeurIPS 2023, New Orleans,
                  LA, USA, December 10 - 16, 2023},
  year         = {2023},
  bibsource    = {dblp computer science bibliography, https://dblp.org}
}

@inproceedings{Poli23HyenaHierarchy,
  author       = {Michael Poli and
                  Stefano Massaroli and
                  Eric Nguyen and
                  Daniel Y. Fu and
                  Tri Dao and
                  Stephen Baccus and
                  Yoshua Bengio and
                  Stefano Ermon and
                  Christopher R{\'{e}}},
  editor       = {Andreas Krause and
                  Emma Brunskill and
                  Kyunghyun Cho and
                  Barbara Engelhardt and
                  Sivan Sabato and
                  Jonathan Scarlett},
  title        = {Hyena Hierarchy: Towards Larger Convolutional Language Models},
  booktitle    = {International Conference on Machine Learning, {ICML} 2023, 23-29 July
                  2023, Honolulu, Hawaii, {USA}},
  series       = {Proceedings of Machine Learning Research},
  volume       = {202},
  pages        = {28043--28078},
  publisher    = {{PMLR}},
  year         = {2023},
  url          = {https://proceedings.mlr.press/v202/poli23a.html},
  bibsource    = {dblp computer science bibliography, https://dblp.org}
}

@article{Graves14NeuralTuringMachines,
  author       = {Alex Graves and
                  Greg Wayne and
                  Ivo Danihelka},
  title        = {Neural Turing Machines},
  journal      = {CoRR},
  volume       = {abs/1410.5401},
  year         = {2014},
  url          = {http://arxiv.org/abs/1410.5401},
  eprinttype    = {arXiv},
  eprint       = {1410.5401},
  bibsource    = {dblp computer science bibliography, https://dblp.org}
}

@inproceedings{Ba16FastWeightsAttendRecentPast,
  author       = {Jimmy Ba and
                  Geoffrey E. Hinton and
                  Volodymyr Mnih and
                  Joel Z. Leibo and
                  Catalin Ionescu},
  editor       = {Daniel D. Lee and
                  Masashi Sugiyama and
                  Ulrike von Luxburg and
                  Isabelle Guyon and
                  Roman Garnett},
  title        = {Using Fast Weights to Attend to the Recent Past},
  booktitle    = {Advances in Neural Information Processing Systems 29: Annual Conference
                  on Neural Information Processing Systems 2016, December 5-10, 2016,
                  Barcelona, Spain},
  pages        = {4331--4339},
  year         = {2016},
  url          = {https://proceedings.neurips.cc/paper/2016/hash/9f44e956e3a2b7b5598c625fcc802c36-Abstract.html},
  bibsource    = {dblp computer science bibliography, https://dblp.org}
}

@article{Tarzanagh23TransformersAsSVM,
  author       = {Davoud Ataee Tarzanagh and
                  Yingcong Li and
                  Christos Thrampoulidis and
                  Samet Oymak},
  title        = {Transformers as Support Vector Machines},
  journal      = {CoRR},
  volume       = {abs/2308.16898},
  year         = {2023},
  url          = {https://doi.org/10.48550/arXiv.2308.16898},
  doi          = {10.48550/ARXIV.2308.16898},
  eprinttype    = {arXiv},
  eprint       = {2308.16898},
  bibsource    = {dblp computer science bibliography, https://dblp.org}
}

@article{Vasudeva25ImplicitBiasRatesSelfAttention,
  author       = {Bhavya Vasudeva and
                  Puneesh Deora and
                  Christos Thrampoulidis},
  title        = {Implicit Bias and Fast Convergence Rates for Self-attention},
  journal      = {Trans. Mach. Learn. Res.},
  volume       = {2025},
  year         = {2025},
  url          = {https://openreview.net/forum?id=pKilnjQsb0},
  bibsource    = {dblp computer science bibliography, https://dblp.org}
}

@inproceedings{TarzanaghLZO23MaxMarginTokenSelection,
  author       = {Davoud Ataee Tarzanagh and
                  Yingcong Li and
                  Xuechen Zhang and
                  Samet Oymak},
  editor       = {Alice Oh and
                  Tristan Naumann and
                  Amir Globerson and
                  Kate Saenko and
                  Moritz Hardt and
                  Sergey Levine},
  title        = {Max-Margin Token Selection in Attention Mechanism},
  booktitle    = {Advances in Neural Information Processing Systems 36: Annual Conference
                  on Neural Information Processing Systems 2023, NeurIPS 2023, New Orleans,
                  LA, USA, December 10 - 16, 2023},
  year         = {2023},
  bibsource    = {dblp computer science bibliography, https://dblp.org}
}

@inproceedings{Edelman24OnetoTwoGramMarkovLearning,
  author       = {Ezra Edelman and
                  Nikolaos Tsilivis and
                  Benjamin L. Edelman and
                  Eran Malach and
                  Surbhi Goel},
  editor       = {Amir Globersons and
                  Lester Mackey and
                  Danielle Belgrave and
                  Angela Fan and
                  Ulrich Paquet and
                  Jakub M. Tomczak and
                  Cheng Zhang},
  title        = {The Evolution of Statistical Induction Heads: In-Context Learning
                  Markov Chains},
  booktitle    = {Advances in Neural Information Processing Systems 38: Annual Conference
                  on Neural Information Processing Systems 2024, NeurIPS 2024, Vancouver,
                  BC, Canada, December 10 - 15, 2024},
  year         = {2024},
  bibsource    = {dblp computer science bibliography, https://dblp.org}
}

@inproceedings{Makkuva24AttentionMarkov,
  author       = {Ashok Vardhan Makkuva and
                  Marco Bondaschi and
                  Adway Girish and
                  Alliot Nagle and
                  Martin Jaggi and
                  Hyeji Kim and
                  Michael Gastpar},
  title        = {Attention with Markov: {A} Curious Case of Single-layer Transformers},
  booktitle    = {The Thirteenth International Conference on Learning Representations,
                  {ICLR} 2025, Singapore, April 24-28, 2025},
  publisher    = {OpenReview.net},
  year         = {2025},
  url          = {https://openreview.net/forum?id=SqZ0KY4qBD},
  bibsource    = {dblp computer science bibliography, https://dblp.org}
}

@inproceedings{Rajaraman24ApproximationNGramsByTransformers,
  author       = {Nived Rajaraman and
                  Marco Bondaschi and
                  Ashok Vardhan Makkuva and
                  Kannan Ramchandran and
                  Michael Gastpar},
  editor       = {Amir Globersons and
                  Lester Mackey and
                  Danielle Belgrave and
                  Angela Fan and
                  Ulrich Paquet and
                  Jakub M. Tomczak and
                  Cheng Zhang},
  title        = {Transformers on Markov data: Constant depth suffices},
  booktitle    = {Advances in Neural Information Processing Systems 38: Annual Conference
                  on Neural Information Processing Systems 2024, NeurIPS 2024, Vancouver,
                  BC, Canada, December 10 - 15, 2024},
  year         = {2024},
  bibsource    = {dblp computer science bibliography, https://dblp.org}
}

@inproceedings{Nichani24OneGramTaskLearnsIH,
  author       = {Eshaan Nichani and
                  Alex Damian and
                  Jason D. Lee},
  title        = {How Transformers Learn Causal Structure with Gradient Descent},
  booktitle    = {Forty-first International Conference on Machine Learning, {ICML} 2024,
                  Vienna, Austria, July 21-27, 2024},
  publisher    = {OpenReview.net},
  year         = {2024},
  url          = {https://openreview.net/forum?id=jNM4imlHZv},
  bibsource    = {dblp computer science bibliography, https://dblp.org}
}

@inproceedings{Xie22InContextLearningImplicitBayesianInference,
  author       = {Sang Michael Xie and
                  Aditi Raghunathan and
                  Percy Liang and
                  Tengyu Ma},
  title        = {An Explanation of In-context Learning as Implicit Bayesian Inference},
  booktitle    = {The Tenth International Conference on Learning Representations, {ICLR}
                  2022, Virtual Event, April 25-29, 2022},
  publisher    = {OpenReview.net},
  year         = {2022},
  url          = {https://openreview.net/forum?id=RdJVFCHjUMI},
  bibsource    = {dblp computer science bibliography, https://dblp.org}
}

@inproceedings{Chen24nGramTaskLearnsIH,
  author       = {Siyu Chen and
                  Heejune Sheen and
                  Tianhao Wang and
                  Zhuoran Yang},
  editor       = {Amir Globersons and
                  Lester Mackey and
                  Danielle Belgrave and
                  Angela Fan and
                  Ulrich Paquet and
                  Jakub M. Tomczak and
                  Cheng Zhang},
  title        = {Unveiling Induction Heads: Provable Training Dynamics and Feature
                  Learning in Transformers},
  booktitle    = {Advances in Neural Information Processing Systems 38: Annual Conference
                  on Neural Information Processing Systems 2024, NeurIPS 2024, Vancouver,
                  BC, Canada, December 10 - 15, 2024},
  year         = {2024},
  bibsource    = {dblp computer science bibliography, https://dblp.org}
}

@article{Elhage21TransformerCircuits,
   title={A Mathematical Framework for Transformer Circuits},
   author={Elhage, Nelson and Nanda, Neel and Olsson, Catherine and Henighan, Tom and Joseph, Nicholas and Mann, Ben and Askell, Amanda and Bai, Yuntao and Chen, Anna and Conerly, Tom and DasSarma, Nova and Drain, Dawn and Ganguli, Deep and Hatfield-Dodds, Zac and Hernandez, Danny and Jones, Andy and Kernion, Jackson and Lovitt, Liane and Ndousse, Kamal and Amodei, Dario and Brown, Tom and Clark, Jack and Kaplan, Jared and McCandlish, Sam and Olah, Chris},
   year={2021},
   journal={Transformer Circuits Thread},
   url={https://transformer-circuits.pub/2021/framework/index.html}
}

@article{Musat25EmergenceInductionHeadsICL,
  author       = {Tiberiu Musat and
                  Tiago Pimentel and
                  Lorenzo Noci and
                  Alessandro Stolfo and
                  Mrinmaya Sachan and
                  Thomas Hofmann},
  title        = {On the Emergence of Induction Heads for In-Context Learning},
  journal      = {CoRR},
  volume       = {abs/2511.01033},
  year         = {2025},
  url          = {https://doi.org/10.48550/arXiv.2511.01033},
  doi          = {10.48550/ARXIV.2511.01033},
  eprinttype    = {arXiv},
  eprint       = {2511.01033},
  bibsource    = {dblp computer science bibliography, https://dblp.org}
}

@inproceedings{Varre25LearningInContextNGrams,
  author       = {Aditya Varre and
                  Gizem Y{\"{u}}ce and
                  Nicolas Flammarion},
  title        = {Learning In-context n-grams with Transformers: Sub-n-grams Are Near-Stationary
                  Points},
  booktitle    = {Forty-second International Conference on Machine Learning, {ICML}
                  2025, Vancouver, BC, Canada, July 13-19, 2025},
  publisher    = {OpenReview.net},
  year         = {2025},
  url          = {https://openreview.net/forum?id=OMwdvGDeHL},
  bibsource    = {dblp computer science bibliography, https://dblp.org}
}

@article{Soudry18ImplicitBiasCE,
  author       = {Daniel Soudry and
                  Elad Hoffer and
                  Mor Shpigel Nacson and
                  Suriya Gunasekar and
                  Nathan Srebro},
  title        = {The Implicit Bias of Gradient Descent on Separable Data},
  journal      = {J. Mach. Learn. Res.},
  volume       = {19},
  pages        = {70:1--70:57},
  year         = {2018},
  url          = {https://jmlr.org/papers/v19/18-188.html},
  bibsource    = {dblp computer science bibliography, https://dblp.org}
}

@article{Patro24Mamba360,
  author       = {Badri Narayana Patro and
                  Vijay Srinivas Agneeswaran},
  title        = {Mamba-360: Survey of State Space Models as Transformer Alternative
                  for Long Sequence Modelling: Methods, Applications, and Challenges},
  journal      = {CoRR},
  volume       = {abs/2404.16112},
  year         = {2024},
  url          = {https://doi.org/10.48550/arXiv.2404.16112},
  doi          = {10.48550/ARXIV.2404.16112},
  eprinttype    = {arXiv},
  eprint       = {2404.16112},
  bibsource    = {dblp computer science bibliography, https://dblp.org}
}

@inproceedings{Jelassi24TransformersVsSSMAtCopying,
  author       = {Samy Jelassi and
                  David Brandfonbrener and
                  Sham M. Kakade and
                  Eran Malach},
  title        = {Repeat After Me: Transformers are Better than State Space Models at
                  Copying},
  booktitle    = {Forty-first International Conference on Machine Learning, {ICML} 2024,
                  Vienna, Austria, July 21-27, 2024},
  publisher    = {OpenReview.net},
  year         = {2024},
  url          = {https://openreview.net/forum?id=duRRoGeoQT},
  bibsource    = {dblp computer science bibliography, https://dblp.org}
}

@inproceedings{Arora24ZoologyRecall,
  author       = {Simran Arora and
                  Sabri Eyuboglu and
                  Aman Timalsina and
                  Isys Johnson and
                  Michael Poli and
                  James Zou and
                  Atri Rudra and
                  Christopher R{\'{e}}},
  title        = {Zoology: Measuring and Improving Recall in Efficient Language Models},
  booktitle    = {The Twelfth International Conference on Learning Representations,
                  {ICLR} 2024, Vienna, Austria, May 7-11, 2024},
  publisher    = {OpenReview.net},
  year         = {2024},
  url          = {https://openreview.net/forum?id=LY3ukUANko},
  bibsource    = {dblp computer science bibliography, https://dblp.org}
}

@article{GPT2,
  author       = {Alec Radford and
                  Jeff Wu and
                  Rewon Child and
                  David Luan and
                  Dario Amodei and
                  Ilya Sutskever},
  title        = {Language Models are Unsupervised Multitask Learners},
  journal      = {OpenAI Blog},
  year         = {2019},
  url          = {https://cdn.openai.com/better-language-models/language_models_are_unsupervised_multitask_learners.pdf},
}

@article{OmniglotDataset,
    author = {Brenden M. Lake  and Ruslan Salakhutdinov  and Joshua B. Tenenbaum },
    title = {Human-level concept learning through probabilistic program induction},
    journal = {Science},
    volume = {350},
    number = {6266},
    pages = {1332-1338},
    year = {2015},
    doi = {10.1126/science.aab3050},
    URL = {https://www.science.org/doi/abs/10.1126/science.aab3050},
    eprint = {https://www.science.org/doi/pdf/10.1126/science.aab3050},
}

\newpage
\appendix
\onecolumn

\section{Experimental Details for Gaussian Match-and-Copy}
\label{app:experiments}

This section provides the implementation details, hyperparameters, 
and additional results for the GMC 
experiments discussed in \Cref{sec:expes}.

\subsection{Computational Resources}
\label{app:computational-resources}
All experiments were conducted 
on a single NVIDIA RTX~4050 GPU (Ubuntu~24.04, CUDA~12.4) 
using PyTorch~2.6.0 and HuggingFace Transformers~4.49.0. 
We deliberately selected hyperparameters to enable rapid experimentation--a 1000-step 
run on the average dimensions tested 
completes in approximately one minute--while 
preserving qualitative behaviors of interest 
that are known to emerge at larger scales. 
This design allows us to replicate core properties 
of the PTH$\to$IH circuit 
in a simpler environment 
(e.g., 
Gaussian data rather than natural language),  
while ensuring that the benchmark remains computationally accessible.

\subsection{Task and Model Configurations}
\label{app:mc-experimental-setup}

\Cref{tab:mc-experiment-specs} summarizes 
all task and model configurations used by default in \Cref{sec:expes}, 
unless stated otherwise.
We now discuss \emph{how} each choice was made. 

\paragraph{Data Generation (\Cref{tab:mc-experiment-specs}).}
The GMC task parameters (\Cref{def:mc}) are chosen to balance difficulty and computational efficiency.
\begin{itemize}
    \item \textbf{Dimensions:} We set $\din = \dout$ for simplicity. 
    We report results for $\din \in \{16, 32\}$ and sequence lengths $T \in \{8, 16\}$. 
    We verified that results hold for larger dimensions (tested up to $\din=1024, T=1024$), 
    but we adhere to the small-dimensionality regime to highlight a minimal setting where 
    the task retains key properties of larger models while allowing for rapid iteration.
    
    \item \textbf{Token and Value Generation:} 
    Context tokens are sampled as 
    $\token_t \sim \mathcal{N}(0, \frac{\stdtoken^2}{\din} \I_{\din})$ 
    with $\stdtoken = 1$. 
    This scaling ensures that the token norm 
    $\|\token_t\|_2$ remains roughly constant 
    regardless of the dimension $\din$. 
    We also experimented with $\token_t, \query \sim \mathcal{N}(0, \I_{\din})$, 
    which produced similar results but led 
    to sharper loss drops at 
    the emergence of the PTH and IH, likely because 
    the larger token norms 
    increased the correlation between 
    the hidden match and the query, making it easier to detect. 
    We stick to the former scaling for consistency across dimensions. 
    The value matrix $\WV$ is fixed for a given task instance, 
    with entries drawn i.i.d.\ from $\mathcal{N}(0,1)$. 
    Finally, the noise distribution $\noiseDistrib$ in \Cref{def:mc} 
    is chosen such that $\query\sim \mathcal{N}(0, \frac{\stdtoken^2}{\din}\I_{\din})$, 
    i.e., the same marginal distribution as the other tokens. 
    Concretely, 
    we set $\query = \Cqe \token_{t_0} + S \xi$, 
    where $\xi \sim \mathcal{N}(0, \frac{\stdtoken^2}{\din}\I_{\din})$ 
    is independent noise and $S$ satisfies $SS^\top = \I_{\din} - \Cqe\Cqe^\top$. 

    \item \textbf{Signal-to-Noise Ratio (SNR):} 
    The target is defined as $\target = \WV \token_{t_0+1} + \eps$, 
    where $\eps \sim \mathcal{N}(0, \stdnoise^2 \I_{\dout})$ 
    with $\stdnoise=0.1$. 
    Given the definitions above, 
    the signal component $\WV \token_{t_0+1}$ 
    has coordinate-wise variance 
    $\din \cdot 1 \cdot (\stdtoken^2/\din) = \stdtoken^2 = 1$. 
    The noise component has variance $\stdnoise^2 = 0.01$. 
    This yields a signal-to-noise variance ratio of roughly $100$. 
    
    \item \textbf{Correlations:} 
    The covariance matrix is chosen isotropic $\Cqe = \frac{1}{c^2}\I_{\din}$ for simplicity. 
    Since $\Cov(\query, \token_{t_0}) \propto \Cqe$, the parameter $c$ controls difficulty; 
    higher $c$ reduces the correlation between the query and the match. 
    We use $c \in \{1.1, 1.2, 1.3\}$ to span different difficulty levels. 
    Indeed, for the dimensions we used, we found that 
      $c=1.1$ and $c=1.3$ were good choices to ensure 
      that
      (1) all considered models were able to learn the task, and
      (2) the training curves differed enough for these 
      two values of $c$ to consider them as different difficulty levels. 
      Note however that the task difficulty is not 
      only controlled by $c$ but also by the 
      sequence length $T$ and the noise level $\sigma_\eps$, 
      as can be seen from \Cref{fig:sweep-llama3-gpt2-metrics} 
      where larger $T$ and $c$ slows down the learning process.
    
    \item \textbf{Test Set:} We evaluate models on a fixed test 
    set of 128 fresh GMC sequences. 
    These are generated using the same parameters ($\din,\dout,T$, $\WV, \Cqe$, $\stdtoken, \stdnoise$)  
    as the training set, 
    but with different random seeds for the tokens, 
    hidden match indices, 
    and noise. 
\end{itemize}

\paragraph{Transformer Architectures (\Cref{tab:mc-experiment-specs}).}
We train GPT-2 \citep{GPT2} and Llama 3 \citep{Llama3} architectures 
initialized with the default HuggingFace scheme 
(weights $\sim \mathcal{N}(0, 0.02^2)$). 
Unless stated otherwise, we use 2-layer models, 
as this is the minimal depth required to implement a full induction circuit 
(PTH in layer 0, IH in layer 1) 
\citep{Elhage21TransformerCircuits,Olsson22AnthropicInContextLearning,Sanford24_1LayerTransformer_arent_Induction_Head}. 
Hidden and head sizes are adapted from \citet{Lee24ICLwithOtherArchitectures}, 
as summarized in \Cref{tab:mc-experiment-specs}.

\paragraph{Optimization and Robustness (\Cref{tab:mc-experiment-specs}).}
We use AdamW ($\beta_1=0.9, \beta_2=0.95$) with a batch size of 512 and a 
learning rate of $10^{-3}$ (with 200 steps warmup and cosine decay). 
We selected these hyperparameters after performing sweeps over:
\begin{itemize}
    \item \textbf{Learning rates:} Range $10^{-1}$ to $10^{-5}$. 
    Larger rates caused instability; 
    smaller rates shifted the loss drop 
    later without changing the qualitative outcome.
    \item \textbf{Batch sizes:} Range $32$ to $1024$. 
    512 offered the best tradeoff between training speed and curve smoothness.
    \item \textbf{Optimizers:} AdamW consistently outperformed SGD and 
    standard Adam in preliminary tests.
\end{itemize}
We found the emergence of the PTH$\to$IH circuit 
to be robust to these optimization choices: 
suboptimal hyperparameters typically led to non-convergence, 
but whenever the model solved the task, 
it did so via the same phase-transition mechanism described in \Cref{sec:expes}.

\begin{table}[h]
\centering
\caption{Hyperparameters for standard GMC experiments in \Cref{sec:expes}.}
\label{tab:mc-experiment-specs}
\begin{tabular}{ll}
\toprule
\textbf{Parameter} & \textbf{Value} \\
\midrule
\textit{GMC Task} & \\
Context length ($T$) & $\{8, 16\}$ \\
Dimension ($\din=\dout$) & $\{16, 32\}$ \\
Covariance ($\Cqe$) & $\tfrac{1}{c^2} \I_{\din}$ \\
Inverse correlation strength ($c$) & $\{1.1, 1.2, 1.3\}$ \\
Token standard deviation ($\stdtoken$) & 1 \\
Noise standard deviation ($\stdnoise$) & 0.1 \\
\midrule
\textit{Model Architecture} & \\
Type & \{Llama 3, GPT-2\} \\
Hidden dimension ($d_{\mathrm{model}}$) & 288 \\
Head dimension ($d_{\mathrm{head}}$) & 32 \\
Number of heads & 9 \\
FFN dimension & $4 \times d_{\mathrm{model}} = 1152$ \\
Layers ($n_\mathrm{layer}$) & $\{2, 4, 8, 16, 32\}$ (default: 2) \\
\midrule
\textit{Optimization} & \\
Optimizer & AdamW ($\beta_1=0.9, \beta_2=0.95$) \\
Learning rate & $10^{-3}$ \\
Scheduler & Linear warmup (200 steps) + Cosine decay \\
Batch size & 512 \\
Training steps & 1000 \\
\bottomrule
\end{tabular}
\end{table}
\subsection{Head Detection Metrics}
\label{app:pth-ih}

To track the emergence of mechanistic circuits, 
we compute three attention-pattern scores for every head in the model. 
Unless specified otherwise,  
we report the \textbf{maximum} score across all heads in a given model 
(e.g., \Cref{fig:phase-change-llama3} 
or \Cref{fig:sweep-llama3-gpt2-metrics}), 
as the emergence of a single specialized head is sufficient to implement the mechanism.

For a head with attention weights $a_{t,s}$ (query position $t$ attending 
to key position $s$), we define:
\begin{enumerate}
\item \textbf{PTH Score:} Measures attention to the immediate predecessor 
\citep{Elhage21TransformerCircuits}.
\[
\text{PTH} = \frac{1}{T} \sum_{t=2}^{T+1} a_{t, t-1} \in [0, 1].
\]
\item \textbf{IH-MC Score:} A metric tailored to GMC (\Cref{def:mc}) 
that measures attention from the query to the token \emph{following} the hidden match $t_0$.
\[
\text{IH-MC} = a_{\text{query}, t_0+1} \in [0, 1].
\]
\item \textbf{IH-Repeat Score:} 
A standard metric 
\citep{Olsson22AnthropicInContextLearning} 
using a duplicated sequence 
$(\token_1, \dots, \token_K, \token_1, \dots, \token_K)$. 
Unless stated otherwise, we use $K=T$. 
It measures attention from the second occurrence of a 
token to the successor of its first occurrence.
\[
\text{IH-Repeat} = \frac{1}{K} \sum_{t=1}^{K} a_{K+t, t+1} \in [0, 1].
\]
\end{enumerate} 

\paragraph{Typical score values at initialization.}
With Gaussian-initialized projection matrices the attention weights are
almost uniform at initialization. 
Hence $a_{t,s} \approx \frac{1}{t-1}$ for $s < t$ (causal attention) 
and $a_{t,s} \approx 0$ for $s \ge t$. 
Denote by $H_t = \sum_{k=1}^t \frac{1}{k} \approx \ln t +O(1)$ the $t$-th harmonic number. 
We deduce that the expected scores at initialization for each head are:
\begin{itemize}
    \item $\mathbb{E}[\text{PTH}] \approx \frac{1}{T}\sum_{t=2}^{T+1} \frac{1}{t-1} 
    = \frac{H_T}{T} \approx \frac{\ln T}{T}$,
    \item $\mathbb{E}[\text{IH-MC}] \approx \frac{1}{T}$,
    \item $\mathbb{E}[\text{IH-Repeat}] \approx \frac{1}{K} \sum_{t=1}^{K} \frac{1}{K+t-1} = 
    \frac{1}{K} \sum_{j=K}^{2K-1} \frac{1}{j} =
    \frac{H_{2K-1} - H_{K-1}}{K} \approx \frac{\ln 2}{K}$.
\end{itemize}
All our experiments 
roughly confirm 
these values at step~$0$ 
(e.g.\ \Cref{fig:phase-change-llama3}).

\subsection{Robustness of Emergence}
\label{app:robustness-emergence}
Besides the robustness to optimization 
choices discussed in \Cref{app:mc-experimental-setup}, 
we verified that the emergence of 
induction circuits is robust across 
different GMC configurations. 

\paragraph{Robustness to GMC Parameters.}
As noted in \Cref{subsec:MCviaPTH+IH}, 
we found that across GMC settings 
where the task is successfully solved,
the model consistently develops
PTH and IH circuits coinciding with a sharp drop in loss. 
The only difference lies in the timing and speed of this transition: 
harder tasks (higher $c$ and ratio $T/\din$) 
lead to later and more gradual transitions. 
\Cref{fig:sweep-llama3-gpt2-metrics} 
shows examples of that for 
different correlation strengths ($1/c^2$), 
dimensions ($\din$) 
context lengths ($T$), 
and Llama 3 and GPT-2 architectures. 

\paragraph{Robustness to Depth.} 
We ran experiments for $n_{\text{layer}}\in\{2,4,8,16,32\}$ 
(see \Cref{fig:sweep-llama3-gpt2-metrics} and \Cref{fig:sweep-llama3-layers}). 
All configurations with $n_{\text{layer}} \ge 2$ led to similar behavior: 
an induction circuit emerges coinciding with a drop in loss. 
We also observe that single-layer models ($n_{\text{layer}}=1$) 
never solve the task, 
even though they attempt to develop primitive PTHs or IHs 
(see \Cref{fig:sweep-llama3-layers}). 
This aligns with the mechanistic understanding that 
PTHs and IHs must reside in \emph{consecutive} layers 
to interact effectively (composition of heads), 
thus requiring at least two layers for the full circuit 
implementing match-and-copy. 

\begin{figure}[ht]
\centering
\includegraphics[width=\textwidth]{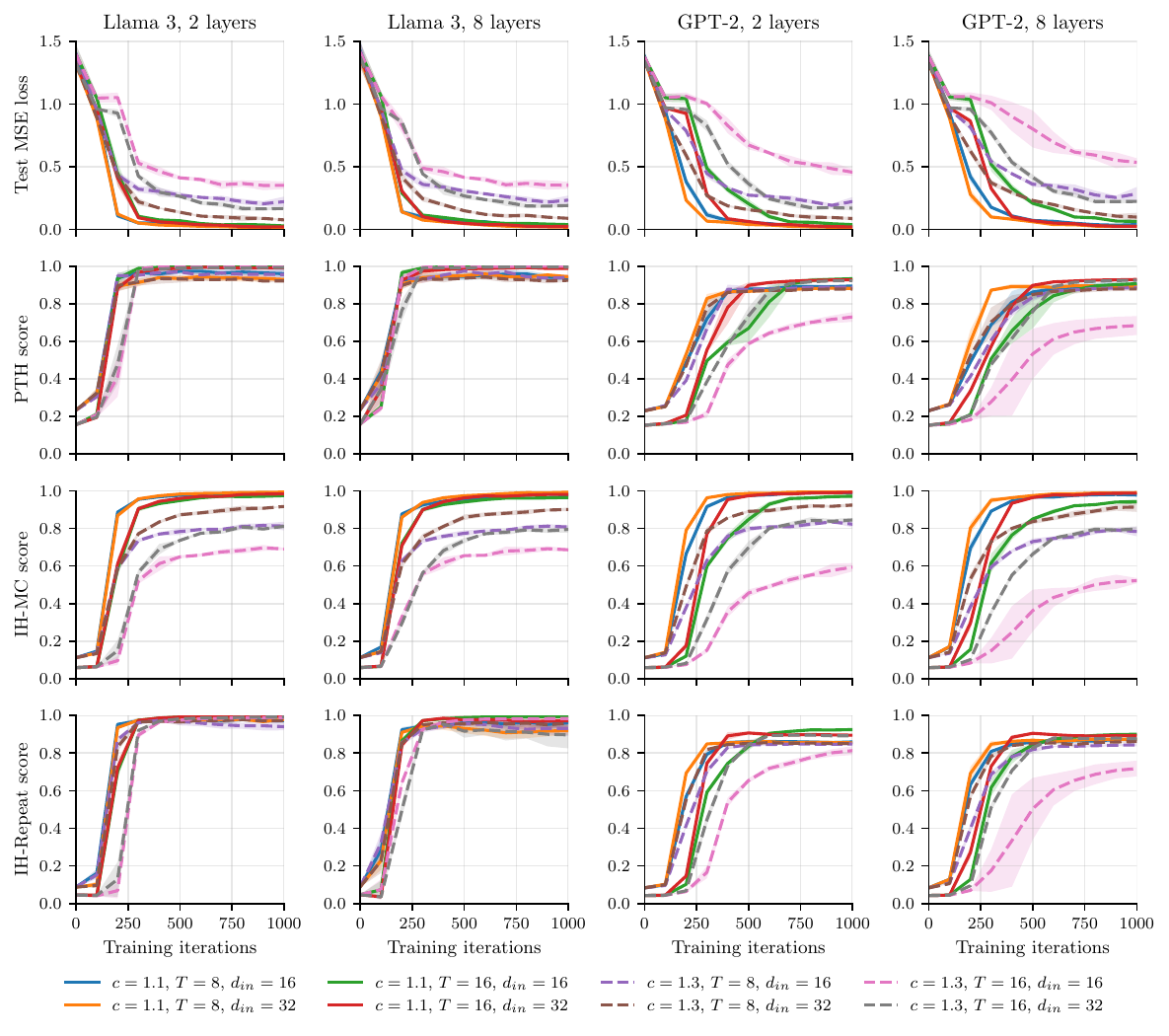} 
\caption{\textbf{Robustness of Emergence.} 
The alignment between the loss drop (top row) 
and the saturation of PTH/IH scores (bottom rows) 
across different architectures (Llama 3, GPT-2) 
and GMC configurations (varying $c$, $\din$, $T$).} 
\label{fig:sweep-llama3-gpt2-metrics}
\end{figure}

\begin{figure}[ht]
    \centering
    \includegraphics[width=\textwidth]{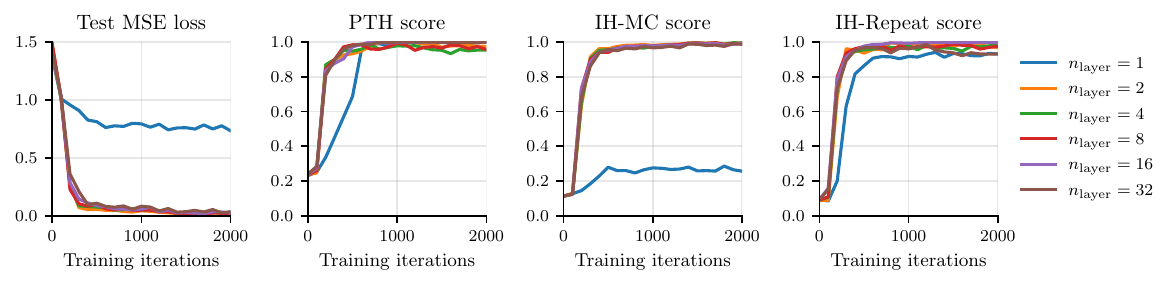}
    \caption{\textbf{Robustness to Depth.} 
    Evolution of loss and detection scores for Llama 3 
    models with varying layer counts ($n_{\text{layer}}$). 
    While all models with $n_{\text{layer}} \ge 2$ exhibit the 
    characteristic loss drop and emergence of PTH/IH circuits, 
    the single-layer model ($n_{\text{layer}}=1$) attempts to 
    increase head scores but fails to solve the task. 
    This confirms that the solution requires the composition of 
    heads across at least two layers.}
    \label{fig:sweep-llama3-layers}
\end{figure}

\subsection{Mechanistic Validation of the PTH$\to$IH Circuit}
\label{app:mechanistic-validation}

We confirm that the induction circuits learned 
by the model strictly adhere to the PTH$\to$IH 
compositional structure 
found in practice in large language models 
\citep{Elhage21TransformerCircuits,Olsson22AnthropicInContextLearning}: 
IHs must come in a layer \emph{after} PTHs, 
since it is the composition of these two head types 
that implements the match-and-copy mechanism. 
We also analyze more closely 
the attention patterns (instead of printing scores only) 
to confirm that the heads implement the functional definitions of PTH and IH 
(attend to previous token, and attend to the match's successor). 
Below, we report these analyses for
the specific 2-layer Llama 3 run ($T=8, c=1.2, \din=16$) 
shown in \Cref{fig:phase-change-llama3}.

\paragraph{Head-wise Scores: Confirms Compositional Structure.} 
\Cref{fig:llama3-metrics-per-head} tracks the evolution 
of attention scores for every head in the model. 
Prior to the loss drop (left, step 120), 
all heads exhibit negligible scores. 
Coinciding with the drop in loss (right, step 200), 
we observe a clean separation of roles: 
heads in Layer 0 specialize as Previous Token Heads 
(high PTH score), 
while heads in Layer 1 specialize as Induction Heads 
(high IH-MC and IH-Repeat scores). 
This simultaneous emergence confirms 
that the model solves the task precisely when it 
unlocks the composition of these two head types. 
See also \Cref{fig:sweep-llama3-layers} where 
single-layer models fail to solve the task, 
lacking the necessary compositional structure 
\citep{Sanford24_1LayerTransformer_arent_Induction_Head}. 

\paragraph{Attention Patterns: Confirms Functional Definition of PTH and IH.} 
We validate the functional definition 
of these heads using 
the IH-Repeat sequence format 
$(\token_1, \dots, \token_K, \token_1, \dots, \token_K)$ 
described in \Cref{app:pth-ih}. 
\Cref{fig:llama3-ih-repeat-attn} visualizes 
the attention weights averaged over all 
9 heads per layer and a batch of 32 sequences of length $2K=16$.
\begin{itemize}
    \item \textbf{Layer 0 (PTHs):} 
    The attention mass concentrates 
    on the first subdiagonal (\Cref{fig:llama3-ih-repeat-attn}a), 
    confirming that these heads consistently attend to 
    the immediate predecessor token.
    \item \textbf{Layer 1 (IHs):} 
    The attention mass concentrates on the $(K-1)$-th subdiagonal 
    during the second half of the sequence (\Cref{fig:llama3-ih-repeat-attn}b). 
    This offset indicates that for a repeated token $\token_{K+t}$, 
    the head attends to the token \emph{following} 
    its first occurrence $\token_{t+1}$. 
\end{itemize}

\begin{figure}[ht]
    \centering
    \begin{subfigure}[b]{0.47\textwidth}
        \centering
        \begin{overpic}[height=5cm,trim=0 0 -140 0,clip]{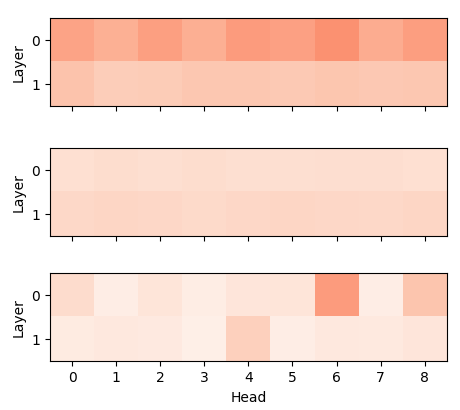}
            \put(84,57){\small \textbf{PTH score}}
            \put(82,35){\small \textbf{IH-MC score}}
            \put(78.5,14){\small \textbf{IH-Repeat score}}
        \end{overpic}
        \caption{Before transition ($i = 120$)\mbox{\hspace{0.9cm}}}
    \end{subfigure}
    \hfill
    \begin{subfigure}[b]{0.42\textwidth}
        \centering
        \includegraphics[height=5cm]{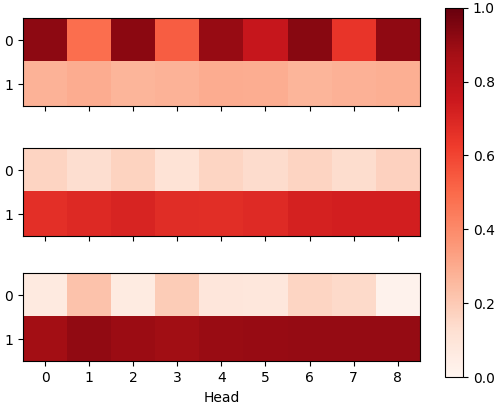}
        \caption{After transition ($i = 200$)\mbox{\hspace{0.7cm}}}
    \end{subfigure}
    \caption{\textbf{Compositional Structure Confirmed by Head-wise Scores.}
    Per-head attention scores for the 2-layer 
    Llama 3 model ($T=8, c=1.2, \din=16$) 
    from \Cref{fig:phase-change-llama3}. 
    (a) Before the loss drop, scores are noise. 
    (b) Immediately after, Layer 0 heads (top markers) 
    maximize PTH scores, and Layer 1 heads (bottom markers) maximize IH scores.}
    \label{fig:llama3-metrics-per-head}
\end{figure}

\begin{figure}[ht]
    \centering
    \begin{subfigure}[b]{0.45\textwidth}
        \centering
        \includegraphics[height=5cm]{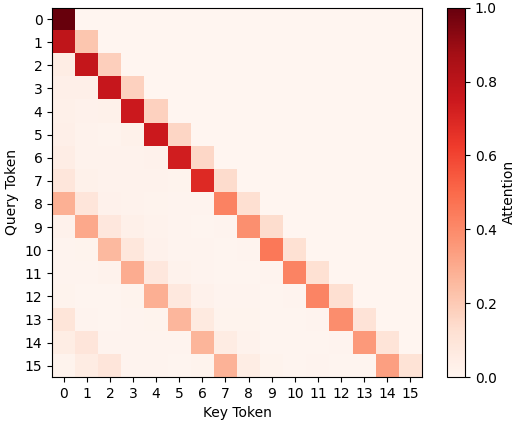}
        \caption{Layer 0 Activations (PTH)\mbox{\hspace{0.5cm}}}
    \end{subfigure}
    \hfill
    \begin{subfigure}[b]{0.45\textwidth}
        \centering
        \includegraphics[height=5cm]{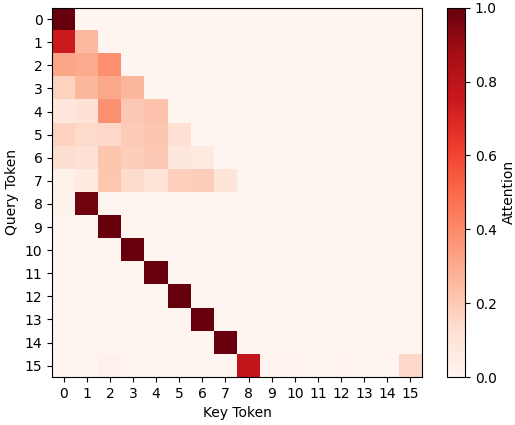}
        \caption{Layer 1 Activations (IH)\mbox{\hspace{0.5cm}}}
    \end{subfigure}
    \caption{\textbf{Functional Definition Confirmed by Attention Maps.} 
    Average attention weights at step $i=1000$ on IH-Repeat samples ($K = 8$). 
    (a) Layer 0 attends to the previous token (subdiagonal). 
    (b) Layer 1 attends to the token following the previous occurrence 
    of the current query (shifted subdiagonal).}
    \label{fig:llama3-ih-repeat-attn}
\end{figure}

\section{Transfer Beyond Gaussian Data: ICL Classification on Omniglot}
\label{app:downstream-omniglot}

To test if the learned induction mechanisms 
are abstract (transferable) or specific to Gaussian data, 
we transfer a GMC-trained model to the Omniglot few-shot 
classification task \citep{OmniglotDataset}.

\paragraph{Dataset and Protocol.}
We use the setup from \citet{Singh24InductionHeadFormation}. 
Sequences consist of interleaved images and labels:
\[
(x_1, y_1, x_2, y_2, x_q) \to y_q
\]
where $x_i$ are embeddings\footnote{
We utilize the $512$-dimensional ResNet-18 embeddings 
of Omniglot characters provided by \citep{Singh24InductionHeadFormation} 
at \url{https://github.com/aadityasingh/icl-dynamics/blob/main/omniglot_resnet18_randomized_order_s0.h5}.
} of handwritten characters and $y_i \in \{1, \dots, 5\}$ are labels. 
In each sequence, 
$x_q$ is a copy of either $x_1$ or $x_2$, 
and the target $y_q$ is the corresponding label. 
This is a match-and-copy task requiring visual feature 
matching rather than Gaussian correlation matching.

In particular, just as in \citet{Singh24InductionHeadFormation}, 
the data is a restriction of the original Omniglot 
dataset \citep{OmniglotDataset} 
($C=1623$ classes of characters with $N=20$ examples each) 
to single exemplars ($E=1$) drawn from 
disjoint $C_{\text{train}}=50$ classes for training 
and $C_{\text{test}}=100$ classes for testing. 
Sequences are constructed by sampling uniformly two 
distinct characters $(x_1, x_2)$ and assigning them 
uniformly random distinct labels $y \in \{1, \dots, L\}$ with $L=5$ 
(hence the same image $x$ can have different associated labels across sequences).  
The query $x_q$ is set to either $x_1$ or $x_2$ with equal probability, 
and the target $y_q$ is the corresponding label. 
The label randomization across sequences 
(same $x$ with different $y$)
forces the model to perform 
in-context learning rather than memorizing fixed image-label associations. 
The total number of unique training sequences is:
\[
\underbrace{C_{\text{train}}(C_{\text{train}}-1)}_{\text{pairs } (x_1,x_2)} 
\cdot \underbrace{L(L-1)}_{\text{pairs } (y_1,y_2)} 
\cdot \underbrace{2}_{\text{query } x_q} = 50 \cdot 49 \cdot 5 \cdot 4 \cdot 2 = 98,000
\]
Unlike \citet{Singh24InductionHeadFormation}, 
which reserves 20\% of training data 
to test "relabeling" performance 
for a separate purpose, 
we train on the full set of $98,000$ sequences and 
we evaluate exclusively on the 
disjoint $C_{\text{test}}$ classes. 

\begin{table}[h!]
\centering
\caption{Specifications of the experiments on Omniglot.}
\label{tab:omniglot-experiment-specs}
\begin{tabular}{lll}
\toprule
\textbf{Architecture} & \textbf{Omniglot Training Parameters (fine-tuning and from scratch)} \\
\midrule
Model: Llama 3 & Training steps: 4000 \\
Layers ($n_\mathrm{layer}$): 2 & Batch size: 32 \\
Hidden dimension ($d_{\mathrm{model}}$): 128 & Optimizer: AdamW \\
Head dimension ($d_{\mathrm{head}}$): 32 & Learning rate: $10^{-3}$ \\
Heads per layer: $d_{\mathrm{model}}/d_{\mathrm{head}}: 4$ & Linear warmup: 100 steps \\
FFN dimension: $4\times d_{\mathrm{model}}= 512$ & Scheduler: cosine decay \\
\bottomrule
\end{tabular} 
\end{table}

\paragraph{Transfer Method.}
We pretrain a Llama 3 model on GMC and transfer it to Omniglot by 
freezing the backbone and fine-tuning only the input and output projections. 
The Llama 3 model is chosen to ensure sufficient capacity while remaining 
computationally efficient (see \Cref{tab:omniglot-experiment-specs}). 
The transfer learning procedure consists of three steps:
\begin{enumerate}
\item \textbf{GMC Pretraining:} 
The model is first trained to convergence on 
an ``easy'' GMC configuration 
($T=8, c=1.1, d_{\text{in}}=16, \sigma=0.1$, 
``easy'' in the sense that the model quickly reaches the noise floor 
for these parameters as shown in \Cref{fig:sweep-llama3-gpt2-metrics}). 
We use the procedure from \Cref{app:mc-experimental-setup}, 
but with 2000 training steps and 400 warmup steps 
(instead of 1000 and 200). 
\item \textbf{Embeddings Adaptation:} 
We freeze the entire backbone (self-attention and FFNs). 
We replace the GMC input/output linear projections 
with new layers suited for Omniglot: 
the input embedding maps the concatenated image vectors 
and one-hot labels (dimension $512+L$) to $d_{\text{model}}$, 
and the unembedding maps $d_{\text{model}}$ to $L=5$ class logits.
\item \textbf{Fine-tuning:} 
We train \emph{only} the new embedding and 
unembedding layers to minimize cross-entropy loss, 
using the hyperparameters in \Cref{tab:omniglot-experiment-specs}.
\end{enumerate}

\paragraph{Main Result: Successful Transfer.}
As shown in \Cref{fig:downstream-omniglot} (main text), 
the model successfully learns Omniglot despite having 
its entire inference mechanism (attention + FFN) frozen 
from the Gaussian task. 
This indicates that the GMC-trained heads 
implement a distribution-agnostic ``match-and-copy'' operation, 
regardless of whether the input data is Gaussian or character embeddings.

\paragraph{Additional Result: Efficiency of Transfer.} 
For comparison, we also train the same Llama 3 model from scratch 
using the same optimization hyperparameters listed in \Cref{tab:omniglot-experiment-specs}.\footnote{
For which we obtain a loss curve 
similar to that reported in \citet{Singh24InductionHeadFormation},
where they used an attention-only architecture trained from scratch on the same Omniglot setup.
} 
Note that these optimization hyperparameters 
were originally tuned for this from-scratch training, 
not for fine-tuning, so 
the comparison is biased in favor of training from scratch. 
The GMC-pretrained model achieves $90\%$ test accuracy 
using $3\times$ fewer training FLOPs than the model trained from scratch. 
Furthermore, the model trained from scratch does not 
reach the $99\%$ test accuracy achieved by the pretrained 
model within the allotted budget (\Cref{fig:downstream-omniglot}). 
This shows that the 
match-and-copy mechanism learned during GMC pretraining 
not only transfers to Omniglot, 
but also confers a significant optimization advantage. 

\section{Comparison with Non-Attention Architectures}
\label{app:comparison-other-archs}

This section details the comparison between 
Transformers and non-attention sequence models presented 
in \Cref{subsec:gmc-comparison-archis}.

\paragraph{Model Configurations.}
To ensure a fair comparison, 
we match all models on inference FLOPs. 
FLOPs are directly measured with 
PyTorch (\texttt{torch.profiler})  
on a training batch (same batch size for all models: 512, 
and task parameters as detailed below). 
We test Gated Recurrent Units (GRU) \citep{Cho14GRU}, 
two State-Space Model (SSM) variants--S4 \citep{Gu22S4StructuredStateSpaceSequenceModel} 
and H3 \citep{Fu23H3}--and the Hyena architecture \citep{Poli23HyenaHierarchy}, 
a hierarchical long-convolution sequence model.

\Cref{tab:inference_flops} reports the 
configurations used for each architecture. 
A few details on the choices made:
\begin{itemize}[leftmargin=1.6em]
\item \textbf{Transformers (Llama 3, GPT-2).}  
  We reuse the settings of \Cref{tab:mc-experiment-specs} with two
  changes:\vspace{-0.3em}
  \begin{enumerate}[label=(\alph*), nosep]
    \item depth fixed to $\,n_{\mathrm{layer}}=2$;
    \item GPT-2 heads increased to $H=10$ so that
      $d_{\mathrm{model}}=H\times d_{\mathrm{head}}=320$, matching the
      FLOPs of the Llama 3 run.
  \end{enumerate}
  As usual the feed-forward width is $4d_{\mathrm{model}}$.
\item \textbf{SSMs (S4, H3), Hyena, and GRU.}
For these non-attention baselines, we explored deeper (up to 15 layers) and wider
(up to $d_{\mathrm{model}}=512$) configurations; none closed the
gap to Transformers.
We therefore retain the mid-range settings listed in
\Cref{tab:inference_flops}.
\end{itemize}

\noindent
Parameter counts in the table \emph{exclude} the input and output
projection layers ($d_{\mathrm{model}}\times d_{\mathrm{in}}$ and
$d_{\mathrm{out}}\times d_{\mathrm{model}}$, respectively), 
whereas FLOPs \emph{include} them.

\paragraph{GMC Configuration.}
All models are trained on the same ``easy'' GMC configuration
(it is among the easiest instances in our 
sweep according to \Cref{fig:sweep-llama3-gpt2-metrics}): 
\[
T=8,\qquad \din=\dout=16,\qquad \Cqe=\frac{1}{(1.1)^2}\I.
\]
All other parameters follow \Cref{app:mc-experimental-setup} 
($\stdnoise=0.1$, $\stdtoken=1$, i.i.d.\ $\WV\sim\mathcal{N}(0,1)$, AdamW, etc.).
Short sequences (small $T$) 
and a strong query-match correlation (small $c$)
maximize the 
chance that non-attention architectures succeed, while
Transformers remain robust on much harder settings
(\Cref{fig:sweep-llama3-gpt2-metrics}). 

\paragraph{Optimization.}
Training follows the protocol of
\Cref{app:mc-experimental-setup} with one modification:
we run $2000$ gradient steps (five times the iterations needed for
Transformers to reach near noise floor) to give slower-converging 
non-attention architectures 
a fair shot.
All other hyper-parameters---AdamW with learning rate
$10^{-3}$, batch size $512$, 
cosine decay with 400-step linear warm-up---are
kept identical across architectures.

\begin{table}[t]
    \centering
    \caption{Model specifications for architecture comparison. 
    All models operate with $d_{\mathrm{model}} \approx 192\text{--}320$ 
    and maintain comparable compute budgets.}
    \label{tab:inference_flops}
    \begin{tabular}{lrrr}
        \toprule
        \textbf{Model} & \textbf{Architecture} & \textbf{Parameter count} & \textbf{Inference GFLOPs} \\
        \midrule
        GPT-2 \citep{GPT2} & 2 layers, 10 heads, $d_{\mathrm{head}}=32$  & 2.47M & 22.8 \\
        Llama 3 \citep{Llama3} & 2 layers, 9 heads, $d_{\mathrm{head}}=32$   & 2.66M & 24.6 \\
        GRU \citep{Cho14GRU} & 10 layers, $d_{\mathrm{model}}=208$ & 2.60M & 24.0 \\
        S4  \citep{Gu22S4StructuredStateSpaceSequenceModel} & 8 layers, $d_{\mathrm{model}}=192$, $d_{\mathrm{state}}=64$ & 2.77M & 24.5 \\
        H3  \citep{Fu23H3} & 6 layers, $d_{\mathrm{model}}=192$, $d_{\mathrm{state}}=64$ & 3.04M & 24.6 \\
        Hyena \citep{Poli23HyenaHierarchy} & 5 layers, $d_{\mathrm{model}}=192$, order 5 & 3.10M & 25.6 \\
        \bottomrule
    \end{tabular}
\end{table}

\paragraph{Loss and Noise Floor.}
The irreducible contribution of the noise to 
the population MSE is 
$\frac{1}{\dout}\E\|\eps\|_2^2=\stdnoise^2=0.01$.  
Indeed, the population version of the MSE loss 
\eqref{eq:mse-loss-expes} we optimize is:
\[
\frac{1}{\dout} \E\|\model - \target\|_2^2
= \frac{1}{\dout} \E\|\model - \WV \token_{t_0+1}\|_2^2 
+ \frac{1}{\dout} \E\|\eps\|_2^2
\]
because $\target = \WV \token_{t_0+1} + \eps$ with independent 
$\eps\sim\mathcal N(0,\stdnoise^2\I_{\dout})$. 
We therefore say that a model reaches 
the \emph{noise floor} when its test MSE 
is on the order of $\stdnoise^2=0.01$.

\paragraph{Extended Training.}
Transformers solve the task 
(reaching the noise floor) 
within 2,000 steps (\Cref{fig:comparison-other-archs}). 
To rule out slow convergence as a factor for the baselines, 
we extended their training to 20,000 steps\footnote{
Same $400$-step warm-up, 
but with cosine decay stretched to 20,000 steps. 
This alters the learning rate profile (and thus the loss curve) during the first 2,000 steps, 
so the the first 2,000 steps of this extended run are not 
directly comparable to the original 2,000-step
curves in \Cref{fig:comparison-other-archs}. 
What \emph{is} comparable, however, are the \emph{final losses} at 
the end of each run, 
corresponding to their performance for a fixed training budget.  
} 
($10\times$ the Transformer budget).

\begin{figure}[ht]
\centering
\includegraphics[width=0.6\textwidth]{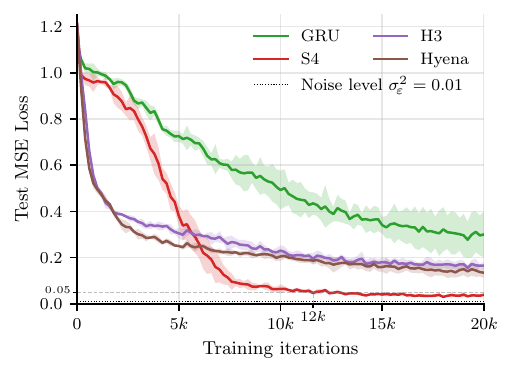}
\caption{\textbf{Extended Training for Non-Attention Models.} 
Even with 20{,}000 training steps, SSMs (S4, H3), the Hyena architecture, and GRU fail to reach the noise floor 
($\stdnoise^2=0.01$) that Transformers achieve in under 1{,}000 steps.}
\label{fig:ssm-20k}
\end{figure}

As shown in \Cref{fig:ssm-20k}, 
even with extended training, 
SSMs fail to solve the task. 
For example, S4 plateaus near MSE $0.05$ 
($5\times$ the noise floor) 
and it takes over 12,000 steps to reach this level, 
which is $30\times$ slower than Llama 3 
(\Cref{fig:comparison-other-archs} shows that 
Llama 3 reaches MSE $0.05$ in under $400$ steps, 
GPT-2 in under 600 steps, 
and both eventually hit the noise floor 
under 2,000 steps).
The other baselines perform worse:
Hyena and H3 plateau above $0.1$, 
and GRU settles in $0.2-0.4$ depending on the run. 
These findings reinforce a broader pattern 
in the literature suggesting that copying 
is fundamentally difficult for state-space models 
\citep{Arora24ZoologyRecall,Jelassi24TransformersVsSSMAtCopying,Patro24Mamba360}. 
Even architectures like H3 \citep{Fu23H3} 
or Hyena \citep{Poli23HyenaHierarchy} 
struggle to perform the precise, 
content-based retrieval required for GMC. 
This can be surprising since H3 explicitly shifts hidden states 
to mimic 
previous-token access, and both H3 and Hyena 
have been designed to 
perform well on tasks that are similar in spirit to GMC 
(associative retrieval, see \Cref{sec:related-work}). 
This shows that although GMC is similar to these tasks, 
it remains a distinct challenge that SSMs struggle to solve. 

\section{Formalization of \Cref{thm:mse-bias-main}}
\label{app:mse-bias}

In this appendix, we provide the formal setup underlying \Cref{thm:mse-bias-main}. 
Its proof is given in \Cref{app:mse-bias-proof}. 
The theorem is stated under explicit assumptions on the data geometry
and on the optimization trajectory. 
These assumptions are not claimed to hold universally; 
rather, they delineate a regime in which 
a max-margin implicit bias for MSE can be rigorously established.
A representative empirical sanity check illustrating this regime is given in
\Cref{subsec:assumptions-hold}. 

This section is organized as follows:
\begin{itemize}
    \item \Cref{subsec:mse-bias-setup} defines the data and model setup.
    \item \Cref{subsec:mse-bias-vectorized} re-expresses the model in vectorized notation to ease comparison with prior work on implicit bias.
    \item \Cref{subsec:mse-bias-ce} makes a small detour to discuss a case already covered in prior work: the Cross-Entropy (CE) loss.
    \item \Cref{subsec:mse-bias-theorem} states the main theorem.
\end{itemize}

\subsection{Notation and Setup}
\label{subsec:mse-bias-setup}
 
Consider $n$ training samples with the GMC format (\Cref{def:mc})
\[
    \underbrace{\token_1, \dots, \token_{t_0}, \token_{t_0+1}, \dots, \token_{T}}_{\text{context}}, 
    \underbrace{\query}_{\text{query}} \to \underbrace{\target}_{\text{target}}
\]
with exact match $\query=\token_{t_0}$ 
and exact copy $\target = \token_{t_0+1}$. 
For each sample $i \in \{1, \dots, n\}$, 
the notation is follows:
\begin{itemize}
    \item \textbf{Data:} 
    A context matrix 
    $\Ectx^{(i)} = [\token^{(i)}_1, \dots, \token^{(i)}_{T}] \in \R^{\din \times T}$ 
    and a query $\query^{(i)} \in \R^{\din}$.
    \item \textbf{Task:} 
    The correct match position is 
    $y_i = t_0^{(i)} \in \{1, \dots, T-1\}$, 
    such that the target is 
    $\target^{(i)} = \token^{(i)}_{y_i+1}$ (exact copy) 
    and the query satisfies $\query^{(i)} = \token^{(i)}_{y_i}$ 
    (exact match).
\end{itemize}

\paragraph{Simplified model: shifted-key single query attention head.}
The model follows the standard single-query attention head with parameters $(\WQ,\WK,\WV)$ 
\citep[Algorithm 3]{Phuong22FormalTransformers},
but with three modifications:
\begin{enumerate}
    \item \textbf{Shifted-keys (Frozen PTH).} The token at position $t$ is 
    copied according to attention weights computed from keys at position $t-1$ (i.e., shifted by one position, highlighted in \magenta{magenta} below):
    \[
        \sum_{t=1}^{T-1} \softmax(\Ectx_{\magenta{1:T-1}}^\top \WK^\top \WQ \query)_{\magenta{t}}\,
        \WV \token_{t\magenta{+1}}.
    \]
    This is exactly the computation needed for match-and-copy: 
    match the query against a previous occurrence, 
    then \emph{copy the successor} (thanks to the shifted key positions). 
    Shifting the keys can be understood 
    as adding a previous layer that would implement an exact PTH \citep[Sec 4.1]{Singh24InductionHeadFormation}\citep[Eq 4]{Wang25LazyRichLearningIH}. 
    \item \textbf{Merged key-query.} 
    We set $\Wkq=\WK^\top \WQ \in \R^{\din \times \din}$ and optimize $\Wkq$ directly. 
    \item \textbf{Fixed value.} 
    We \emph{freeze} the value matrix $\WV$ 
    to the ground truth one used in GMC data generation (namely, $\WV=\I$)
    and optimize only $\Wkq$. 
    This has two motivations.
    (i) Freezing $\WV$ allows us to isolate the non-convex dynamics induced by the softmax. 
    In contrast,
    the (empirical and population) MSE loss is \emph{strictly convex} in $\WV$
    since it appears linearly in the model output. 
    Thus, $\WV$ is not the main source of the optimization difficulty and, 
    in principle, can be solved optimally at each step. 
    In fact, already at initialization, 
    starting from $\Wkq = 0$ and solving 
    first for $\WV$ yields an unbiased estimator 
    of the true $\WV$, which converges to the exact 
    solution in the infinite-data limit. 
    This observation echoes multi-stage training procedures 
    proposed in prior work 
    \cite{Nichani24OneGramTaskLearnsIH,Chen24nGramTaskLearnsIH,Wang25LazyRichLearningIH}; we do not pursue this here. 
    (ii) $\WV$ does not affect \emph{which position is attended to},
and hence does not directly influence the directional implicit bias studied here. 
\end{enumerate}

For a context $\Ectx=(\token_1,\dots,\token_T)$ and query $\query$, the model output from \eqref{eq:minimal-model-output} is:
\begin{equation}
\label{eq:shifted-key-attention-model}
\text{model}(\Wkq;\Ectx,\query)
=
\sum_{t=2}^{T}
\underbrace{\softmax\bigl(\big(\token_s^\top \Wkq \query\big)_{s=1}^{T-1}\bigr)_{t-1}}_{
    =:a_t(\Wkq;\Ectx,\query)
}\; \WV\token_{t}.
\end{equation}

The implicit bias we prove is 
not specific to this key-shift: 
if instead of a target $\target=\token_{t_0+1}$ 
we had $\target=\token_{\sigma(t_0)}$ 
for a given injective map $\sigma$ 
fixed 
across all 
samples (i.e., the target
was copied from a fixed pairing $t_0\mapsto \sigma(t_0)$), 
the results would hold verbatim with notational changes, 
e.g., replacing $t_0+1$ by $\sigma(t_0)$, 
and the shifted-keys by a model 
with keys shifted according to $\sigma$: 
$\sum_{t\in\sigma^{-1}(\{1,\dots,T\})}
\softmax(\Ectx_{\sigma^{-1}(\{1,\dots,T\})}^\top \Wkq \query)_{\sigma^{-1}(t)}\,\token_{t}$.

\subsection{Vectorized Notation to Ease Comparison with \cite{Soudry18ImplicitBiasCE}}
\label{subsec:mse-bias-vectorized}

For readers who are familiar with max-margin bias results, 
we first reformulate the model using vectorized notation. 
This facilitates comparison with the seminal work of \citet{Soudry18ImplicitBiasCE} 
which characterizes the implicit bias 
of logistic regression (linear model with CE loss). 
This highlights how the MSE gradients diverge from the standard 
CE gradients and motivates the specific assumptions 
used to bridge the gap.

Denote the vectorized form of 
$\Wkq \in \R^{\din \times \din}$ by 
\[
w = \flatten(\Wkq) \in \R^{\din^2}.
\]
We introduce the feature vector for sample $i$ and position $t$:
\begin{equation}
    x_{i,t} := \flatten(\token^{(i)}_t \query^{(i)\top}) \in \R^{\din^2}.
\end{equation}
Using the identity $\token_t^{\top} \Wkq \query = \langle \Wkq, 
\token_t \query^\top \rangle_F = w^\top \flatten(\token_t \query^\top)$, 
the model's attention weights in \eqref{eq:shifted-key-attention-model}
can be expressed as:
\[
    a^{(i)}_1(w) = 0, \quad
    a^{(i)}_{t+1}(w) = \softmax\bigl(\{w^\top x_{i,s}\}_{s=1}^{T-1}\bigr)_{t}, 
    \quad \text{for } t \in \{1, \dots, T-1\}.
\]
Introducing the \emph{difference vectors} relative to the correct 
match position $y_i$:
\begin{equation}
    \label{eq:diff-features}
    \tilde{x}_{i,t} := x_{i,y_i} - x_{i,t}, \quad \text{for } t \neq y_i,
\end{equation}
we can rewrite the attention coefficients 
in terms of these difference vectors, 
as in \cite{Soudry18ImplicitBiasCE}:
\begin{equation}
  S_i(w) := \sum_{\substack{s=2,\dots,T \\ s \neq y_i}} \exp(-w^\top \tilde{x}_{i,s}),
  \quad \text{and} \quad
    a^{(i)}_{t+1}(w) 
    = \frac{\exp(-w^\top \tilde{x}_{i,t})}{
        1 + S_i(w)
    },
    \quad \text{for } t \neq y_i,
    \qquad
    a^{(i)}_{y_i+1}(w)
    = \frac{1}{
        1 + S_i(w)}.
\end{equation}
The margins in \eqref{eq:margin} can also be expressed in terms of these difference vectors:
\begin{equation}
    \label{eq:margin-vectorized}
    m_{i,t}(w) 
    := w^\top \tilde{x}_{i,t}, \quad \text{for } t \neq y_i.
\end{equation} 
As before, diverging margins ($m_{i,t}(w) \to \infty$) imply that 
the attention mass concentrates on the correct index $y_i$, i.e., $a^{(i)}_{y_i+1} \to 1$.

\subsection{The CE (Convex) Case: Covered by Existing Literature}
\label{subsec:mse-bias-ce}

This section connects \Cref{thm:mse-bias-main} 
to prior work on implicit bias,
by discussing the case of Cross-Entropy (CE) loss, 
which is already covered in \citet{Soudry18ImplicitBiasCE}. 
This serves to highlight the differences
between the CE and MSE cases,
and motivates the additional assumptions
needed to handle MSE.

The implicit bias of gradient descent is well understood 
for binary classification tasks
$(x_i, y_i)_{i \le n}$ with $y_i \in \{\pm 1\}$, 
provided the loss is strictly decreasing in 
a ``score/margin'' $y_i f(x_i; w)$. 
A canonical example is the linear model $f(x; w) = w^\top x$.
In this regime, 
gradient descent converges toward the max-margin solution. 
This includes linear models trained with cross-entropy,
\[
L(w) = \log\!\bigl(1 + \exp(-y w^\top x)\bigr) = \ell(y w^\top x),
\]
where $\ell$ is the logistic monotone decreasing loss
\cite{Soudry18ImplicitBiasCE,JiTelgarsky19Nonseparable}.

More recently, these results have been extended to a case of particular interest here:
one-layer attention models trained with losses that are monotone in the margin
(for binary classification, $y \in \{\pm 1\}$).
For instance, consider
\[
L(w) = \ell\!\left(y\, v^\top X \softmax\!\left(X^\top \WK^\top \WQ X\right)\right),
\]
where $v$ is fixed (a fixed linear decoder) and $\ell$ is monotone decreasing
(e.g., logistic $\ell(m)=\log(1+e^{-m})$, exponential $\ell(m)=e^{-m}$, or correlation $\ell(m)=-m$ losses) 
\citep{TarzanaghLZO23MaxMarginTokenSelection,Tarzanagh23TransformersAsSVM,Vasudeva25ImplicitBiasRatesSelfAttention}.

In particular, the existing literature already covers the 
multi-class CE loss defined in our setting by:
\begin{equation}
\label{eq:ce-loss}
\CE(w)
:= -\frac{1}{n} \sum_{i=1}^n \log a^{(i)}_{y_i+1}(w)
= \frac{1}{n} \sum_{i=1}^n \log\Big(1 + \sum_{\substack{t=1 \\ t \neq y_i}}^{T-1} \exp(-m_{i,t}(w))\Big).
\end{equation}
Indeed, 
the proof of Theorem~7 in \citet{Soudry18ImplicitBiasCE} applies 
with only a minor adaptation: 
unlike multi-class linear classification where one has a 
separate vector parameter $w_t$ per class $t$, 
here \emph{all positions share the same parameter} $w$. 
This shared-parameter constraint does not fundamentally alter the core argument,
and the standard proof can be adapted with minor modifications.
We obtain the next implicit bias: 
$w(\tau)/\|w(\tau)\|_F \to \hat w/\|\hat w\|_F$ 
and 
$\|w(\tau)\|_F \sim \|\hat w\|_F \log (\tau)$ 
where $\hat w$ is the max-margin solution:
\begin{equation}
    \label{eq:svm}
    \hat w 
    := \argmax_{w} \frac12 \|w\|_2^2
    \quad \text{s.t.} \quad
    \min_{i\le n} \min_{t \neq y_i} m_{i,t}(w) \ge 1.
\end{equation}

In terms of the original matrix $\Wkq$ (recall $w=\flatten(\Wkq)$), we get the following result, to compare with \Cref{thm:mse-bias-main}. 
\begin{theorem}[Implicit bias of $\CE$, adapted from \citet{Soudry18ImplicitBiasCE}]
\label{thm:ce-bias-mainpaper}
Consider $\Wkq(\tau+1) = \Wkq(\tau) - \eta \nabla \CE(\Wkq(\tau))$ 
on the \emph{CE} loss \eqref{eq:ce-loss}. 
Assume: 
(i) the max-margin problem in \eqref{eq:svm} is feasible,
(ii) a standard ``correction vector'' existence condition as in \citet{Soudry18ImplicitBiasCE},
and (iii) $\eta$ is sufficiently small (below a smoothness-based bound). 
Then $\Wkq(\tau)/\|\Wkq(\tau)\|_F \to \maxmargin/\|\maxmargin\|_F$ 
and 
$\|\Wkq(\tau)\|_F \sim \|\maxmargin\|_F \log \tau$. 
\end{theorem}

In contrast, our GMC objective is MSE regression, 
which differs\footnote{
MSE is also nonconvex, which 
creates additional challenges compared to the logistic regression 
case \cite{Soudry18ImplicitBiasCE}, 
but nonconvexity is a
feature shared with 
prior attention implicit-bias results 
\citep{TarzanaghLZO23MaxMarginTokenSelection,Tarzanagh23TransformersAsSVM,Vasudeva25ImplicitBiasRatesSelfAttention} 
so this is not a new difficulty specific to MSE.
} in important ways from the classification losses above:
it is not monotone in the margins (it is not of the form $\sum_{i,t} \ell(m_{i,t})$ with $\ell$ decreasing 
and $m_{i,t}$ the margins defined in \eqref{eq:margin-vectorized}), 
it admits finite-norm interpolants under non-identifiability, 
and its gradient contains quadratic interaction terms that 
do not need to keep increasing the margins (see \Cref{rem:ce-comparison} for more details).  
These differences will motivate additional geometric and 
trajectory-stability structural assumptions to recover a max-margin bias. 
Besides these different assumptions, 
\Cref{thm:mse-bias-main} provides 
a different growth rate for the weight norm:
for MSE, we have $\|\Wkq(\tau)\|_F \sim \|\maxmargin\|_F \,\tfrac12 \log \tau$,
whereas for CE (\Cref{thm:ce-bias-mainpaper})
it is $\|\Wkq(\tau)\|_F \sim \|\maxmargin\|_F \, \log \tau$. 
Here,
the extra factor $\tfrac12$ (relative to the $\CE$ scaling of $1$) 
is due to the different structure of the MSE gradients. 
Indeed, we will see in the proof that the function $f(\tau)$
governing the norm growth by $\|w(\tau)\| \approx f(\tau)$
solves $f'(\tau) \approx \|\nabla \MSE(w(\tau))\|$,
and the MSE gradients evolve differently from CE gradients
due to the quadratic terms in the MSE gradients
(discussed in more detail above \eqref{eq:def-tilde-w} in \Cref{app:mse-bias-proof}).

\subsection{Main Theorem}
\label{subsec:mse-bias-theorem}

We now formally state the implicit bias result. 
The assumptions are grouped into three categories: 
\emph{Data Geometry} (properties of the tokens and query), 
\emph{Successful Optimization} (convergence of the loss), and 
\emph{Admissible Trajectory} 
(technical controls on the softmax dynamics 
that ensure asymptotic stability). 

\begin{theorem}[Implicit Bias of MSE]
\label{thm:mse-bias-app}
Consider the dataset and model defined above. 
We train the model's weight vector $w \in \R^{\din^2}$ 
via GD $w(\tau+1) = w(\tau) - \eta \nabla\MSE(w(\tau))$ 
with a fixed step size $\eta>0$ on the MSE loss:
\begin{equation}
    \label{eq:mse-loss-vec}
    \MSE(w) = \sum_{i=1}^n \frac{1}{2} \Bigl\| \sum_{t=1}^{T} a^{(i)}_t(w) 
    \token^{(i)}_t - \token^{(i)}_{y_i+1} \Bigr\|^2.
\end{equation}
For any sample $i$ and token position
$t\neq y_i$, 
let 
$d_{i,t} := \token^{(i)}_{t+1} - \token^{(i)}_{y_i+1}$ 
denote the error directions in token space.
Assume the following conditions hold:

\paragraph{I. \DataGeom{Formalization of the Event in \Cref{subsec:gmc-geometric-event}}.}
\begin{enumerate}
    \item[\textbf{(A1)}] \textbf{Separability.} 
    The hard-margin SVM problem is feasible, i.e., the constraints in the following optimization problem are satisfiable.  
    As a consequence, there is a unique max-margin separator $\hat{w}$ defined by:
    \begin{equation}
        \label{eq:svm-app}
        \hat{w} = \argmin_{w \in \R^{\din^2}} \frac{1}{2}\|w\|^2 
        \quad \text{s.t.} \quad w^\top \tilde{x}_{i,t} \ge 1, 
        \quad \forall i, \forall t \neq y_i.
    \end{equation}
    Uniqueness of $\hat{w}$ follows from the strict convexity of the problem. 
    We denote by $\ISV := \{(i,t) : \hat{w}^\top \tilde{x}_{i,t} = 1\}$ 
    the set of support vectors (active constraints) for the max-margin solution $\hat{w}$. 
    
    \item[\textbf{(A2)}] \textbf{Linear Independence.} 
    The support vectors $\{\tilde{x}_{i,t} : (i,t) \in \ISV\}$ are 
    linearly independent.
    
    \item[\textbf{(A3)}] \textbf{Identifiability.} 
    For all $i$, the target $\token^{(i)}_{y_i+1}$ does not lie in 
    the convex hull of the distractor tokens 
    $\{\token^{(i)}_{t+1} : t \neq y_i\}$.
    
    \item[\textbf{(A4)}] 
    \textbf{Half-space.} 
    For all $i$, the error directions lie in a common half-space: 
    $\langle d_{i,t}, d_{i,s} \rangle > 0$ for all $t, s \neq y_i$.
\end{enumerate}

\paragraph{II. \Success{Assuming Successful Optimization}.}
\begin{enumerate}
    \item[\textbf{(A5)}] \textbf{Loss Convergence.} 
    The loss converges to zero: $\MSE(w(\tau)) \to 0$ as $\tau \to \infty$.
    \item[\textbf{(A6)}] \textbf{Summable Gradients.} 
    The gradients are square-summable: 
    $\sum_{\tau=0}^{\infty} \|\nabla \MSE(w(\tau))\|^2 < \infty$.
\end{enumerate}

\paragraph{III. \Converg{Formalization of the Assumptions in \Cref{subsec:mse-technical-assumptions} (\Cref{ass:mse-stabilization})}.}
Let 
$h_{i,t}(w) := \sum_{s \neq y_i} \langle d_{i,t}, d_{i,s} \rangle \, 
\frac{a^{(i)}_{s+1}(w)}{1 - a_{i,y_i+1}(w)}$ 
be a renormalization of the projection of the residual 
onto the error direction $d_{i,t}$. 
\begin{enumerate}
    \item[\textbf{(A7)}] \textbf{Convergence of Pre-factors.} 
    For all $(i,t)$, the limit 
    $h_{i,t}^\infty := \lim_{\tau \to \infty} h_{i,t}(w(\tau))$ exists, 
    and the convergence is sufficiently fast such that 
    $\sum_{\tau=1}^\infty \tau^{-1} |h_{i,t}(w(\tau)) - h_{i,t}^\infty| < \infty$.
    
    \item[\textbf{(A8)}] \textbf{Alignment of Off-Diagonal Drift.} 
    The interactions between support vectors via the softmax 
    Jacobian satisfy the alignment condition detailed in 
    \Cref{ass:alignment} 
    (ensuring the drift does not destabilize the max-margin direction). 
    
    \item[\textbf{(A9)}] \textbf{Support Domination.} 
    The softmax tail is dominated by the support vectors.  
    Specifically, there exists a function 
    $g: \mathbb{N} \to \R$ with 
    $\sum_{\tau} g(\tau)\tau^{-1} < \infty$ 
    such that for all $i$, letting $\ISV(i) := \{s : (i,s) \in \ISV\}$,
    \[
        \sum_{s \notin \ISV(i)} \exp(-w(\tau)^\top \tilde{x}_{i,s}) 
        \le g(\tau) 
        \sum_{s \in \ISV(i)} \exp(-w(\tau)^\top \tilde{x}_{i,s}).
    \]
\end{enumerate}

\noindent \textbf{Conclusion.} Under these assumptions, 
the parameter vector $w(\tau)$ diverges in norm and aligns with the max-margin separator $\hat{w}$:
\[
    \|w(\tau)\| \sim \frac{\|\hat w\|}{2} \log(\tau)
    \qquad\text{and}\qquad
    \frac{w(\tau)}{\|w(\tau)\|} \longrightarrow \frac{\hat{w}}{\|\hat{w}\|}.
\]
\end{theorem}

\paragraph{Interpretation.}
\Cref{thm:mse-bias-app} is a \emph{conditional} asymptotic result.
It does not assert that gradient descent on MSE always converges 
to a max-margin direction,
nor that the listed assumptions are automatically satisfied.
Instead, it states that, conditioned on the right data geometry, 
if a trajectory achieves vanishing loss
and enters a stabilized asymptotic regime described by (A6)--(A9),
then its direction converges to the max-margin separator at a logarithmic rate.
The size and prevalence of this regime are left open;
nevertheless, our experiments indicate that such stabilized trajectories
are observed in practice, 
and the next remark clarifies which assumptions are structural
and which are primarily optimization-related.

\begin{remark}[On the scope of the assumptions]
The assumptions of \Cref{thm:mse-bias-app} serve different roles and should not be interpreted uniformly.
\begin{enumerate}
    \item \textbf{Geometric assumptions (A1--A4).}
    These concern separability, identifiability, and correlation structure of GMC data.
    For Gaussian match-and-copy, we show in \Cref{subsec:gaussian-data} that these conditions
    hold with high probability as the embedding dimension $\din$ grows relative to $(T,n)$. 
    Assumption A3 (identifiability) prevents the existence of 
    finite-norm interpolants (\Cref{subsec:identifiability-margin-divergence}). 
    Without it, the loss could reach zero with finite weights, 
    and the margins would not need to diverge. 
    Assumption A4 (half-space) ensures that
    the gradient will eventually point in a direction 
    that increases all margins simultaneously 
    as the loss vanishes (see the discussion before \Cref{lem:mse-asymptotics-dvg-margins}). 
    
    \item \textbf{Successful optimization (A5--A6).}
    These assumptions condition on gradient descent reaching a low-loss regime with square-summable gradients.
    We do not attempt to characterize when this occurs; instead, 
    we ask what happens \emph{if} it does motivated by the empirical observation that this is typical. 

    \item \textbf{Trajectory stability (A7--A9).}
    These assumptions formalize an asymptotic stabilization regime of the MSE dynamics
    in which support vectors dominate and cross-terms remain controlled.
    They are not claimed to hold for all runs. 
    In \Cref{subsec:assumptions-hold}, we exhibit a representative trajectory 
    of the runs where we found this regime to hold empirically. 
\end{enumerate}
\end{remark}

\section{Proof of \Cref{thm:mse-bias-app}}
\label{app:mse-bias-proof}

We prove \Cref{thm:mse-bias-app} in this section, 
which is the formal version of \Cref{thm:mse-bias-main} in the main text. 

At a high level, 
the proof adapts the 
classical implicit-bias machinery
for separable losses 
\citep{Soudry18ImplicitBiasCE,JiTelgarsky19Nonseparable} 
to our setting, with additional 
assumptions and work to 
handle the challenges specific to MSE. 

For a fixed $\tilde w$ (chosen later), consider the target trajectory
\begin{equation}
    w^*(\tau) := f(\tau)\hat{w} + \tilde{w}, \qquad f(\tau) = \frac{1}{2}\log(\tau).
    \label{eq:target-trajectory}
\end{equation}
It is enough to show that the residual
$r(\tau) = w(\tau) - w^*(\tau)$ remains bounded as $\tau\to\infty$. 
Since 
$\sup_{\tau }\|r(\tau)\|^2 = \sup_{\tau} \sum_{s=1}^\tau \big(\|r(s+1)\|^2 - \|r(s)\|^2\big)$ 
and each increment can be decomposed as:
\begin{equation}
  \|r(\tau+1)\|^2 - \|r(\tau)\|^2
  = \underbrace{\|r(\tau+1)-r(\tau)\|^2}_{\energyinc(\tau): \text{ energy increment}}
    + 2\underbrace{\big\langle r(\tau+1)-r(\tau),\,r(\tau)\big\rangle}_{
      \drift(\tau): \text{ instantaneous drift}},
  \label{eq:norm-increment}
\end{equation}
it suffices to show that both series 
$\energyinc(\tau)$ and $\drift(\tau)$ are bounded from above by summable sequences.

Introduce the shorthand $v(\tau) := f(\tau+1)-f(\tau)$. 
The increment of the residual is given by:
\[
r(\tau+1) - r(\tau) 
= \underbrace{-\eta \nabla \L(w(\tau))}_{\text{GD step}}
 - \underbrace{v(\tau)\hat{w}}_{\text{target drift}}.
\]
In particular, $\energyinc(\tau)$ is summable.
Indeed, the inequality $\|a+b\|^2 \le 2(\|a\|^2+\|b\|^2)$ gives:
\[
      \energyinc(\tau)
      \le 2\eta^2 \|\nabla\L(w(\tau))\|^2
         + 2v(\tau)^2 \|\hat{w}\|^2
\]
and we have assumed the gradients are \Success{$\ell^2$-summable}, 
and the second term satisfies $v(\tau)^2 = O(1/\tau^2)$. 
We now study the drift $\drift(\tau)=\langle 
-\eta \nabla \L(w(\tau)) - 
v(\tau)\hat{w} 
,
r(\tau)\big\rangle$. 

\subsection{Drift decomposition}

Since the gradient can be expressed as 
$-\nabla\L = 
\sum_{(i,t)\in\ISV\cup(\ISV)^c} \pi^{\MSE}_{i,t}(w(\tau)) \tilde x_{i,t}$ 
(\Cref{lem:grads}), 
and KKT conditions give $\hat{w} = \sum_{(i,t)\in\ISV} \alpha_{i,t} \tilde{x}_{i,t}$, 
we can decompose the drift as a sum of a contribution from support vectors and a contribution from non-support vectors:
\begin{align}
    \drift(\tau) 
    &= \underbrace{\sum_{(i,t) \in \ISV} 
    \left( \eta \pi^{\MSE}_{i,t}(w(\tau)) 
    - v(\tau)\alpha_{i,t} \right) 
    \inner{\tilde{x}_{i,t}}{r(\tau)}}_{\DISV(\tau)} \nonumber \\
    &+ \underbrace{\sum_{(i,t) \notin \ISV} 
    \eta \pi^{\MSE}_{i,t}(w(\tau)) 
    \inner{\tilde{x}_{i,t}}{r(\tau)}}_{\DnonISV(\tau)}.
    \label{eq:drift-decomp-raw}
\end{align}

\subsection{Analysis of the support drift}

We analyze
\[
    \DISV(\tau)
    :=
    \sum_{(i,t)\in\ISV}
    \bigl(\eta \pi^{\MSE}_{i,t}(w(\tau)) - v(\tau)\alpha_{i,t}\bigr)\,u_{i,t}(\tau),
    \qquad
    u_{i,t}(\tau):=\inner{r(\tau)}{\tilde x_{i,t}}.
\]

To understand how to analyze this term, 
we begin with an informal explanation 
leading to the definition of the correction vector $\tilde w$ in \Cref{def:correction-vector} below.

\paragraph{Why the target scale $f(\tau)=\tfrac12\log\tau$ and the correction $\tilde w$ are forced 
if $r(\tau)$ is bounded.} 
We start building intuition by 
assuming what we want to show, i.e., that $r(\tau)$ remains bounded as $\tau\to\infty$. 
In particular, $u_{i,t}(\tau)=\inner{r(\tau)}{\tilde x_{i,t}}$ 
remains bounded as $\tau\to\infty$. 
We rely on informal asymptotics\footnote{A rigorous version of 
these asymptotics can be derived under bounded $r(\tau)$ using 
arguments similar to \Cref{lem:mse-asymptotics-dvg-margins}, 
we omit the details here as they are not needed for the rest of the proof.}  
and the reader can confidently skip the details here if desired 
and resume at the definition of $\tilde w$ in \Cref{def:correction-vector} below.

In the $r(\tau)$ bounded regime, it can be 
seen
that the leading contribution to the MSE gradient coefficients with $(i,t)\in\ISV$ is
\[
    \eta \pi^{\MSE}_{i,t}(w(\tau))
    \approx
    \eta h_{i,t}\,e^{-2f(\tau)}\,
    e^{-\tilde w^\top\tilde x_{i,t}}
    \sum_{s\in\ISV(i)}e^{-\tilde w^\top\tilde x_{i,s}},
\]
where $\ISV(i) := \{s : (i,s) \in \ISV\}$ is the set of support vectors for sample $i$, 
and where $h_{i,t}:=\lim_{\tau\to\infty} h_{i,t}(w(\tau))\in(0,\infty)$ 
automatically exists for bounded $r(\tau)$. 
These coefficients decay at rate $e^{-2f(\tau)}$.  
By contrast, non-support vector contributions can be shown to decay faster. 
So the gradient 
is supported asymptotically on the support vectors only: 
$-\nabla\L(w(\tau))\approx \sum_{(i,t)\in\ISV} \pi^{\MSE}_{i,t}(w(\tau)) \tilde x_{i,t}$ 
as $\tau\to\infty$, 
with a decay rate of order $e^{-2f(\tau)}$ for the coefficients.

Meanwhile the target trajectory has an increment 
$v(\tau)\hat w$ with 
$v(\tau)=f(\tau+1)-f(\tau)$, and 
if the 
target trajectory remains close to the GD iterates,
their respective increments must match asymptotically:
\[
    -\eta\nabla\L(w(\tau))\approx v(\tau)\hat w.
\]
Since $\nabla\L(w(\tau))$ is asymptotically supported on the support vectors only, 
and $\hat w=\sum_{(i,t)\in\ISV}\alpha_{i,t}\tilde x_{i,t}$, 
with linearly independent $\tilde x_{i,t}$'s by the \DataGeom{rank} assumption,
we must have, for each $(i,t)\in\ISV$,
\[    
\eta \pi^{\MSE}_{i,t}(w(\tau))
    \approx v(\tau)\alpha_{i,t}.
\] 
Based on the above asymptotics for $\eta \pi^{\MSE}_{i,t}(w(\tau))$, 
this can be rewritten as 
\[
    \eta h_{i,t}\,e^{-2f(\tau)}\,
    e^{-\tilde w^\top\tilde x_{i,t}}
    \sum_{s\in\ISV(i)}e^{-\tilde w^\top\tilde x_{i,s}}
    \approx
    v(\tau)\alpha_{i,t}.
\]
This forces $v(\tau) \approx f'(\tau)\asymp e^{-2f(\tau)}$ which implies $f(\tau)=\tfrac12\log\tau+O(1)$. 
We hence fix $f(\tau)=\tfrac12\log\tau$. 
Then $e^{2f(\tau)} v(\tau) \to \frac12$, so $\tilde w$ must be a solution of:
\begin{equation}
        \eta h_{i,t}\,e^{-\tilde w^\top\tilde x_{i,t}}
    \sum_{s\in\ISV(i)}e^{-\tilde w^\top\tilde x_{i,s}}
    =
    \frac12\,\alpha_{i,t},
    \qquad\forall (i,t)\in\ISV.
    \label{eq:def-tilde-w}
\end{equation}
\begin{definition}
[Correction vector $\tilde w$]
\label{def:correction-vector}
We define $\tilde w$ as the unique solution of the system \eqref{eq:def-tilde-w} 
in the linear span of the support vectors $\linspan(\ISV):=
\linspan\{\tilde x_{i,t} : (i,t)\in\ISV\}$.
\end{definition}
Existence and uniqueness follow from the following argument. 
Define $c_{i,t}:= \frac{\alpha_{i,t}}{2\eta h_{i,t}} > 0$ 
and consider $z_{i,t} := c_{i,t} / \sqrt{\sum_{s\in\ISV(i)} c_{i,s}}$, 
which is well defined since $c_{i,t}>0$ for all $(i,t)\in\ISV$.  
By \DataGeom{linear independence of the support vectors} $\{\tilde x_{i,t} : (i,t)\in\ISV\}$,
there exists a unique $\tilde w\in\linspan(\ISV)$ 
such that $\tilde w^\top \tilde x_{i,t} = -\log(z_{i,t})$ for all $(i,t)\in\ISV$. 
This $\tilde w$ is a solution of \eqref{eq:def-tilde-w}
(and the only one in $\linspan(\ISV)$). We fix this $\tilde w$ from now on and resume the proof of \Cref{thm:mse-bias-app}.

\paragraph{Softmax tail split and deterministic remainders.}

Split the softmax tail
\[
    S_i(w):=\sum_{s\neq y_i} e^{-w^\top \tilde x_{i,s}}
\]
into support and non-support parts:
\[
    S_i(w)= S_i^{\ISV}(w) + S_i^{\nonISV}(w),
    \qquad
    S_i^{\ISV}(w):=\sum_{s\in\ISV(i)} e^{-w^\top \tilde x_{i,s}},
    \qquad
    S_i^{\nonISV}(w):=\sum_{s\notin\ISV(i)} e^{-w^\top \tilde x_{i,s}}.
\]
For each $(i,t)\in\ISV$ define the  remainders
\begin{equation}
\begin{aligned}
    \rest^{\pi}_{i,t}(w)
    \;:=\;&
    \eta\pi^{\MSE}_{i,t}(w)
    - \eta h_{i,t}\,e^{-w^\top \tilde x_{i,t}}\,
    S_i^{\ISV}(w),
    \\
    \rest^{v}_{i,t}(\tau)
    \;:=\;&
    v(\tau)\alpha_{i,t}
    - \frac{1}{2\tau} \alpha_{i,t},
\end{aligned}
\label{eq:def-rest-coef}
\end{equation}
where $h_{i,t}=\lim_{\tau\to\infty}h_{i,t}(w(\tau))$ is the limit given by the 
\Converg{convergence of pre-factors} assumption.

With this notation,
\begin{equation}
\begin{aligned}
    \eta \pi^{\MSE}_{i,t}(w(\tau)) - v(\tau)\alpha_{i,t}
    &=
    \Big(\eta h_{i,t}e^{-w(\tau)^\top\tilde x_{i,t}}S_i^{\ISV}(w(\tau)) 
    - \frac{1}{2\tau}\alpha_{i,t}\Big)
    + \rest^{\pi}_{i,t}(w(\tau)) - \rest^{v}_{i,t}(\tau),
\end{aligned}
\label{eq:ISV-coef-split}
\end{equation}
and hence we split the support drift into a \emph{restoring} term plus a \emph{remainder} term:
\begin{equation}
\DISV(\tau)=\restoreISV(\tau)+\restISV(\tau),
\label{eq:ISV-drift-split}
\end{equation}
where
\begin{equation}
\begin{aligned}
    \restoreISV(\tau)
    &:=
    \sum_{(i,t)\in\ISV}
    \Big(
    \eta h_{i,t}e^{-w(\tau)^\top\tilde x_{i,t}}S_i^{\ISV}(w(\tau))
    - \frac{1}{2\tau}\alpha_{i,t}
    \Big)\,u_{i,t}(\tau),
    \\
    \restISV(\tau)
    &:=
    \sum_{(i,t)\in\ISV}
    \Big(\rest^{\pi}_{i,t}(w(\tau))-\rest^{v}_{i,t}(\tau)\Big)\,u_{i,t}(\tau).
\end{aligned}
\label{eq:restore-rest-defs}
\end{equation}

\paragraph{Step 1: $\restoreISV$ is restoring.}

Introduce the notation
\[
    p_{i,t} := e^{-\tilde w^\top \tilde x_{i,t}} > 0.
\] 
Note that by definition of $\tilde w$ in \eqref{eq:def-tilde-w}, it holds for all $(i,t)\in\ISV$:
\begin{equation}
    \eta h_{i,t}\,p_{i,t}
    \sum_{s\in\ISV(i)} p_{i,s}
    =
    \frac{1}{2}\alpha_{i,t}.
    \label{eq:implicit-w-tilde}
\end{equation}
We now plug $w(\tau)=f(\tau)\hat w+\tilde w+r(\tau)$ with 
$f(\tau)=\tfrac12\log\tau$ into $\restoreISV(\tau)$. 
For $(i,t)\in\ISV$,
\[
    e^{-w(\tau)^\top\tilde x_{i,t}}=\tau^{-1/2}p_{i,t}e^{-u_{i,t}(\tau)},
    \qquad
    S_i^{\ISV}(w(\tau))=\tau^{-1/2}\sum_{s\in\ISV(i)}p_{i,s}e^{-u_{i,s}(\tau)}.
\]
Therefore,
\begin{align}
    \restoreISV(\tau)
    &=
    \eta\,\tau^{-1}\sum_i\sum_{t\in\ISV(i)}
    h_{i,t}\,p_{i,t}\,u_{i,t}(\tau)\,
    \Bigg(
      e^{-u_{i,t}(\tau)}\sum_{s\in\ISV(i)}p_{i,s}e^{-u_{i,s}(\tau)}
      \;-\;
      \sum_{s\in\ISV(i)}p_{i,s}
    \Bigg)\\
    &=
    \eta\,\tau^{-1}\sum_i\sum_{t,s\in\ISV(i)}
    h_{i,t}\,p_{i,t}p_{i,s}\,u_{i,t}(\tau)\,
    \big(e^{-u_{i,t}(\tau)-u_{i,s}(\tau)}-1\big).
\label{eq:restoreISV-unsym}
\end{align}

Symmetrizing \eqref{eq:restoreISV-unsym} in $(s,t)$ yields 
(we add the same sum with $s$ and $t$ swapped and divide by 2):
\begin{equation}
    \restoreISV(\tau)
    =
    \frac{\eta}{2}\,\tau^{-1}\sum_i\sum_{s,t\in\ISV(i)}
    p_{i,t}p_{i,s}\,H_{i,t,s}(\tau)\,\big(e^{-U_{i,t,s}(\tau)}-1\big),
\label{eq:restoreISV-sym}
\end{equation}
with
\[
    H_{i,t,s}(\tau):=h_{i,t}u_{i,t}(\tau)+h_{i,s}u_{i,s}(\tau),
    \qquad
    U_{i,t,s}(\tau):=u_{i,t}(\tau)+u_{i,s}(\tau).
\]

For off-diagonal terms $s\neq t$, 
since $U(e^{-U}-1)\le 0$ for all $U\in\R$,
there might be positive contributions to the drift when 
$H$ and $U$ have opposite signs. 
However, after numerical inspections on 
Gaussian match-and-copy data, 
we find that eventually most off-diagonal contributions 
in \eqref{eq:restoreISV-sym} 
remain negative, i.e., 
most $s\neq t$ pairs satisfy $\sign(H_{i,t,s}(\tau))=\sign(U_{i,t,s}(\tau))$. 
To formalize this, we make the following assumption. 

\begin{assumption}[\Converg{Alignment Assumption (A8) in \Cref{thm:mse-bias-app}}]
\label{ass:alignment}
There exists $\eps\in(0,1/2)$ such that, for all sufficiently large $\tau$,
the total off-diagonal contribution in \eqref{eq:restoreISV-sym} 
coming from pairs $(s,t)\in\ISV(i)^2$ with $s\neq t$
and $\sign(H_{i,t,s}(\tau))\neq\sign(U_{i,t,s}(\tau))$ is at 
most an $\eps$-fraction (in absolute value) of the total off-diagonal mass:
\begin{equation}
    \label{eq:alignment-assumption}
  \sum_{
    \substack{s,t\in\ISV(i),\\ s\neq t,\\ \sign(H_{i,t,s}(\tau))\neq\sign(U_{i,t,s}(\tau))}
  }
  p_{i,t} p_{i,s}\,|H_{i,t,s}(\tau)|\,\big|e^{-U_{i,t,s}(\tau)}-1\big|
  \le \eps
  \sum_{
    \substack{s,t\in\ISV(i),\\ s\neq t}
  }
  p_{i,t} p_{i,s}\,|H_{i,t,s}(\tau)|\,\big|e^{-U_{i,t,s}(\tau)}-1\big|.
\end{equation}
\end{assumption}

Under this assumption, 
the off-diagonal terms have a non-positive contribution.
Indeed, writing for brevity
$H_{t,s}=H_{i,t,s}(\tau)$ and $U_{t,s}=U_{i,t,s}(\tau)$, 
and all sums being restricted to $s, t\in\ISV(i)$ with $s\neq t$ (off-diagonal terms):
\begin{align}
\sum
     p_{i,t} p_{i,s} \,
     H_{t,s} \big(e^{-U_{t,s}} - 1\big)
  &= \sum
     p_{i,t} p_{i,s} \,
     \big|H_{t,s}\big| \,\big|e^{-U_{t,s}} - 1\big|\,
     \underbrace{\sign(H_{t,s})\sign(e^{-U_{t,s}}-1)}_{=-\,\sign(H_{t,s}U_{t,s})} \nonumber\\
  &= - \sum
     p_{i,t} p_{i,s} \,
     \big|H_{t,s}\big| \,\big|e^{-U_{t,s}} - 1\big|\,
     \sign(H_{t,s}U_{t,s}) \nonumber\\
  &= 
     \sum
     p_{i,t} p_{i,s} \,
     \big|H_{t,s}\big| \,\big|e^{-U_{t,s}} - 1\big| \indicator_{H U < 0}
     \;-\;
     \sum
     p_{i,t} p_{i,s} \,
     \big|H_{t,s}\big| \,\big|e^{-U_{t,s}} - 1\big| \indicator_{H U > 0} \nonumber\\
  &= 
     2 \sum
     p_{i,t} p_{i,s} \,
     \big|H_{t,s}\big| \,\big|e^{-U_{t,s}} - 1\big| \indicator_{H U < 0}
     \;-\;
     \sum
     p_{i,t} p_{i,s} \,
     \big|H_{t,s}\big| \,\big|e^{-U_{t,s}} - 1\big| \nonumber\\
  &\le 
     -(1-2\eps)
     \sum
     p_{i,t} p_{i,s} \,
     \big|H_{t,s}\big| \,\big|e^{-U_{t,s}} - 1\big|
  \;\le\; 0 \label{eq:off-diag-negative-alignment}
\end{align}
which is negative since $\eps\in(0,1/2)$. 
We deduce the following lemma, which shows that $\restoreISV$ is a restoring force: 
it helps push the trajectory back towards the target trajectory when $u_{i,t}(\tau)$ is large. 

\begin{lemma}[Restoring force from support-tail of support drift]
\label{lem:isv-restoring}
Fix any $R>0$ large enough. Under \Converg{Assumption (A8) of alignment}, 
formalized in
\Cref{ass:alignment}, there exist constants $c>0$ and $\tau_1$
such that for all $\tau\ge\tau_1$,
\begin{equation}
    \restoreISV(\tau)
    \le
    -\frac{c}{\tau}\sum_{(i,t)\in\ISV}\Phi\big(u_{i,t}(\tau)\big),
    \label{eq:isv-restoring}
\end{equation}
where
\[
    \Phi(u):=
    |u|\Big(e^{2|u|}\,\indicator_{\{u\le -R\}}+\indicator_{\{u\ge R\}}\Big).
\]
\end{lemma}
\begin{proof}
Consider first diagonal terms $s=t$ in \eqref{eq:restoreISV-sym}.
They equal
\[
    p_{i,t}^2\,(2h_{i,t}u_{i,t}(\tau))\,(e^{-2u_{i,t}(\tau)}-1)\le 0.
\]
Since $u(e^{-2u}-1)\sim -|u|e^{2|u|}$ for $u\to-\infty$ and 
$u(e^{-2u}-1)\sim -u$ for $u\to\infty$, 
this diagonal term is close to $-2p_{i,t}^2 h_{i,t} \Phi(u_{i,t}(\tau))$ 
provided that $|u_{i,t}(\tau)|$ is large enough. 
In particular, there exists a constant $c>0$ 
(depending on $R$ and on $\min_{(i,t)\in\ISV}p_{i,t}^2 h_{i,t} > 0$) such that
the diagonal part is bounded above by
$-(c/\tau)\sum_{(i,t)\in\ISV}\Phi(u_{i,t}(\tau))$. 

Since off-diagonal terms $s\neq t$ remains non-positive 
\eqref{eq:off-diag-negative-alignment}, 
we get the upper bound in \eqref{eq:isv-restoring}.
\end{proof}

\paragraph{Step 2: $\restISV$ is compensated 
by the restoring force $\restoreISV$.} 

We bound the remainder drift
\[
    \restISV(\tau)
    =
    \sum_{(i,t)\in\ISV}
    \Big(\rest^{\pi}_{i,t}(w(\tau))-\rest^{v}_{i,t}(\tau)\Big)\,u_{i,t}(\tau).
\]

\paragraph{The time-discretization remainder $\rest^v$.}

Since $f(\tau)=\frac12\log\tau$, we have
\[
    v(\tau)=f(\tau+1)-f(\tau)
    =\frac12\log\Big(1+\frac{1}{\tau}\Big)
    =\frac{1}{2\tau} + O(\tau^{-2})
\]
hence there exists $c>0$ such that for all $\tau\ge 1$ and all $(i,t)\in\ISV$,
\begin{equation}
    |\rest^{v}_{i,t}(\tau)|
    =
    \Big|v(\tau)-\frac{1}{2\tau}\Big|\,\alpha_{i,t}
    \le
    c\,\tau^{-2}.
\label{eq:restv-abs}
\end{equation}
Therefore,
\begin{equation}
    \sum_{(i,t)\in\ISV} |\rest^{v}_{i,t}(\tau)|\,|u_{i,t}(\tau)|
    \le
    c\,\tau^{-2}\sum_{(i,t)\in\ISV}|u_{i,t}(\tau)|
    \le c'\,\tau^{-2}\Big(R+ \sum_{(i,t)\in\ISV}\Phi\big(u_{i,t}(\tau)\big)\Big),
\label{eq:restv-drift-bound}
\end{equation}
where we used $u \le R + \Phi(u)$ for all $u\in\R$ in the last inequality. 
The first term $c'R\tau^{-2}$ is summable in $\tau$,
while the second term can be absorbed by the restoring force \eqref{eq:isv-restoring} 
for $\tau$ large enough 
($c' \tau^{-2} \Phi(u) \le \frac{1}{100} c \tau^{-1} \Phi(u)$ 
for $\tau$ large enough so we still get a restoring force with modified constant $0.99 c$, say).

\paragraph{The gradient coefficient remainder $\rest^{\pi}$.}

\Cref{lem:grads} 
in the appendix gives the exact expression
for the MSE gradient coefficients:
\[
    \pi^{\MSE}_{i,t}(w)
    =
    a^{(i)}_{t+1}(w)\Big((1-a^{(i)}_{y_i+1}(w))\,h_{i,t}(w) - \|r^{(i)}(w)\|_2^2\Big),
    \quad \forall (i,t), w.
\]
Then,
\begin{align*}
    \rest^{\pi}_{i,t}(w(\tau)) u_{i,t}(\tau)
    =&
    \eta\,h_{i,t}(w(\tau)) u_{i,t}(\tau)
    \,\Big[
    a^{(i)}_{t+1}(w(\tau))\,(1-a^{(i)}_{y_i+1}(w(\tau)))
    - e^{-w(\tau)^\top\tilde x_{i,t}}S_i(w(\tau))
    \Big]
    \\
    &
    +\eta\,e^{-w(\tau)^\top\tilde x_{i,t}}S_i(w(\tau))\,u_{i,t}(\tau)\,
    \Big[h_{i,t}(w(\tau))-h_{i,t}\Big]
    \\
    &
    +\eta\,h_{i,t}\,e^{-w(\tau)^\top\tilde x_{i,t}}\,u_{i,t}(\tau)\,
    \Big[S_i(w(\tau))-S_i^{\ISV}(w(\tau))\Big]
    \\
    &
    -\eta\,a^{(i)}_{t+1}(w(\tau))\,\|r^{(i)}(w(\tau))\|_2^2\,u_{i,t}(\tau).
\end{align*}
We treat each of these four terms separately. 
We will use repeatdly that for $(i,t)\in\ISV$, 
expanding $w(\tau)=f(\tau)\hat w+\tilde w+r(\tau)$ with $f(\tau)=\tfrac12\log\tau$ 
and $u_{i,t}(\tau)=\inner{r(\tau)}{\tilde x_{i,t}}$ 
gives 
\[
    e^{-w(\tau)^\top\tilde x_{i,t}}
    =
    \tau^{-1/2}p_{i,t}e^{-u_{i,t}(\tau)},
\]
with $p_{i,t}=e^{-\tilde w^\top \tilde x_{i,t}}>0$ 
a constant independent of $\tau$ (only depending on the data 
through $\tilde w$ and $\tilde x_{i,t}$), 
and gives $(i,s)\notin\ISV$, 
\[
    e^{-w(\tau)^\top\tilde x_{i,s}}
    =
    \tau^{-1/2 - \theta_{i,t}/2} p_{i,s} e^{-u_{i,s}(\tau)}.
\]
\begin{enumerate}
    \item Since $a^{(i)}_{t+1}(w)(1-a^{(i)}_{y_i+1}(w)) 
    = e^{-w^\top\tilde x_{i,t}}S_i(w)/(1+S_i(w))^2$ for all $w$, 
    the first term equals
    \[
        \eta\,h_{i,t}(w(\tau))\,u_{i,t}(\tau)
        \,e^{-w(\tau)^\top\tilde x_{i,t}}\,S_i(w(\tau))
        \underbrace{\Big(\frac{1}{(1+S_i(w(\tau)))^2}-1\Big)}_{\le 0}.
    \]
    If $u_{i,t}(\tau)\ge 0$, this term is non-positive 
    and keeps the sum bounded by something summable in $\tau$.
    If $u_{i,t}(\tau) < 0$, 
    it is positive and 
    using that $S_i(w(\tau)) = S_i^{\ISV}(w(\tau)) + S_i^{\nonISV}(w(\tau)) 
    \le (1+g(\tau)) S_i^{\ISV}(w(\tau)) \le c \tau^{-1/2} e^{|u_{\min(\ISV)}(\tau)|}$ 
    by the \Converg{softmax domination} assumption, 
    we can bound it as
    \begin{align*}
        & \underbrace{c \eta\,h_{i,t}(w(\tau))\,p_{i,t}}_{
            0<\cdot < c\eta h_{\max} p_{\max}
        }
        \, \tau^{-1/2} 
        \underbrace{|u_{i,t}(\tau)| e^{|u_{i,t}(\tau)|}}_{
            \le 
            |u_{\min(\ISV)}(\tau)| e^{|u_{\min(\ISV)}(\tau)|}
        }
        \,
        \tau^{-1/2} p_{\max} e^{|u_{\min(\ISV)}(\tau)|}
        \,
        \underbrace{\Big(1 - \frac{1}{(1+S_i(w(\tau)))^2}\Big)}_{
            = \frac{S_i(2+S_i)}{(1+S_i)^2}
            \le 2S_i
        }
        \\
        \le & 
        c'\,\tau^{-1} 
        \,|u_{\min(\ISV)}(\tau)| e^{2|u_{\min(\ISV)}(\tau)|}
        \,S_i(w(\tau)).
    \end{align*}
    Either $u_{\min(\ISV)}(\tau) \ge -R$, in which case the term is 
    bounded by $c'\tau^{-1} S_i(w(\tau))$, 
    and using again that $S_i(w(\tau)) \le c \tau^{-1/2} e^{|u_{\min(\ISV)}(\tau)|}$ 
    by the \Converg{softmax domination} 
    assumption, this is bounded by $c''\tau^{-3/2}$ 
    which is summable in $\tau$;
    or $u_{\min(\ISV)}(\tau) \le -R$, in 
    which case $|u_{\min(\ISV)}(\tau)| e^{2|u_{\min(\ISV)}(\tau)|} = \Phi(u_{\min(\ISV)}(\tau))$ 
    and since $S_i(w(\tau))\to 0$ as $\tau\to\infty$ 
    by divergence of the margins (\Cref{lem:identifiability-margin-divergence}), 
    this becomes negligible compared to the restoring force \eqref{eq:isv-restoring}
    for $\tau$ large enough. In both cases, the first term is controlled.
    \item For the second term 
    \[
    \eta\,e^{-w(\tau)^\top\tilde x_{i,t}}S_i(w(\tau))\,u_{i,t}(\tau)\,(h_{i,t}(w(\tau))-h_{i,t})
    \]
    we bound it in absolute value using that 
    $S_i(w(\tau)) \le c \tau^{-1/2} e^{|u_{\min(\ISV)}(\tau)|}$ 
    by the \Converg{softmax domination} assumption,
    yielding the bound
    \[
        c\,\tau^{-1} e^{-u_{i,t}(\tau)} |u_{i,t}(\tau)|
        e^{-u_{\min(\ISV)}(\tau)}
        \,|h_{i,t}(w(\tau))-h_{i,t}|.
    \]
    If $u_{\min(\ISV)}(\tau) \ge -R$, 
    then this is bounded by $c'\tau^{-1} e^{-u_{i,t}(\tau)} |u_{i,t}(\tau)| |h_{i,t}(w(\tau))-h_{i,t}|$. 
    Either $u_{i,t}(\tau)\ge 0$, in which case $|u_{i,t}(\tau)| e^{-u_{i,t}(\tau)} \le 1$, 
    or $-R\leq u_{i,t}(\tau)\le 0$, 
    and in both cases the term is bounded by
    \[
        c''\,\tau^{-1}
        \,|h_{i,t}(w(\tau))-h_{i,t}|,
    \]
    which is summable in $\tau$ by the \Converg{pre-factor convergence rate} assumption.
    If $u_{\min(\ISV)}(\tau) \le -R$, 
    then again if $u_{i,t}(\tau)\ge 0$, we have $|u_{i,t}(\tau)| e^{-u_{i,t}(\tau)} \le 1$, 
    and the term is bounded by
    \[
        c\,\tau^{-1} \Phi(u_{\min(\ISV)}(\tau))
        \,|h_{i,t}(w(\tau))-h_{i,t}|.
    \]
    Or if $u_{i,t}(\tau)\le 0$, then $|u_{i,t}(\tau)| e^{-u_{i,t}(\tau)} 
    \le |u_{\min(\ISV)}(\tau)| e^{|u_{\min(\ISV)}(\tau)|}$ and the term can again be bounded by
    \[
        c\,\tau^{-1} |u_{\min(\ISV)}(\tau)| e^{2|u_{\min(\ISV)}(\tau)|}
        \,|h_{i,t}(w(\tau))-h_{i,t}|
        = c\,\tau^{-1} \Phi(u_{\min(\ISV)}(\tau))
        \,|h_{i,t}(w(\tau))-h_{i,t}|.
    \]
    In both cases, we have the same bound that can be absorbed 
    by the restoring force \eqref{eq:isv-restoring} 
    since $|h_{i,t}(w(\tau))-h_{i,t}|\to 0$ as $\tau\to\infty$ 
    by the \Converg{convergence of pre-factors} assumption.
    \item For the third term
    \[
    \eta\,h_{i,t}\,e^{-w(\tau)^\top\tilde x_{i,t}}\,u_{i,t}(\tau)\,(S_i(w(\tau))-S_i^{\ISV}(w(\tau))),
    \]
    we note that $S_i(w(\tau))-S_i^{\ISV}(w(\tau))=S_i^{\nonISV}(w(\tau))\ge 0$. 
    For $u_{i,t}(\tau)\le 0$, this term is non-positive
    and keeps the sum bounded by something summable in $\tau$. 
    For $u_{i,t}(\tau)\ge 0$, we use that $S_i^{\nonISV}(w(\tau))
    \le g(\tau) S_i^{\ISV}(w(\tau))
    \le c g(\tau) \tau^{-1/2} e^{|u_{\min(\ISV)}(\tau)|} 
    $ by the \Converg{softmax domination} assumption, hence the term is bounded by
    \[
        c'\,\tau^{-1} \underbrace{e^{-u_{i,t}(\tau)} |u_{i,t}(\tau)|}_{\le 1}
        \,g(\tau) \,e^{|u_{\min(\ISV)}(\tau)|}.
    \]
    If $u_{\min(\ISV)}(\tau) \ge -R$, this is bounded by $c'' \tau^{-1} g(\tau)$,
    which is summable in $\tau$ by the \Converg{softmax domination rate} assumption.
    If $u_{\min(\ISV)}(\tau) \le -R$, this is bounded by
    \[
        c'\,\tau^{-1} g(\tau) \,\Phi(u_{\min(\ISV)}(\tau)),
    \]
    which can be absorbed by the restoring force \eqref{eq:isv-restoring}
    since $g(\tau)\to 0$ as $\tau\to\infty$ by the \Converg{softmax domination} assumption.
    \item For the fourth term
    \[
    -\eta\,a^{(i)}_{t+1}(w(\tau))\,\|r^{(i)}(w(\tau))\|_2^2\,u_{i,t}(\tau),
    \]
    it is non-positive when $u_{i,t}(\tau)\ge 0$ and keeps the sum 
    bounded by something summable in $\tau$. 
    For $u_{i,t}(\tau)\le 0$, we have $e^{-u_{i,t}(\tau)} |u_{i,t}(\tau)| 
    \le |u_{\min(\ISV)}(\tau)| e^{|u_{\min(\ISV)}(\tau)|}$. 
    We also have that $\|r^{(i)}(w(\tau))\|_2^2 \le C S_i(w(\tau))^2 
    $ by \Cref{lem:identifiability-margin-divergence} (no assumption needed here). 
    So it is bounded in absolute value by
    \[
        c\,\tau^{-1/2} |u_{\min(\ISV)}(\tau)| e^{|u_{\min(\ISV)}(\tau)|} S_i(w(\tau))^2.
    \]
    Recall that $S_i(w(\tau)) = c \tau^{-1/2} e^{|u_{\min(\ISV)}(\tau)|}$ 
    by the \Converg{softmax domination} assumption. 
    If $u_{\min(\ISV)}(\tau) \ge -R$, then $S_i(w(\tau))^2 \le c' \tau^{-1}$ 
    and the term is bounded by $c'' \tau^{-3/2}$ which is summable in $\tau$. 
    If $u_{\min(\ISV)}(\tau) \le -R$, then we bound one $S_i(w(\tau))$ 
    by $c \tau^{-1/2} e^{|u_{\min(\ISV)}(\tau)|}$ and keep the other, 
    so the term is bounded by
    \[
        c\,\tau^{-1} |u_{\min(\ISV)}(\tau)| e^{2|u_{\min(\ISV)}(\tau)|}
        \,S_i(w(\tau))
        = c\,\tau^{-1} \Phi(u_{\min(\ISV)}(\tau))
        \,S_i(w(\tau)),
    \]
    which can be absorbed by the restoring force \eqref{eq:isv-restoring}
    since $S_i(w(\tau))\to 0$ as $\tau\to\infty$ by divergence of the margins (\Cref{lem:identifiability-margin-divergence}).
\end{enumerate}

\subsection{Analysis of non-support drift}

We now bound
\[
    \DnonISV(\tau)
    =\sum_{(i,t)\notin\ISV}\eta\,\pi^{\MSE}_{i,t}(w(\tau))\,u_{i,t}(\tau).
\]
Recall the gradient coefficients given by \Cref{lem:grads}:
\[
    \pi^{\MSE}_{i,t}(w)
    =
    a^{(i)}_{t+1}(w)\Big((1-a^{(i)}_{y_i+1}(w))\,h_{i,t}(w) - \|r^{(i)}(w)|_2^2\Big),
    \quad \forall (i,t), w.
\]
We treat separately the two terms in the parenthesis. 
\begin{enumerate}
    \item For the first part of the drift given by
    \[
        \sum_{(i,t)\notin\ISV}
        \eta\,a^{(i)}_{t+1}(w(\tau))\,(1-a^{(i)}_{y_i+1}(w(\tau)))\,h_{i,t}(w(\tau))\,u_{i,t}(\tau),
    \]
    we note that if $u_{i,t}(\tau)\le 0$, this term is non-positive
    and keeps the sum bounded by something summable in $\tau$. 
    For $u_{i,t}(\tau)\ge 0$, we use that
    \[
        a^{(i)}_{t+1}(w(\tau))\,(1-a^{(i)}_{y_i+1}(w(\tau)))
        =
        \frac{e^{-w(\tau)^\top\tilde x_{i,t}}\,S_i(w(\tau))}{(1+S_i(w(\tau)))^2}
        \le
        e^{-w(\tau)^\top\tilde x_{i,t}}\,S_i(w(\tau)).
    \]
    Therefore the term is bounded by
    \[
        c\,e^{-w(\tau)^\top\tilde x_{i,t}}\,S_i(w(\tau))\,u_{i,t}(\tau).
    \]
    Since $(i,t)\notin\ISV$, we have
    \[
        e^{-w(\tau)^\top\tilde x_{i,t}}
        = \tau^{-(1+\theta_{i,t})/2}\,p_{i,t} \,e^{-u_{i,t}(\tau)}.
    \]
    Moreover, by the \Converg{softmax domination} assumption,
    we have $S_i(w(\tau)) \le c \tau^{-1/2} e^{|u_{\min(\ISV)}(\tau)|}$.
    Therefore the term is bounded by
    \[
        c'\,\tau^{-1-\theta_{i,t}/2}\,
        \underbrace{u_{i,t}(\tau)
        e^{-u_{i,t}(\tau)}}_{\le 1}
        \,e^{|u_{\min(\ISV)}(\tau)|}.
    \]
    If $u_{\min(\ISV)}(\tau) \ge -R$, this is bounded by
    $c'' \tau^{-1-\theta_{i,t}/2}$ which is summable in $\tau$.
    If $u_{\min(\ISV)}(\tau) \le -R$, then it is bounded by $c'\,\tau^{-1-\theta_{i,t}/2}\,
    \Phi(u_{\min(\ISV)}(\tau))$,
    which can be absorbed by the restoring force \eqref{eq:isv-restoring}
    since $\tau^{-\theta_{i,t}/2}\to 0$. 
    \item For the second part of the drift given by
    \[
        -\sum_{(i,t)\notin\ISV}
        \eta\,a^{(i)}_{t+1}(w(\tau))\,\|r^{(i)}(w(\tau))\|_2^2\,u_{i,t}(\tau),
    \]
    we note that if $u_{i,t}(\tau)\ge 0$, this term is non-positive
    and keeps the sum bounded by something summable in $\tau$. 
    For $u_{i,t}(\tau)\le 0$, the derivations are similar to those of 
    the fourth term we have treated in the previous section for the support drift, 
    but it decays even faster here since $(i,t)\notin\ISV$. 
    We have
    \[
        |u_{i,t}(\tau)| e^{-u_{i,t}(\tau)}
        \le
        |u_{\min(\ISV)}(\tau)| e^{|u_{\min(\ISV)}(\tau)|}.
    \]
    We also have that $\|r^{(i)}(w(\tau))\|_2^2 \le C S_i(w(\tau))^2 
    $ by \Cref{lem:identifiability-margin-divergence} (no assumption needed here).
    So the term is bounded in absolute value by
    \[
        c\,\tau^{-(1+\theta_{i,t})/2}
        \,|u_{\min(\ISV)}(\tau)| e^{|u_{\min(\ISV)}(\tau)|}
        \,S_i(w(\tau))^2.
    \]
    Recall that $S_i(w(\tau)) \le c \tau^{-1/2} e^{|u_{\min(\ISV)}(\tau)|}$ 
    by the \Converg{softmax domination} assumption. 
    If $u_{\min(\ISV)}(\tau) \ge -R$, then $S_i(w(\tau))^2 \le c' \tau^{-1}$ 
    and the term is bounded by $c'' \tau^{-3/2 - \theta_{i,t}/2}$ which is summable in $\tau$. 
    If $u_{\min(\ISV)}(\tau) \le -R$, then we bound one $S_i(w(\tau))$ 
    by $c \tau^{-1/2} e^{|u_{\min(\ISV)}(\tau)|}$ and keep the other, 
    so the term is bounded by
    \[
        c\,\tau^{-1 - \theta_{i,t}/2}
        \,|u_{\min(\ISV)}(\tau)| e^{2|u_{\min(\ISV)}(\tau)|}
        \,S_i(w(\tau))
        = c\,\tau^{-1 - \theta_{i,t}/2}
        \,\Phi(u_{\min(\ISV)}(\tau))
        \,S_i(w(\tau)),
    \]
    which can be absorbed by the restoring force \eqref{eq:isv-restoring}
    since $\tau^{-\theta_{i,t}/2} S_i(w(\tau))\to 0$ as $\tau\to\infty$ 
    by divergence of the margins (\Cref{lem:identifiability-margin-divergence}).
\end{enumerate}
This concludes the proof.
\qed

\section{Proof Details for \Cref{thm:mse-bias-app}}
\label{app:mse-proof-details}

This section contains:
\begin{itemize}
    \item \Cref{subsec:grads}: Gradient computations for MSE and CE losses, with a discussion about their differences.
    \item \Cref{subsec:convexity}: Proof that the MSE loss is non-convex in the merged key-query parameters, while CE is convex.
\end{itemize}
\subsection{Gradient coefficients}
\label{subsec:grads}
\begin{lemma}[Per-sample gradients, feature-space form]\label{lem:grads}
Fix one sample $(\Ectx,\query,t_0)$ 
with $T_0=t_0+1\in\{2,\dots,T\}$. 
Recall the model output $\hat y(\Wkq)=\Ectx\,\attn(z(\Wkq))$ 
with linear logits and shifted softmax:
\[
z(\Wkq):=\lambda \Ectx^\top \Wkq \query\in\R^T,
\qquad
\attn(z):=\begin{pmatrix}0\\ \softmax(z_{1:T-1})\end{pmatrix}\in\R^T.
\]
Consider the losses 
\[
\MSE(\Wkq)=\tfrac12\|\Ectx\attn(z(\Wkq))-\token_{T_0}\|^2,
\qquad
\CE(\Wkq)=-\log(\attn_{T_0}(z(\Wkq))).
\]
Let $\indicator_{T_0}\in\R^T$ be the canonical basis vector and 
\[
\shift=\begin{pmatrix}
    0_{1\times (T-1)} & 0\\
    I_{T-1} & 0_{(T-1)\times 1}
\end{pmatrix},
\qquad
J(\attn)=\diag(\attn)-\attn\attn^\top.
\]
Then
\[
\nabla_{\Wkq}\MSE(\Wkq)
=\lambda\,\Ectx\,\shift^\top J(\attn)\,\Ectx^\top \Ectx\,(\attn-\indicator_{T_0})\,\query^\top,
\qquad
\nabla_{\Wkq}\CE(\Wkq)
=\lambda\,\Ectx\,\shift^\top(\attn-\indicator_{T_0})\,\query^\top.
\]

\smallskip
\noindent
\emph{Equivalently, in the flattened feature notation of \Cref{thm:mse-bias-app}},
recall the notation $w=\flatten(\Wkq)$ and for $t=1,\dots,T-1$,
\[
x_t:=\lambda\,\flatten(\token_t\,\query^\top)\in\R^{\din^2},
\qquad
y:=t_0=T_0-1\in\{1,\dots,T-1\},
\qquad
\tilde x_t:=x_y-x_t\quad(t\neq y).
\]
Then $\nabla \L(w) = - \sum_{t \neq y} \pi_{t}(w) \tilde{x}_{t}$ 
for $\L\in\{\CE,\MSE\}$, with coefficients given by
\begin{align}
\pi^{\MSE}_{t}(w) & =\lambda\,a_{t+1}(w)\,\big(v_{t+1}(w)-\bar v(w)\big),
\label{eq:mse-grad-tilde-form-lemma}\\
\pi^{\CE}_{t}(w)&=a_{t+1}(w)
\label{eq:ce-grad-tilde-form-lemma}
\end{align}
where 
$v(w)=\Ectx^\top\Ectx(\attn(w)-\indicator_{T_0})\in\R^T, 
\bar v(w):=\sum_{s=1}^{T} \attn_s(w)\,v_s =\attn^\top v$. 
An equivalent expression for the MSE coefficient, useful to analyze its sign 
under small residual 
$r(w):=\Ectx\,\attn(w)-\token_{T_0}$, 
is given next. 
Let 
$d_t:=\token_{t+1}-\token_{T_0}$ and
$h_t(w):=\sum_{s\neq y} \frac{a_{s+1}(w)}{1-a_{y+1}(w)}\,\langle d_t,d_s\rangle$; 
we have
\begin{equation}
\label{eq:v-difference}
v_{t+1}(w) - \bar v(w)=
\langle d_t, r(w)\rangle
-
\|r(w)\|_2^2
=
(1-a_{y+1}(w))\,h_t(w)
- \|r(w)\|_2^2.
\end{equation}
\end{lemma}

\begin{proof}
Write the model output as $\hat y(\Wkq):=\Ectx\,\attn(z(\Wkq))$.
For any direction $\Delta$,
\[
Dz(\Wkq)[\Delta]=\lambda\,\Ectx^\top \Delta \query.
\]
We first work out the Jacobian of the \emph{shifted} softmax map $\attn(\cdot)$.
For $t\ge2$ and $s\le T-1$,
\begin{equation}
\label{eq:partial-shifted-attention}
\frac{\partial \attn_t}{\partial z_s} 
= 
\frac{\partial}{\partial z_s} \frac{\exp(z_{t-1})}{\sum_{r=1}^{T-1}\exp(z_r)}
=
    \frac{\exp(z_{t-1})}{\sum_{r=1}^{T-1}\exp(z_r)} \indicator_{s=t-1}
    -
    \frac{\exp(z_{t-1})\exp(z_{s})}{(\sum_{r=1}^{T-1}\exp(z_r))^2}
= \attn_t \indicator_{s=t-1} - \attn_t \attn_{s+1}.
\end{equation}
and all other partials are zero. This is exactly the matrix identity 
$D\attn(z)=J(\attn(z))\,\shift$ since
\begin{equation}
    \label{eq:shifted-attention-jacobian}
    \begin{aligned}
        J(\attn) \shift
    =& 
    \begin{pmatrix}
        0 & 0_{1\times T-1} \\
        0_{T-1\times 1} & \text{diag}(\attn_{2:T}) - \attn_{2:T} \attn_{2:T}^\top
    \end{pmatrix} \begin{pmatrix}
        0_{1\times T-1} & 0\\
        I_{T-1} & 0_{T-1\times 1}
    \end{pmatrix}\\
    =& \begin{pmatrix}
        0_{1\times T-1} & 0\\
        \text{diag}(\attn_{2:T}) - \attn_{2:T} \attn_{2:T}^\top & 0_{T-1\times 1}
    \end{pmatrix}
    = D\attn(z).
    \end{aligned}
\end{equation}
Hence, by the chain rule,
\[
D(\attn\circ z)(\Wkq)[\Delta]
= D\attn(z(\Wkq))[Dz(\Wkq)[\Delta]]
=\lambda\,J(\attn(z(\Wkq)))\,\shift\,\Ectx^\top \Delta \query.
\]

\paragraph{MSE.}
Let $r(\Wkq):=\hat y(\Wkq)-\token_{T_0}$, so $\MSE(\Wkq)=\tfrac12\|r(\Wkq)\|^2$ and
\[
D\MSE(\Wkq)[\Delta]=\inner{r(\Wkq)}{Dr(\Wkq)[\Delta]}
=\inner{r(\Wkq)}{\Ectx\,D(\attn\circ z)(\Wkq)[\Delta]}.
\]
Plugging in the expression from above yields
\[
D\MSE(\Wkq)[\Delta]
=\lambda\,\inner{r}{\Ectx\,J(\attn)\,\shift\,\Ectx^\top \Delta \query}
=\inner{\lambda\,\Ectx\,\shift^\top J(\attn)\,\Ectx^\top r\,\query^\top}{\Delta}_F.
\]
Finally note $r=\Ectx\attn-\token_{T_0}=\Ectx(\attn-\indicator_{T_0})$, 
giving the claimed formula:
\[
\nabla_\Wkq \MSE(\Wkq) = \lambda\,\Ectx\,\shift^\top J(\attn)\,\Ectx^\top\Ectx(\attn-\indicator_{T_0})\,\query^\top. 
\]

In the flattened feature notation, we use that for any $c\in\R^T$, 
we have $\Ectx c= \sum_{t=1}^T c_t \token_t$, 
hence $\Ectx c \query^\top = \sum_{t=1}^T c_t \token_t \query^\top$ and
\begin{equation}
\flatten\big(\Ectx c\,\query^\top\big)
=\sum_{t=1}^T c_t\,\flatten(\token_t\query^\top)
=\sum_{t=1}^{T} c_{t}\,x_t.
\label{eq:vec-Ecqt}
\end{equation}

Set $v:=\Ectx^\top\Ectx(\attn-\indicator_{T_0})=\Ectx^\top r\in\R^T$ 
and $\bar v:=\attn^\top v$.
Using $J(\attn)v=\diag(\attn)v-\attn(\attn^\top v)=(\attn_s(v_s-\bar v))_{s=1}^T$ 
and \eqref{eq:vec-Ecqt},
\begin{align*}
\nabla \MSE(w)
&=\flatten\big(\nabla_{\Wkq}\MSE(\Wkq)\big)
=\lambda\,\flatten\big(\Ectx\,\shift^\top J(\attn)\,v\,\query^\top\big)\\
&=\sum_{t=1}^{T-1}\big(J(\attn)v\big)_{t+1}\,x_t
=\sum_{t=1}^{T-1} a_{t+1}\,\big(v_{t+1}-\bar v\big)\,x_t.
\end{align*}
Finally, since $\sum_{t=1}^{T-1}a_{t+1}\,(v_{t+1}-\bar v)
=\sum_{s=2}^{T}\attn_s(v_s-\bar v)=\attn^\top v-\bar v\sum_s\attn_s=0$,
we can subtract $0\cdot x_y=
\sum_{t=1}^{T-1} a_{t+1}\,\big(v_{t+1}-\bar v\big)\,x_y
$ and rewrite
\[
\nabla \MSE(w)=\sum_{t\neq y} a_{t+1}\,\big(v_{t+1}-\bar v\big)\,(x_t-x_y)
=-\sum_{t\neq y}\pi^{\MSE}_{t}(w)\,\tilde x_t.
\]
Note that an equivalent expression for $v$ is $v=\Ectx^\top r$, so that 
$v_{t+1} - \bar v = \token_{t+1}^\top r - a^\top \Ectx^\top r = 
(\token_{t+1} - \Ectx a)^\top r$. 
Since $\token_{t+1}-\Ectx a=(\token_{t+1}-\token_{T_0})-(\Ectx a-\token_{T_0})=d_t-r$ 
with $d_t:=\token_{t+1}-\token_{T_0}$, we have the alternative expression
\begin{equation}
v_{t+1}-\bar v=\inner{d_{t+1}-r}{r}=\inner{d_t}{r}-\|r\|^2.
\label{eq:alignment}
\end{equation}
Moreover, $r=\Ectx a-\token_{T_0} 
= 
\sum_{s=1}^{T-1} a_{s+1} \token_{s+1} - 
(\sum_{s=1}^{T-1} a_{s+1})\token_{T_0}
= \sum_{s\neq y} a_{s+1} d_s$ so that 
$\inner{d_t}{r}
=\sum_{s\neq y} a_{s+1} \langle d_t, d_s\rangle
=(1-a_{y+1}) h_t$,
which finishes the proof \eqref{eq:v-difference} 
and thus proves the lemma for MSE. 

\paragraph{CE.}
We have $\CE(\Wkq)=-\log(\attn_{T_0}(z(\Wkq)))$, hence
\[
D\CE(\Wkq)[\Delta]
=-\frac{1}{\attn_{T_0}}\,D\attn_{T_0}(z(\Wkq))[Dz(\Wkq)[\Delta]].
\]
Using $D\attn(z)=J(\attn(z))\shift$ and $Dz(\Wkq)[\Delta]=\lambda\Ectx^\top\Delta\query$,
\[
D\CE(\Wkq)[\Delta]
=-\frac{\lambda}{\attn_{T_0}}\,(\indicator_{T_0}^\top J(\attn)\shift)\,\Ectx^\top\Delta\query.
\]
Since $\indicator_{T_0}^\top J(\attn)=\attn_{T_0}(\indicator_{T_0}^\top-\attn^\top)$, we have
\[
-\frac{1}{\attn_{T_0}}\,\indicator_{T_0}^\top J(\attn)
= \attn^\top-\indicator_{T_0}^\top,
\]
and thus
\[
D\CE(\Wkq)[\Delta]
=\lambda\,(\attn-\indicator_{T_0})^\top \shift\,\Ectx^\top\Delta\query
=\inner{\lambda\,\Ectx\,\shift^\top(\attn-\indicator_{T_0})\,\query^\top}{\Delta}_F,
\]
which gives the gradient formula. In the flattened feature notation, this gives:
\[
\nabla \CE(w)
=\flatten\big(\nabla_{\Wkq}\CE(\Wkq)\big)
=\lambda\,\flatten\big(\Ectx\,\shift^\top(\attn-\indicator_{T_0})\,\query^\top\big)
=\sum_{t=1}^{T-1}\Big((\shift^\top\attn)_t-\indicator_{t=y}\Big)\,x_t.
\]
Since $(\shift^\top\attn)_t=\attn_{t+1}$ and $\sum_{t=1}^{T-1}a_{t+1}=1$ 
so that $x_y = x_y \sum_{t=1}^{T-1}a_{t+1}$, we get
\[
\nabla \CE(w)=\sum_{t=1}^{T-1}a_{t+1}(w)\,x_t-x_y
=-\sum_{t\neq y}a_t(w)\,(x_y-x_t)
=-\sum_{t\neq y}\pi_t^{\CE}(w)\,\tilde x_t,
\]
as announced.
\end{proof}

\begin{remark}[Sign of gradient coefficients for MSE vs CE]
\label{rem:ce-comparison}
\Cref{lem:grads} shows that both the MSE and CE loss 
have gradients that can be expressed as linear combinations of 
the same set of vectors $\tilde x_{i,t} = x_{i,y_i} - x_{i,t}$ for $t\neq y_i$:
\begin{equation}
\nabla \L(w)
=-\sum_{i=1}^n \sum_{t\neq y_i} \pi^{\L}_{i,t}(w)\,\tilde x_{i,t}. 
\end{equation}
A GD step thus moves the weights $w$ in a direction given 
by a weighted combination of the $\tilde x_{i,t}$'s, 
with coefficients $\pi^{\L}_{i,t}(w)$ depending on the loss $\L$. 
\begin{itemize}
\item If $\pi_t^{\L}>0$, we move \emph{toward} $\tilde x_t$ 
and increase the margin $w^\top(x_y-x_t)$, i.e., 
the inner product with the correct class $y$ (the $y$'s logit) 
increases relative to class $t$.
\item If $\pi_t^{\L}<0$, we move \emph{opposite} to $\tilde x_t$ 
and raise class $t$’s score relative to the true class.
\end{itemize}
For $\CE$, all coefficients are positive, 
so \emph{every step monotonically increases the margins}. 
For $\MSE$, 
coefficients can have either sign: 
sometimes (e.g., early on) increasing a wrong logit helps reduce squared error 
by pulling the mean $\Ectx a=\sum_{t} a_t \token_t$ toward $\token_{T_0}$. 
This comes from the fact that MSE 
also cares about the $\ell^2$ geometry of the embeddings $\token_t$. 
Dropping the sample index~$i$ for readability, the CE coefficients are
$\pi^{\CE}_{i,t}
=a_{t+1}$ 
while $\pi^{\MSE}_{i,t} = a_{t+1}(v_{t+1}-\bar v)$ 
has an additional factor $v_{t+1}-\bar v$ which determines its sign.  
This term $v_{t+1}-\bar v$ in $\pi^{\MSE}_t$ 
encodes how changing $a_{t+1}$ affects the residual 
$r=\Ectx a - \token_{t_0+1}$. 
If we infinitesimally shift attention toward 
$t+1$ via $a\mapsto a+\eps(\indicator_{t+1}-a)$, then
\begin{equation}
\frac{d}{d\eps}\Big|_{\eps=0}\ \frac12\|\Ectx(a+\eps(\indicator_{t+1}-a))-\token_{T_0}\|^2
=\inner{\token_{t+1}-\Ectx a}{r}
=v_{t+1}-\bar v.
\label{eq:first-order}
\end{equation}
Hence $v_{t+1}-\bar v>0$ means pushing mass to $t$ \emph{increases} the error, 
while $v_{t+1}-\bar v<0$ \emph{decreases} it. 
At least in aggregate, 
\begin{equation}
\sum_{t\neq y}\pi^{\MSE}_t
=\sum_{t\neq y} a_{t+1}(v_{t+1}-\bar v)
=a_{y+1}\,(v_{y+1}-\bar v)
=a_{T_0}\,\|r\|^2 \;>\;0,
\label{eq:sum-positive}
\end{equation}
so at least one non-target direction is decreased in favor of the true class.

The next lemma shows that along diverging-margin trajectories, 
the sign of $\pi^{\MSE}_{i,t}(w)$ is eventually 
determined by the first term:  
$\sign(\pi^{\MSE}_{i,t}(w(\tau))) = \sign(h_{i,t}(w(\tau)))$. 
For Gaussian match-and-copy data,
w.h.p., $\langle d_{i,t}, d_{i,s}\rangle > 0$ for all $(i,t), (i,s)$,
so $h_{i,t}(w) > 0$ for all $w$ and $(i,t)$, 
and hence $\pi^{\MSE}_{i,t}(w(\tau)) > 0$ 
(eventually) along diverging-margin trajectories, 
thus 
ensuring margin monotone increase like for CE. 
\end{remark}

\begin{lemma}[Asymptotics of MSE coefficients along diverging-margin trajectories]
\label{lem:mse-asymptotics-dvg-margins}
If $(w(\tau))_{\tau\ge 0}$ is a trajectory such that the margins diverge, i.e.,
$\min_{i} \min_{t\neq y_i} w(\tau)^\top \tilde x_{i,t} \to +\infty$ as $\tau\to\infty$, then eventually:
\begin{equation}
    \pi^{\MSE}_{i,t}(w(\tau))
    \sim
    h_{i,t}(w(\tau))\,e^{-w(\tau)^\top \tilde{x}_{i,t}}\,S_i(w(\tau)),
    \label{eq:pi-mse-isv-asymptotics}
\end{equation}
where $S_i(w) := \sum_{s\neq y_i} e^{-w^\top \tilde x_{i,s}}$ is the softmax tail. 
\end{lemma}

\begin{proof} 
Equation \eqref{eq:mse-grad-tilde-form-lemma} 
gives the explicit expression
\[
  \pi^{\MSE}_{i,t}(w)
  =
  a^{(i)}_{t+1}(w)\,\big(
    (1-a^{(i)}_{y_i+1}(w))\,h_{i,t}(w)
    - \|r^{(i)}(w)\|_2^2
  \big).
\]
Since the margins diverge, 
the loss goes to zero and hence the MSE residuals
$r^{(i)}(w(\tau))\to 0$ for each $i$. 
By \eqref{eq:v-difference},
\[
  \big(1-a^{(i)}_{y_i+1}(w)\big)\,h_{i,t}(w) = \langle d_{i,t}, r^{(i)}(w)\rangle
\]
so the first term in 
$\pi^{\MSE}_{i,t}(w)$ 
is linear in $r^{(i)}(w)$,
while the second term is quadratic in $r^{(i)}(w)$. 
As $r^{(i)}(w(\tau))\to 0$, the quadratic term becomes negligible:
\[
  \pi^{\MSE}_{i,t}(w(\tau))
  \sim
  a^{(i)}_{t+1}(w(\tau))\,
  \big(1-a^{(i)}_{y_i+1}(w(\tau))\big)\,h_{i,t}(w(\tau)).
\]
The final asymptotic equivalent is obtained by
noting that since $S_i(w(\tau))\to 0$ (as margins diverge), we have
\[
a^{(i)}_{t+1}(w) = \frac{e^{-w^\top \tilde x_{i,t}}}{1+S_i(w)} \sim e^{-w^\top \tilde x_{i,t}}, \qquad
  1 - a^{(i)}_{y_i+1}(w) = \frac{S_i(w)}{1+S_i(w)} \sim S_i(w).
\]
\end{proof}

\subsection{Convexity of $\CE$ and nonconvexity of $\MSE$}
\label{subsec:convexity}
Fix one sample $(\Ectx,\query,T_0)$ and denote $z(\Wkq)=\lambda\Ectx^\top\Wkq\query$ and $\attn=\attn(z(\Wkq))$.

\textbf{(i) $\CE$ is convex in $\Wkq$.}
This is standard: $\CE(\Wkq)=-\log(\attn_{T_0}(z(\Wkq)))$ depends on $\Wkq$ only through the affine map $\Wkq\mapsto z_{1:T-1}(\Wkq)$, and on these logits it equals the log-sum-exp form
\[
\CE(\Wkq)
=\log\Big(\sum_{s=1}^{T-1} e^{z_s(\Wkq)}\Big)-z_{T_0-1}(\Wkq),
\]
hence $\CE$ is convex as a composition of a convex function (log-sum-exp) with an affine one.

\textbf{(ii) $\MSE$ is not convex in $\Wkq$ in general.}
Write the per-sample residual
\[
r(\Wkq):=\Ectx\, a(z(\Wkq))-\token_{T_0},\qquad \MSE(\Wkq)=\tfrac12\|r(\Wkq)\|^2.
\]
For any direction $\Delta$, the second derivative admits the generic ``Gauss-Newton + curvature'' 
decomposition by the product rule:
\begin{equation}
D^2\MSE(\Wkq)[\Delta,\Delta]
=
\underbrace{\big\|Dr(\Wkq)[\Delta]\big\|^2}_{\ge 0}
\;+\;
\underbrace{\inner{r(\Wkq)}{D^2 r(\Wkq)[\Delta,\Delta]}}_{\text{can have either sign}}.
\label{eq:mse-hess-split-remark}
\end{equation}
The first term is always positive semidefinite.
The second term is \emph{indefinite} in general, it can be written as 
$\inner{r}{\Ectx\,H(\attn)[\delta,\delta]}$ where 
$\delta:=\lambda\,\shift\,\Ectx^\top\Delta\query$ and
$H(\attn)$ denotes the second differential of the shifted-softmax; 
since $H(\attn)[\delta,\delta]$ changes sign with $\delta$ while $r$ 
can point in arbitrary directions, the inner product can be negative 
for suitable $(\Wkq,\Delta)$ whenever $r\neq 0$.
This rules out global convexity of $\MSE$ in $\Wkq$ under 
our setup (in contrast with $\CE$). 
In practice, negative curvature is easy to witness by 
minimizing the Rayleigh quotient
$\inner{\Delta}{\nabla^2\MSE(\Wkq)[\Delta]}_F/\|\Delta\|_F^2$ 
over $\Delta$ at a given $\Wkq$.

\section{Remarks on the Assumptions of \Cref{thm:mse-bias-app}}
\label{app:remarks-assumptions}

\begin{remark}
The assumptions used in \Cref{thm:mse-bias-app} play distinct roles.
\begin{enumerate}
    \item Geometric assumptions are shown to hold with high probability for GMC data
    under appropriate dimensional scaling (\Cref{subsec:gaussian-data}).
    \item Identifiability is the only obstruction to margin divergence under vanishing MSE loss
    (\Cref{subsec:identifiability-margin-divergence}).
    \item The remaining assumptions concern properties of the optimization trajectory
    and are illustrated empirically in \Cref{subsec:assumptions-hold} for a representative run.
\end{enumerate}
We emphasize that the present analysis characterizes one interpretable asymptotic regime
of MSE training dynamics.
It does not preclude the existence of other behavior 
outside the scope of the assumptions
considered here, including convergence to locally optimal separating directions
rather than the globally optimal max-margin solution, as observed in related settings \cite{Tarzanagh23TransformersAsSVM}.
\end{remark}

\subsection{Geometric assumptions for Gaussian match-and-copy}
\label{subsec:gaussian-data}

We work under the GMC model of \Cref{def:mc}.
For each sample $i$, write $y_i:=t_0^{(i)}$ for the hidden match index and define, for $t\neq y_i$,
\[
d_{i,t} := \token^{(i)}_{t+1} - \token^{(i)}_{y_i+1},
\qquad
\tilde x_{i,t} := \flatten\bigl(d_{i,t}\,\query^{(i)\top}\bigr)\in\R^{\din^2}.
\]
Tokens are i.i.d.\ $\token^{(i)}_t \sim \mathcal{N}(0,\frac{\vartoken}{\din}\I_{\din})$, and
\[
\query^{(i)} = \Cqe\,\token^{(i)}_{y_i} + \xi^{(i)},
\qquad
\xi^{(i)} \perp \{\token^{(i)}_t\}_{t=1}^{T+1}.
\]

\begin{proposition}[Validity of Geometric Assumptions]
\label{prop:geom-ass=satisfied-gmc}
\begin{enumerate}
    \item \textbf{Linear Independence (A2):}
    Let $n_{\ISV} = |\{i : \exists t, (i,t) \in \ISV\}|$  
    be the number of samples containing 
    support vectors, and $T_{\ISV} = \max_i |\{t : (i,t) \in \ISV\}|$ be the 
    maximum number of support vectors in a single sample.
    Then $\din \ge \max(n_{\ISV}, T_{\ISV}) \Longrightarrow$ \textbf{a.s.} 
    the support vectors are linearly independent. 
    \item \textbf{Identifiability (A3):}
    $\din \ge T-1 \Longrightarrow$ \textbf{a.s.} 
    the target token is not in the convex hull of distractors. 
    \item \textbf{Half-space (A4):}
    $\din \gg \log(nT) \Longrightarrow$ 
    \textbf{w.h.p.}
    all pairwise correlations $\langle d_{i,t}, d_{i,s} \rangle$ are positive. 
\end{enumerate}
\end{proposition}

In particular, 
these assumptions hold more easily as the
input dimension $\din$ grows relative to $T$ and $n$. 

\begin{proof}[Proof of \Cref{prop:geom-ass=satisfied-gmc}]
\textbf{1. Linear Independence:} 
We exploit the rank-one structure $\tilde x_{i,t} = \flatten(d_{i,t}\query^{(i)\top}) 
= \query^{(i)} \otimes d_{i,t}$.
Consider a linear dependence $\sum_{(i,t)\in\ISV} \alpha_{i,t} \tilde x_{i,t} = 0$. 
Regrouping terms by sample index $i$, this can be written as  
$\sum_{i} \query^{(i)} \otimes v_i = 0$, where $v_i = \sum_{t \in \ISV(i)} \alpha_{i,t} d_{i,t}$.
If $\din \ge n_{\ISV}$, the query vectors $\{\query^{(i)}\}$ associated with the active 
samples are linearly independent a.s. (as they are independent Gaussian vectors). 
This implies that $v_i = 0$ for each sample $i$. 
Furthermore, if $\din \ge T_{\ISV}$, the difference vectors $\{d_{i,t}\}_{t\in\ISV(i)}$ 
within each sample are linearly independent a.s. (as they are Gaussian in $\R^{\din}$).
The condition $v_i = 0$ thus forces $\alpha_{i,t} = 0$ for all coefficients.

\textbf{2. Identifiability:} 
Fix a sample $i$. 
The set of distractor tokens 
$D_i = \{\token^{(i)}_{t+1} : t \neq y_i\}$ 
contains $T-1$ vectors. 
The convex hull 
$\mathcal{K} = \text{Conv}(D_i)$ 
is contained 
in an affine subspace of dimension at most $T-2$. 
Since $\din > T-2$, 
the Lebesgue measure of $\mathcal{K}$ in 
$\R^{\din}$ is zero. 
The target token 
$\target = \token^{(i)}_{y_i+1}$ is drawn from 
$\mathcal{N}(0, \frac{\stdtoken^2}{\din} I_{\din})$ 
independently of $D_i$. 
Therefore, 
$\mathbb{P}(\target \in \mathcal{K}) = 
\E[\mathbb{P}(\target \in \mathcal{K} | D_i)] = 0$.
Taking a union bound over $n$ samples 
preserves the almost sure validity.

\textbf{3. Half-space condition:} 
We show that for some universal constants $c_1,c_2>0$,
\[
\mathbb{P}\Bigl(
  \exists\, i, \exists s \neq t \neq y_i:\ 
  \langle d_{i,s}, d_{i,t} \rangle
         \le 0
\Bigr)
\le c_1 n T^2 e^{-c_2 \din}.
\]
The event is scale-invariant, 
so assume $\token^{(i)}_t\sim \mathcal{N}(0, I_{\din})$.
Fix a sample $i$ and distinct $s,t \neq y_i$. 
Omit index $i$ for brevity. 
Let
$d_s = \token_{s+1} - \token_{t_0+1}$ and 
$d_t = \token_{t+1} - \token_{t_0+1}$.
Expanding the inner product:
\[
\langle d_s, d_t \rangle
= \token_{s+1}^\top \token_{t+1}
  - \token_{s+1}^\top \token_{t_0+1}
  - \token_{t+1}^\top \token_{t_0+1}
  + \|\token_{t_0+1}\|^2.
\]
Expanding in coordinates, 
this is a sum of $\din$ i.i.d.\ random variables $Z_k$ with the structure $XY - XW - YW + W^2$ 
(where $X,Y,W$ are standard normal).
We calculate the moments: 
$\E[Z_k] 
= \E[W^2]
= 1$ 
and $\E[Z_k^2] = \E[X^2 Y^2] + \E[X^2 W^2] + \E[Y^2 W^2] + \E[W^4]
= 1 + 1 + 1 + 3 = 6$ since terms in the square expansion with odd
powers have zero expectation. 
Thus, $\Var[Z_k] = \E[Z_k^2] - \E[Z_k]^2 = 5$ 
and we deduce that 
\[
\E[\langle d_s, d_t \rangle] = \din,
\qquad
\Var(\langle d_s, d_t \rangle) = 5\din.
\]
Thus $\langle d_s, d_t\rangle$ is a sum of i.i.d.\ \emph{sub-exponential}
variables with mean $\din$ and variance of order $\din$. By standard
concentration inequalities for sub-exponential sums there exist
universal constants $c_1,c_2>0$ such that
\[
\mathbb{P}\bigl(\langle d_s, d_t \rangle \le 0\bigr)
= \mathbb{P}\bigl(\langle d_s, d_t \rangle - \din \le -\din\bigr)
\le c_1 e^{-c_2 \din}.
\]
There are at most $n \binom{T-2}{2} \le n T^2/2$ such pairs. 
By a union bound, the probability of any violation is 
at most $c_1 n T^2 e^{-c_2 \din}$, which vanishes 
exponentially fast in $\din - \log(nT)$.
\end{proof}

\subsection{On why identifiability is the only obstacle to margin divergence 
under vanishing MSE loss}
\label{subsec:identifiability-margin-divergence}
The MSE loss can go to zero even when the margins do not diverge,
unlike the CE loss studied in \citet{Soudry18ImplicitBiasCE}. 
A trivial example is when the target token 
$\token_{T_0} = \sum_{t\ne T_0} \alpha_t \token_t$ is in the 
convex hull of the non-target tokens, with $(\alpha_t)_t$ 
that can be realized as attention weights. 
In that case, setting the attention weights to $(\alpha_t)_t$
yields zero MSE loss, but the margins remain bounded. 
The next lemma shows that this is the only obstacle to margin divergence. 

\begin{lemma}[Identifiability implies margin divergence under vanishing MSE loss]
\label{lem:identifiability-margin-divergence}
For each sample $i$, let 
$C_i := \textrm{conv}\{\token^{(i)}_t : t\in{2,\dots,T_i}\setminus{T_0^{(i)}}\}$ 
be the convex hull of the non-target tokens and define its 
identifiability gap as
\[
\gamma_i:=\dist\big(\token^{(i)}_{T_0^{(i)}},C_i\big).
\]
There exists a constant $C_i>0$ such that 
the MSE residual $r^{(i)}(w) = \Ectx^{(i)} a^{(i)}(w) - \token^{(i)}_{T_0^{(i)}}$ 
for sample $i$ satisfies for all $w$:
\[
(1-a^{(i)}_{T_0^{(i)}}) C_i
\ge 
\|r^{(i)}\|
\ge (1-a^{(i)}_{T_0^{(i)}})\gamma_i.
\]
In particular, when $\gamma_i>0$, we have $\MSE_i\to0$ 
\emph{iff} $a^{(i)}_{T_0^{(i)}}\to1$, 
and there is no finite-norm minimizer of the MSE loss under identifiability.
\end{lemma}
\begin{proof}
We drop the index $i$. By definition of shifted attention, $a_1=0$, hence 
\[ 
\Ectx a = a_{T_0}\token_{T_0} + (1-a_{T_0}) 
\underbrace{
  \sum_{t\ne T_0, t \ge 2} 
\frac{a_t}{1-a_{T_0}} \token_t
}_{=:u}
\] 
where $u$ lies 
in the convex hull $C$. 
Therefore, $r=\Ectx a - \token_{T_0}=(1-a_{T_0})(u-\token_{T_0})$ and taking the norms yields the claim.
\end{proof}

\subsection{A representative run where the assumptions hold}
\label{subsec:assumptions-hold}

We provide an empirical sanity check that the technical assumptions used in our analysis
can hold in practice on Gaussian match-and-copy (GMC) data in a regime where gradient descent
exhibits clear directional alignment with the max-margin solution within a finite iteration budget.
This section is intentionally illustrative: 
it documents one representative configuration where
all proxy diagnostics behave consistently with the assumptions of \Cref{thm:mse-bias-app}.

\paragraph{Setup.}
We consider GMC samples with exact copy ($\target=\token_{t_0+1}$) and exact match ($\query=\token_{t_0}$),
and we optimize the MSE objective by (full-batch) gradient descent.
Unless stated otherwise, we run for $40$k iterations with step size $10^{-3}$ and log relevant quantities at each iteration. In particular, partial sums indexed by time iteration are computed with all terms. 

\paragraph{Directional alignment and loss decay.}
\Cref{fig:mse-implicit-bias} reports the training loss decay and the cosine similarity between the iterate
$\Wkq(\tau)$ and the max-margin predictor $\maxmargin$.
In this run, the loss converges to a small value and the direction of $\Wkq(\tau)$ rapidly aligns with $\maxmargin$.
Empirically, we observe that such strong alignment within a fixed budget occurs in some regimes but not uniformly across all
configuration depending on the interplay of $(\din,T,n)$ and step size. Since our theoretical statement concerns the asymptotic max-margin regime, we focus here on runs that clearly aligns within the iteration budget, and seek to verify whether the assumptions used to guarantee the alignment occur or not, and hence explain or not the observed behavior. 

\begin{figure}[t]
    \centering
    \begin{subfigure}[t]{0.32\linewidth}
        \centering
        \includegraphics[width=\linewidth]{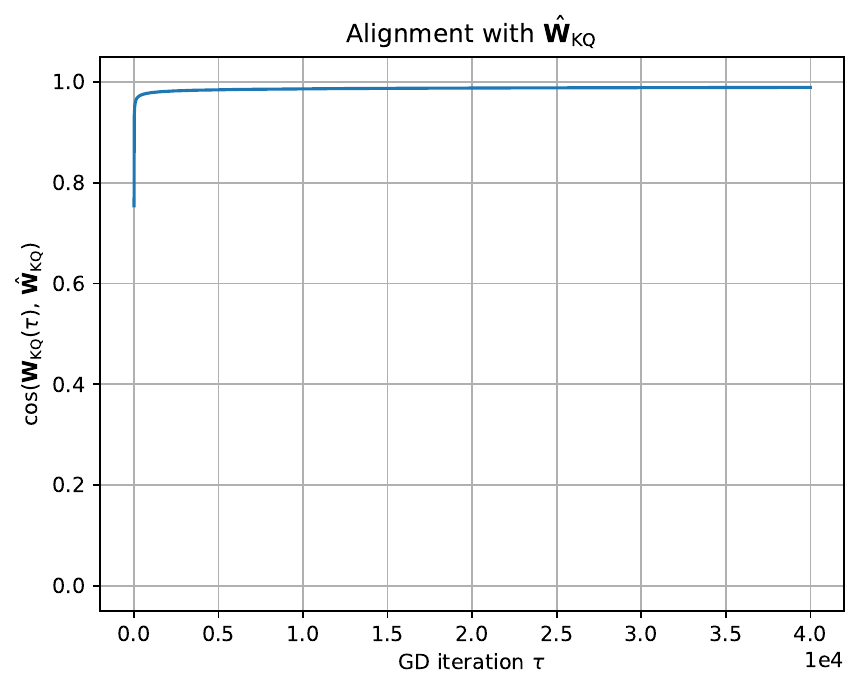}
        \caption{$\cos(\Wkq(\tau),\maxmargin)$}
    \end{subfigure}\hfill
    \begin{subfigure}[t]{0.32\linewidth}
        \centering
        \includegraphics[width=\linewidth]{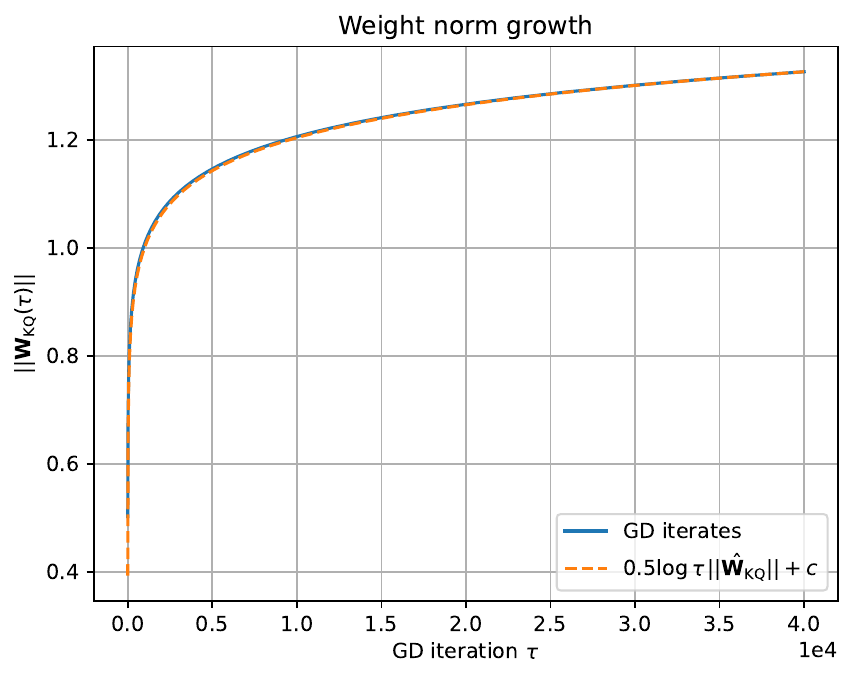}
        \caption{$\|\Wkq(\tau)\|$}
    \end{subfigure}\hfill
    \begin{subfigure}[t]{0.32\linewidth}
        \centering
        \includegraphics[width=\linewidth]{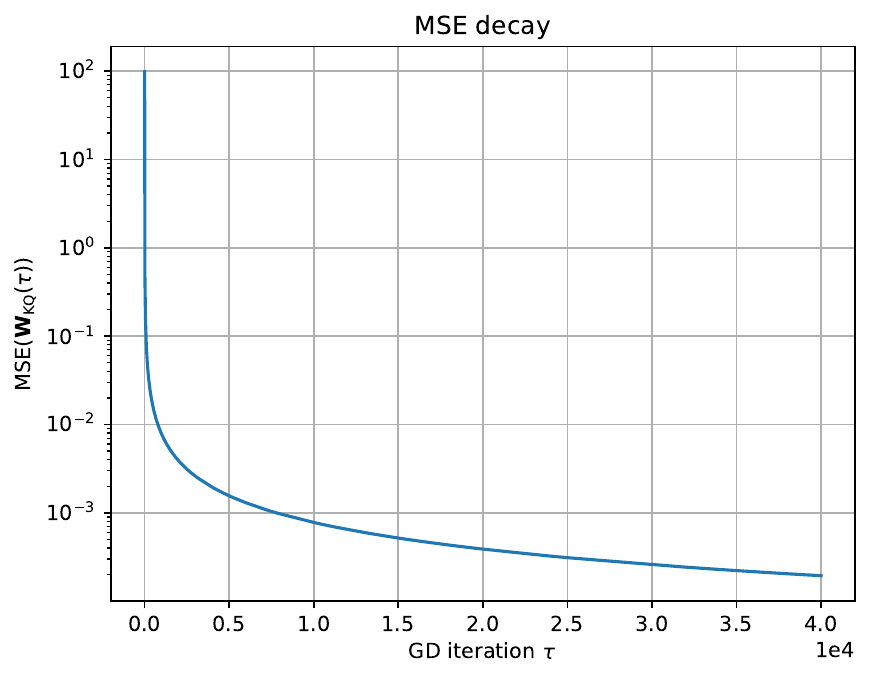}
        \caption{training loss}
    \end{subfigure}
    \caption{GD on MSE for shifted-key one-layer attention on GMC data with exact match and copy
    ($\din{=}16$, $T{=}4$, $n{=}8$, step size $10^{-3}$, $40$k iterations).
    We plot (left) cosine similarity between $\Wkq(\tau)$ and the max-margin solution $\maxmargin$,
    (middle) the weight norm $\|\Wkq(\tau)\|$, and (right) training loss. In (b), we verify that the norm diverges at speed $\frac12 \log \tau \|\maxmargin\|$ (dashed line), confirming the rate in \Cref{thm:mse-bias-app}.}
    \label{fig:mse-implicit-bias}
\end{figure}

\paragraph{Geometric conditions (A1--A4).} 
We verify a hard-margin feasibility condition (minimum margin $\ge 1$), linear independence of support vectors (numerical rank), strict separation from convex hull of distractors, and positivity of relevant dot-products within the support set. These are direct finite-dimensional checks on the synthetic data, which are satisfied in this configuration. 

\paragraph{Training loss (A5).} 
\Cref{fig:mse-implicit-bias} shows that the loss decreases monotonically across iterations, 
and has reached a low-loss regime at the end of the run.

\paragraph{Gradient-square summability proxy (A6).}
\Cref{fig:assump-grad-l2} shows that the cumulative sum of squared gradient norms appears to stabilize, tail increments decrease, 
and $\|\nabla\mathcal{L}(\Wkq(\tau))\|^2$ exhibits an approximate power-law decay on log-log axes $~\tau^{-2}$. 
Together these observations are consistent with the summability condition required in (A6).

\begin{figure}[t]
    \centering
    \begin{subfigure}[t]{0.40\linewidth}
        \centering
        \includegraphics[width=\linewidth]{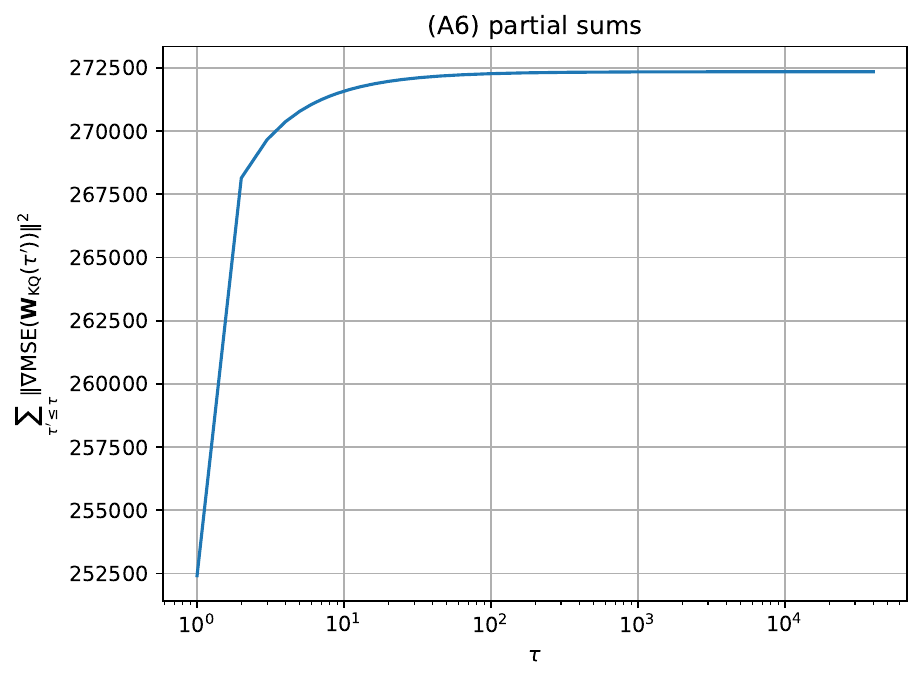}
        \caption{$\sum_{\tau'\le \tau}\|\nabla \mathcal{L}(\Wkq(\tau'))\|^2$}
    \end{subfigure}\hfill
    \begin{subfigure}[t]{0.40\linewidth}
        \centering
        \includegraphics[width=\linewidth]{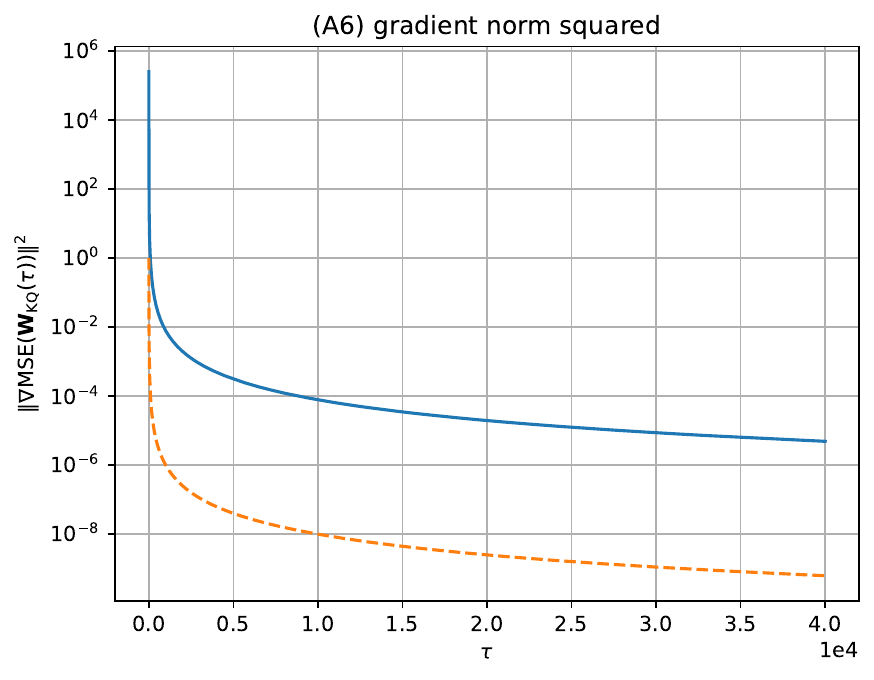}
        \caption{$\|\nabla \mathcal{L}(\Wkq(\tau))\|^2$}
    \end{subfigure}
    \caption{Empirical proxies for (A6). 
    The orange dashed line in (b) corresponds to a power-law decay $\tau^{-2}$ fitted on the last $10$k iterations.}
    \label{fig:assump-grad-l2}
\end{figure}

\paragraph{Pre-factor convergence proxy (A7).}
\Cref{fig:assump-h} reports empirical proxies for convergence of the auxiliary quantities $h_{i,t}(\Wkq(\tau))$
and for the weighted series that appears in (A7). 
We estimate $h^\infty$ by a tail average over the last percent of checkpoints, 
and report the maximum relative deviation $\max_{i,t}|h_{i,t}(\tau)-h^\infty_{i,t}|/ \sum_{i,t} |h_{i,t}^\infty|$ 
and the mean of these relative deviations across $i,t$ as a function of iteration $\tau$. 
We observe stabilization of $h_{i,t}(\Wkq(\tau))$ and a saturating behavior of the partial sums. 
Precisely estimating asymptotic summability rates from finite horizons is delicate, but these plots are consistent 
with the qualitative behavior required by (A7). 

\begin{figure}[t]
    \centering
    \begin{subfigure}[t]{0.40\linewidth}
        \centering
        \includegraphics[width=\linewidth]{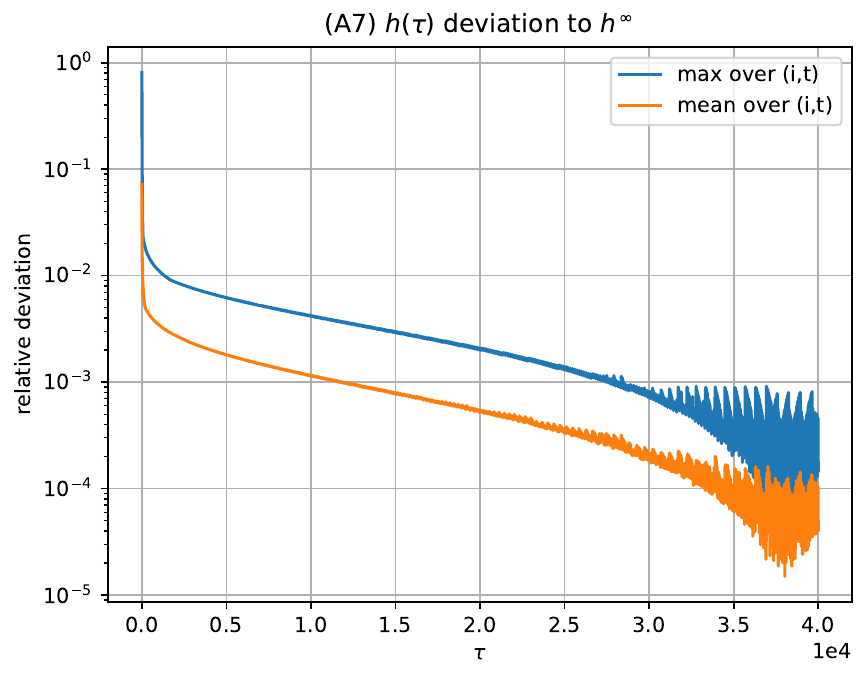}
        \caption{deviation to $h_\infty$}
    \end{subfigure}\hfill
    \begin{subfigure}[t]{0.40\linewidth}
        \centering
        \includegraphics[width=\linewidth]{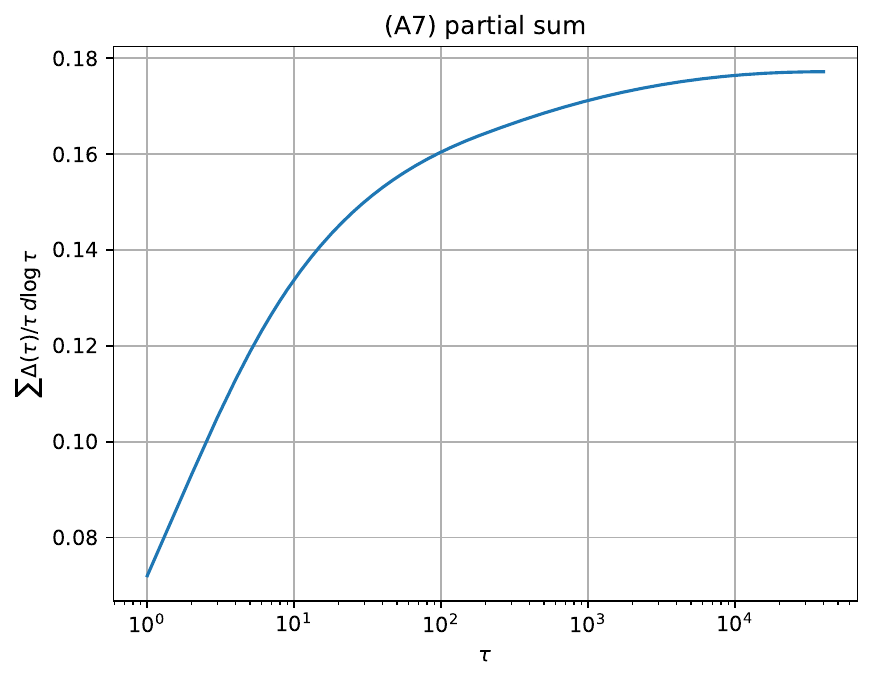}
        \caption{$\sum \|h(\tau)-h_\infty\|/\tau$}
    \end{subfigure}
    \caption{Empirical proxies for (A7) 
    based on estimation of $h^\infty$ from tail averages.} 
    \label{fig:assump-h}
\end{figure}

\paragraph{Alignment-ratio proxy (A8).} 
We estimate the ratio between the off-diagonal contributions
in \eqref{eq:alignment-assumption} 
with disagreeing signs
and the total off-diagonal contributions. 
The estimation is based on the estimate of $h^\infty$ from tail averages,
and subsequent estimates of $\tilde w_{i,t}(\tau)$, $U_{i,t,s}(\tau)$, $H_{i,t,s}(\tau)$.
We find that this ratio remains close to zero across iterations.
Since (A8) only requires that misaligned off-diagonal contributions
remain a strict minority of the total off-diagonal mass,
this behavior is compatible with the condition in this configuration.

\paragraph{Softmax tail domination proxy (A9).}
\Cref{fig:assump-softmax-tail-proxy} reports proxies for the softmax tail domination condition (A9).
The tail-mass proxy $g(\tau)$ decreases, 
and the corresponding partial sum appears to stabilize,
which is consistent with the behavior required by (A9). 

\begin{figure}[t]
    \centering
    \begin{subfigure}[t]{0.40\linewidth}
        \centering
        \includegraphics[width=\linewidth]{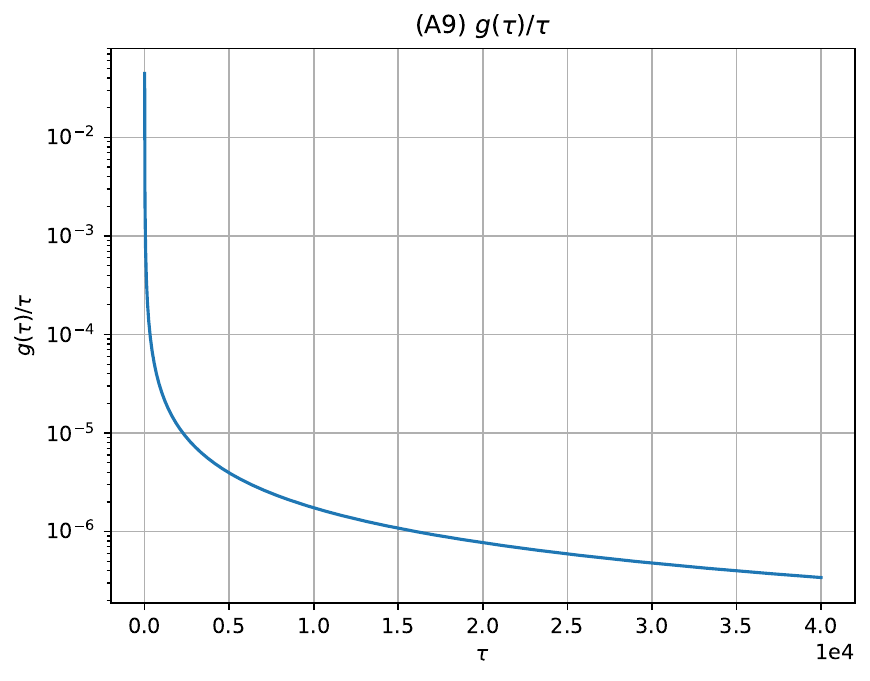}
        \caption{$g(\tau)/\tau$}
    \end{subfigure}\hfill
    \begin{subfigure}[t]{0.40\linewidth}
        \centering
        \includegraphics[width=\linewidth]{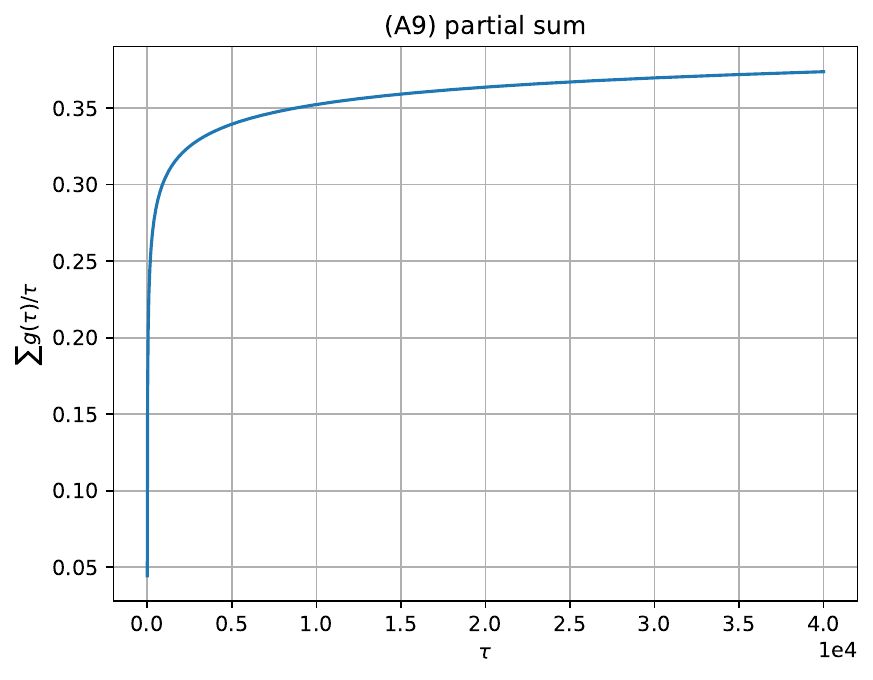}
        \caption{$\sum g(\tau)/\tau$}
    \end{subfigure}
    \caption{(A9) Softmax tail domination.}
    \label{fig:assump-softmax-tail-proxy}
\end{figure}

\section{Notes on Benchmark Classification and Design Choices}
\label{app:related-works-table}

\Cref{tab:benchmark-comparison} 
summarizes several benchmarks from the literature that
study match-and-copy or in-context learning tasks. This table is not intended to be exhaustive (and is surely not, given the size of the literature). Instead, it is  
intended to highlight how the design choices vary across works 
and how GMC fits within this landscape. 

Overall, \Cref{tab:benchmark-comparison} should be read as a navigational aid rather than a
definitive taxonomy.
Its purpose is to highlight which benchmarks simultaneously (i) make match-and-copy
structure explicit, (ii) give rise to empirically observable and robust PTH$\to$IH behavior, 
(iii) exhibit architectural sensitivity favoring Transformers,
and (iv) is amenable to the theoretical analysis of the optimization dynamics. 

Obviously, each column compresses a wide range of modeling and experimental choices across works, so the annotations should be interpreted with care.

\paragraph{Classification of \citet{Marion22AttentionLayers}.} 
The task of single-location regression studied in \citet{Marion22AttentionLayers} can be reformulated as a special case of
Gaussian match-and-copy (GMC) with a deterministic, fixed query across samples. 
We therefore include \citet{Marion22AttentionLayers} in the GMC family, with the annotation
``fixed query'' to emphasize this restriction.

\paragraph{Why some entries are marked as ``not discussed'' or ``unlikely''.}
In several works, the presence or robustness of PTH$\to$IH mechanisms 
is not examined, and we therefore mark the corresponding entries as 
``not discussed''. 
In a small number of cases, although this aspect is not studied explicitly, 
we label the entry as ``unlikely'', 
as the reported results indirectly suggest that PTH$\to$IH 
mechanisms do not emerge or are not robust.

For instance, in \citet{Marion22AttentionLayers}, 
a solution based on detecting a mean shift can solve 
the task without requiring a PTH$\to$IH circuit. 
Likewise, in \citet{Lee24ICLwithOtherArchitectures}, 
metrics associated with PTH$\to$IH behavior vary across data distributions, 
suggesting limited robustness of such mechanisms.

\begin{table}[t]
\centering
\scriptsize
\setlength{\tabcolsep}{4pt}
\renewcommand{\arraystretch}{1.15}
\caption{
Comparison of synthetic and empirical benchmarks for match-and-copy / ICL. 
\textbf{Task:} NTP = next-token prediction; 
NLP = natural language processing; 
NLP Class. ICL = NLP classification in-context learning 
as in \eqref{eq:template-icl-rw},
like sentiment analysis, 
topic classification, etc.;
(G)MC = (Gaussian) match-and-copy; 
AR = recent variants of associative recall, which differ from the original discrete version of \citet{Graves14NeuralTuringMachines} as discussed in \Cref{sec:related-work}; 
MQAR = multi-query associative recall; n-gram ICL = in-context learning of n-gram models; 
LinReg ICL = in-context learning of linear regression ($f$ in \eqref{eq:template-icl-rw} is a linear function);
Class. ICL = classification in-context learning ($f$ in \eqref{eq:template-icl-rw} 
is a classifier), typically Omniglot or Gaussian clusters. 
\textbf{Match signal:} $=$ for exact equality (the query is equal to the context token it should match); \Cone\ for injected mean-shift (first-order marker); 
\Ctwo\ for injected second-order correlation; \noExplicitMatch 
for no explicit match signal provided (typically the case for ICL tasks as in \eqref{eq:template-icl-rw}, where the query $x$ is independent of the examples $x_i$);
same marg. for query and match having the same marginal distribution. 
\textbf{Arch. sens.:} qualitative architectural sensitivity of the benchmark: 
pre-T for pre-Transformer architecture comparisons (e.g., RNNs, LSTMs, Fast Weights);
$T>$ when Transformers outperform alternatives;
$\approx$ when ranking depends on setup and no clear winner; 
Alt when a new architecture is proposed to close a gap;
left blank when no comparison with non-Transformer architectures. 
\textbf{PTH$\to$IH:} \emph{Obs.} = observed in trained models; \emph{Rob.} = robust across setups. 
\cmark when observed/robust; \xmark when not emerging or not robust; \unlikely when not discussed but indirect evidence against; left blank \notDiscussed when not studied.
\textbf{Theory:} $\to$ for works establishing results about the optimization dynamics; 
$\nabla$ for ones exhibiting that among critical points, some are (close to) PTH$\to$IH solutions;
$\exists$ for ones establishing representational/existence statements only (e.g., 
the model can be instantiated to implement a PTH$\to$IH circuit, but it is not known whether it is learned by optimization), 
left blank when no such results are provided.
}
\begin{tabular}{l l c c cc cc}
\toprule
\textbf{Paper} & \textbf{Task} & \textbf{Match} & \textbf{Arch. sens.} &
\multicolumn{2}{c}{\textbf{PTH$\!\to\!$IH}} &
\multicolumn{2}{c}{\textbf{Theory}} \\
& & & & Obs. & Rob. & Opt. & Rep. \\
\midrule

\citep{Elhage21TransformerCircuits,Olsson22AnthropicInContextLearning}
& NLP NTP & \noExplicitMatch & \noComparison & \cmark & \cmark & & \\
\citep{Min22ICLwithRandomLabels}
& NLP ICL & \noExplicitMatch & \noComparison & \notDiscussed & \notDiscussed & & \\
\citep{Dai23GPTMetaOptimization}
& NLP Class. ICL & \noExplicitMatch & \noComparison & \notDiscussed & \notDiscussed & & \\
\citep{Crosbie24AblatingIHDeterioratesICL}
& NLP Class. ICL & \noExplicitMatch & $T>$ & \cmark & \cmark & & \\

\midrule
\citep{Graves14NeuralTuringMachines}
& discrete MC (orig. AR) & $=$ & pre-T & \notDiscussed & \notDiscussed & & $\exists$ \\
\citep{Marion22AttentionLayers}
& GMC with fixed query & \Cone & \noComparison & \unlikely & \notDiscussed & $\to$ & \\
\textbf{This work (GMC)}
& GMC & \Ctwo & $T>$ & \cmark & \cmark & $\to$ & \\

\midrule
\citep{Ba16FastWeightsAttendRecentPast}
& AR & $=$ & pre-T & \notDiscussed & \notDiscussed & & \\
\citep{Arora24ZoologyRecall}
& MQAR & $=$ & $T>$ & \notDiscussed & \notDiscussed & & $\exists$ \\
\citep{Fu23H3,Poli23HyenaHierarchy}
& AR & $=$ & Alt & \notDiscussed & \notDiscussed & & \\

\midrule
\citep{Bietti23TransformerMemoryViewpoint}
& n-gram ICL & $=$ & \noComparison & \cmark & \xmark & & $\exists$ \\
\citep{Akyurek24ICLLanguageGenbyAutomata}
& n-gram ICL & \noExplicitMatch & $T>$ & \cmark & \notDiscussed & & \\
\citep{Nichani24OneGramTaskLearnsIH}
& n-gram ICL & \noExplicitMatch & \noComparison & \cmark & \notDiscussed & $\to$ & \\
\citep{Chen24nGramTaskLearnsIH}
& n-gram ICL & \noExplicitMatch & \noComparison & \cmark & \notDiscussed & $\to$ & \\
\citep{Makkuva24AttentionMarkov}
& n-gram ICL & \noExplicitMatch & \noComparison & \cmark & \notDiscussed & & $\exists$ \\
\citep{Edelman24OnetoTwoGramMarkovLearning}
& n-gram ICL & \noExplicitMatch & \noComparison & \cmark & \notDiscussed & & $\exists$ \\
\citep{Rajaraman24ApproximationNGramsByTransformers}
& n-gram ICL & \noExplicitMatch & \noComparison & \notDiscussed & \notDiscussed & & $\exists$ \\
\citep{Varre25LearningInContextNGrams}
& n-gram ICL & \noExplicitMatch & \noComparison & \cmark & \notDiscussed & $\nabla$
& \\

\midrule
\citep{Akyurek23InContextLearningLinearModelsviaGD}
& LinReg ICL & $=$ & \noComparison & \notDiscussed & \notDiscussed & & $\exists$ \\
\citep{Oswald23InContextGD}
& LinReg ICL & \noExplicitMatch & \noComparison & \notDiscussed & \notDiscussed & & \\
\citep{Garg22InContextCaseStudy}
& LinRe / Reg. Tree / 2-layer NN ICL & \noExplicitMatch & \noComparison & \notDiscussed & \notDiscussed & & \\
\midrule
\citep{Chan22DataDistributionDrivesInContextLearning,Singh23TransienceICL}
& Class. ICL & mix & $T>$ & \cmark & \xmark & & \\
\citep{Singh24InductionHeadFormation}
& Class. ICL & $=$ & \noComparison & \cmark & \notDiscussed & & \\
\citep{Reddy24GaussianDataforMechanisticStudyClassificationICL}
& Class. ICL & same marg. & \noComparison & \cmark & \xmark & & \\

\midrule 
\citep{Bhattamishra23ICLBooleanFunctions}
& Boolean ICL & \noExplicitMatch & $\approx$ & $\approx$ & \notDiscussed & & \\

\midrule
\citep{Wang25LazyRichLearningIH}
& mix 4-gram and IH & \noExplicitMatch & \noComparison & \cmark & \notDiscussed & $\to$ & \\

\midrule
\citep{Lee24ICLwithOtherArchitectures}
& AR / ICL & \noExplicitMatch & $\approx$ & \notDiscussed & \unlikely & & \\

\bottomrule
\end{tabular}
\label{tab:benchmark-comparison}
\end{table}

\end{document}